\documentclass[twoside,11pt]{amsart}

\usepackage{amsmath,amsfonts,amssymb,mathrsfs,mathtools}
\usepackage{stmaryrd}
\usepackage[normalem]{ulem}
\usepackage{soul}
\usepackage{color}

\usepackage{multirow}
\usepackage{array}
\usepackage{subcaption}

\newtheorem{theorem}{Theorem}[section]
\newtheorem{lemma}[theorem]{Lemma}
\newtheorem{proposition}[theorem]{Proposition}
\newtheorem{corollary}[theorem]{Corollary}

\newtheorem{properties}[theorem]{Properties}

\theoremstyle{definition}

\newtheorem{example}[theorem]{Example}

\newtheorem{definitions and remarks}[theorem]{Definitions and Remarks}

\theoremstyle{remark}

\newtheorem{remark}[theorem]{Remark}

\DeclareMathOperator{\sign}{sign}
\renewcommand{\c}[1]{\mathcal{#1}}

\newcommand{\N}{\mathbb{N}}
\newcommand{\R}{\mathbb{R}}

\newcommand{\y}{\mathbf{y}}

\newcommand{\eps}{\varepsilon}
\newcommand{\norm}[1]{\left\|#1\right\|}
\newcommand{\abs}[1]{\left|#1\right|}
\newcommand{\ps}[2]{\left\langle #1, #2\right\rangle}

\newcommand{\x}{ \mathbf{x}}
\newcommand{\xe}{ \widetilde{\mathbf{x}}}

  \DeclareMathOperator{\argmin}{argmin}

\newtheorem{Assumption}{Assumption}

\newcommand{\Ada}{\textsc{Adagrad}}
\newcommand{\Adadelta}{\textsc{Adadelta}}

\newcommand{\Ams}{\textsc{Amsgrad}}
\newcommand{\Adam}{\textsc{Adam}}
\newcommand{\RMSprop}{\textsc{RMSprop}}

\numberwithin{equation}{section}

\definecolor{OliveGreen}{rgb}{0,0.6,0}

\usepackage[disable]{todonotes} 

\makeatletter
\def\paragraph{\@startsection{paragraph}{4}%
  \z@\z@{-\fontdimen2\font}%
  {\normalfont\bfseries}}
\makeatother

\usepackage{hyperref}

\makeatletter
\newcommand{\printfnsymbol}[1]{%
  \textsuperscript{\@fnsymbol{#1}}%
}

\makeatother

\begin{document}

\title[Differential equations to model adaptive algorithms]{A general system of differential equations to model first order adaptive algorithms}

\author[A.~Belotto da Silva]{Andr\'e Belotto da Silva}
\author[M.~Gazeau]{Maxime Gazeau}
\address[A.~Belotto da Silva]{Universit\'e Aix-Marseille, Institut de Math\'ematiques de Marseille
(UMR CNRS 7373), Centre de Math\'ematiques et Informatique, 39 rue F. Joliot Curie, 
13013 Marseille, France}
\email{andre-ricardo.belotto-da-silva@univ-amu.fr}
\address[M.~Gazeau]{Borealis AI, MaRS Heritage Building, 101 College St, Suite 350, Toronto, ON M5G 1L7}
\email{maxime.gazeau@borealisai.com}


\maketitle

 \begin{abstract}
First order optimization algorithms play a major role in large scale machine learning.
A new class of methods, called \emph{adaptive algorithms}, were recently introduced to adjust iteratively the learning rate for each coordinate.
Despite great practical success in deep learning, their behaviour and performance on more general loss functions are not well understood.
In this paper, we derive a non-autonomous system of differential equations,
which is the continuous time limit of adaptive optimization methods. We study the convergence of its trajectories and give conditions under which
the differential system, underlying all adaptive algorithms, is suitable for optimization.
  We discuss convergence to a critical point in the non-convex case
     and give conditions for the dynamics to avoid saddle points and local maxima.
     For convex loss function, we introduce a suitable Lyapunov functional which allow us to study its rate of convergence. Several other properties of both the continuous and discrete systems are briefly discussed. The differential system studied in the paper is general enough to encompass many other classical algorithms (such as Heavy ball and Nesterov's accelerated method)
     and allow us to recover several known results for these algorithms.
\end{abstract}%


\section{Introduction}\label{sec:Introduction}
Optimization is at the core of many machine learning problems. Estimating the model
parameters can often be formulated in terms of an unconstrained optimization problem of the form
\begin{equation}\label{eq:unconstrained_opt}
\min_{\theta \in \R^d} f(\theta) \qquad \text{where }  f:\R^d \to \R \text{ is differentiable}.
\end{equation}
The emergence  of  deep  learning  has  spawned  the  recent  popularity  of  a special class of optimizers to solve \eqref{eq:unconstrained_opt}:
  first order \emph{adaptive} optimization algorithms (\RMSprop{} \cite{TH:12}, \Ada \cite{DHS:11, DS:13}, \Adadelta{} \cite{Z:12}, \Adam{}
\cite{KB:14}) were originally designed to solve unconstrained optimization problem (minimizing an empirical risk in supervised learning). It is commonly observed that the value of the training loss decays faster than for stochastic gradient descent and they have become a default method
of choice for training feed-forward and recurrent neural network \cite{gregor:15, Radford2015UnsupervisedRL}.
However, recent research paper suggests to not use \Adam{} \cite{RKK:18} as it diverges for a simple example.
As of today, there is no consensus on the benefit of adaptive algorithms over other methods and no guidance on how to choose the many hyper-parameters of the model.

\vspace{\baselineskip}

Despite its obvious efficiency in deep learning, the reasons of their success are unclear and a large number of fundamental questions are still unanswered. Our work started from the belief that these algorithms are not intrinsically better than gradient descent but rather well suited to the subclass of non-convex function given by standard deep learning architecture. Studying the convergence of the discrete and stochastic adaptive algorithms for non-convex functional is far too complex
and general to get insightful explanation about their efficiency in deep learning. We, therefore, start by studying a deterministic and continuous equation, and we prove that in simple cases (such as convex functional), adaptive algorithm are not converging faster than gradient descent. In particular, the key insights of our analysis are:

\begin{enumerate}
  \item The convergence rate is nonlinear --in the sense that it depends on the variables-- and depends on the history of the dynamics. Initialization is therefore of crucial importance.
  \item With the standard choices of hyperparameters,
  adaptivity degrades the rate of convergence to the global minimum of a convex function compared to gradient descent.
\end{enumerate}
These observations are crucial to unwind the mystery of adaptive algorithms and the next questions to ask are now obvious:
\begin{enumerate}
\item Does adaptivity reduces the variance (compared to SGD) and speed up the training for convex functional?
\item Is the fast training observed in deep learning induced by the specificity of the loss surface and common initialization scheme for the weights?
\end{enumerate}

\vspace{\baselineskip}

The main contribution of this paper is to provide a theoretical framework to study deterministic adaptive algorithms. Inspired by the history of gradient descent and stochastic gradient descent, we analyse discrete \emph{adaptive} optimization algorithms by introducing their continuous time counterparts (equation \eqref{eq:ODE}), with a focus on \Adam{} (equation \eqref{eq:ODEAdam}). The techniques and analysis are similar for other algorithms and includes classical accelerated method. The continuous time analysis provides conditions on the hyper-parameters of several algorithms (in particular for \Adam{}, see $\S\S$~\ref{ssec:Adam}) that guarantee convergence of their trajectories.
The convergence analysis of the continuous system brings important informations
on the discrete dynamics in the limit of large batch sizes and small learning rates. Our study is, nevertheless, far more general than \Adam{}. We provide practitioners a wide range of options to develop and test adaptive-type algorithms, and our analysis give certain guidelines which are discussed in the body of the paper and summarized in $\S\S$~\ref{ssec:Future}.

This work is intended, furthermore, to serve as a solid foundation for the posterior study in the discrete and stochastic settings, but in this paper we put an emphasize on the deterministic equation to understand
the fundamental properties of adaptive algorithms.

\vspace{\baselineskip}

In section \ref{sec:ModelPresentation} we introduce two general continuous dynamical system \eqref{eq:ODE} and \eqref{eq:ODEbis} whose \emph{forward} Euler approximation \eqref{eq:discrete} matches a large class of first order methods, summarized in tables \ref{tab:1} and \ref{tab:2}. The connection with adaptive and accelerated algorithms is made precise in section \ref{sec:Appl}. In particular, \Adam{} differential equation \eqref{eq:ODEAdam} is derived in subsection \ref{ssec:Adam}. Section \ref{ssec:Energy} starts with a first basic energy functional \eqref{eq:EnergyPrelim} (which is inspired on the works of \cite{alvarez2000}). Basic properties of the ODE \eqref{eq:ODE} are presented in subsection \ref{ssec:Basic}, including Theorem \ref{th:existencePrel} on the existence and uniqueness of solutions and Theorem \ref{thm:ConvergenceDiscretizationRate} which provides the relation between continuous and discrete systems. We note that a sharper result (Theorem \ref{th:existence}) on the existence and uniqueness of solutions of ODE \eqref{eq:ODE} is presented in section \ref{sec:ExistenceResult}.

Section \ref{sec:MainResults} contains the statement of our main results, on the asymptotic behaviour of the continuous deterministic trajectories of the ODE \eqref{eq:ODE}. In the non-convex setting we prove, under mild assumptions, that the trajectories converge to the critical locus of $f$ (see Theorem \ref{thm:convergence}). This result is supplemented with the analysis of sufficient conditions in order to avoid convergence to saddle or local maximum points of the loss function $f$ (see Theorem \ref{thm:avoiding}). For convex functions, we design a Lyapunov functional \eqref{eq:Lyapunov} and obtain a rate of convergence to at least a neighbourhood of the critical locus (see Theorem \ref{thm:convergenceRate}). The rate of convergence crucially depends on the behaviour over time of $\nabla f$ and on the term $v$ (see \eqref{eq:CrucialTerm} and the subsequent discussion). In particular, this indicates that the efficiency of adaptive algorithms is highly dependant on the loss function. In sections \ref{sec:Appl}, we specialize the convergence results to \Adam{}, AdaFom, Heavy Ball, Nesterov, Adagrad and RMSProp. In particular, Corollary \ref{thm:convergenceAdam} provides new results on the convergence of the dynamics of \Adam{}, while Corollary \ref{thm:convergenceRateNesterov} recovers previously known convergence rates of Nesterov accelerated method. We stress that \textbf{sections \ref{sec:MainResults} and \ref{sec:Appl} can be read independently}. In Section \ref{sec:ExistenceResult}, we prove existence and uniqueness of solutions at $t_0 = 0$ when the functions $h, p,q,r$
are not defined at zero.
In Section \ref{sec:considerations} we provide several empirical observations on adaptive algorithms which are inspired by the continuous analysis, and we collect guidelines for designing new adaptive algorithms in $\S\S$~\ref{ssec:Future}. Most proofs supporting the paper are postponed to the Appendix.

\subsection{Related work}

Gradient descent \cite{Cauchy:1847}, which only depends on the partial derivatives of $f$, is the simplest discrete algorithm to address the optimization
problem above
\begin{equation}\label{eq:gradient_descent}
\theta_{k+1} = \theta_k - s \nabla f(\theta_k).
\end{equation}
Another popular iterative approach to solve the above smooth optimization is the \emph{proximal} point algorithm \cite{Rockafellar, ParikhBoyd}
\begin{align}\label{eq:proximal}
\theta_{k+1} = \argmin_{u} \left(\frac{1}{2s}\norm{u-\theta_k}^2 + f(u)\right)
\end{align}
These discrete methods can be studied solely from the standpoint of optimization performance. It can be proved
that both algorithms converge to a critical point ($\nabla f(\theta_k) \to 0$ as $k \to \infty$) \cite{Nesterov2004} but also almost surely to a local minimizer \cite{LSJR:16, LPPSJR:2017}.
For convex functionals with globally Lipschitz gradient, both algorithms converges at linear rate $f(\theta_k) - f(\theta_{\star}) = \c{O}(1/(sk))$, where $\theta_{\star}$ is a minimal point of $f$ \cite{Nesterov2004, Rockafellar, ParikhBoyd}. These results give important guarantees on the convergence of each method.

For small and constant learning rate $s$, gradient descent \eqref{eq:gradient_descent} (resp. \emph{proximal} point algorithm \eqref{eq:proximal}) corresponds to the
\emph{forward} (resp. \emph{backward}) Euler's discretization of the gradient flow system
\begin{equation}\label{eq:gradient_flow}
\dot{\theta}(t) = - \nabla f(\theta(t)), \qquad \theta_0 = \theta(0),
\end{equation}
under the time scaling $t = ks$  \cite{Hairer:2006, SRBA:17}. The stable equilibria of this continuous system are given by the the set of strict (local) minima of the loss function $f$
and if the level sets of $f$ are bounded ($f$ coercive for example), then its trajectories asymptotically converge to a critical point in the
sense that $\nabla f(\theta(t)) \to 0$ as $t \to 0$.
Moreover for convex functions, a linear rate of convergence $f(\theta(t)) - f(\theta_{\star}) = \c{O}(1/t)$ holds, which is analogue to those obtained for
both gradient descent and proximal point algorithm.

The study of the continuous dynamical system is very useful. The well-behaved convergence properties of the gradient flow \eqref{eq:gradient_flow}
allows for different discretizations and optimization algorithms \cite{SRBA:17}. It, furthermore, provides valuable
intuition to prove convergence  of discrete systems: for example, continuous Lyapunov functional can be often adapted to the discrete counterparts.

\vspace{\baselineskip}

The rate of convergence for gradient descent is not optimal and depending on the class of functions $f$ belongs to,
more efficient algorithms can be designed \cite{NoceWrig06, Bonnans:2006, Boyd:2004, Nesterov2004, BubeckBook}.
For smooth convex or strongly convex functions, Nesterov  \cite{Nesterov2004} introduced an accelerated gradient algorithm which was proven to be optimal
(a lower bound matching an upper bound is provided)
\begin{align}\label{eq:Nesterov}
    v_{k+1} &= \theta_k - s \nabla f(\theta_k)\\
    \theta_{k+1} &= v_{k+1} + \frac{k}{k+3} (v_{k+1} - v_k).
\end{align}
However, the key mechanism for acceleration is not well understood and have many interpretations \cite{Bubeck2015AGA, HL:17,LRP:16}.
A particular interesting interpretation of acceleration is through the lens of a second order differential equation of the form
\begin{equation}\label{eq:second_order}
\ddot{\theta}  = a(t) \dot{\theta} + \nabla f(\theta), \qquad \theta(0) = \theta_0, \quad \dot{\theta}(0) =  \psi_0,
\end{equation}
where $t \mapsto a(t)$ is a smooth, positive and decreasing function of time, having possibly a pole at zero. Even if this singularity has important implications for the choice of the initial velocity $\psi_0$, we are more interested by the long term behavior of the solution to \eqref{eq:second_order} and hence at $\lim_{t \to \infty} a(t)$.
This system is called dissipative because its energy  $E(t) = \frac{1}{2} ||\dot{\theta} ||^2 + f(\theta) $ decreases over time.
Most accelerated optimization algorithms can be seen as the numerical integration of equation \eqref{eq:second_order}.
For the \emph{Heavy Ball method}, the function $a$ is constant and is called the damping parameter \cite{AABR:02,alvarez2000}.
In \cite{gadat2014, Cabot09, CEG:07}, conditions on the rate of decay of $a$ and its limit are given in order for the trajectories of
\eqref{eq:second_order} to converge to a critical point of $f$. This analysis highlights situations where \eqref{eq:second_order} are fit (or not) for optimization. Intuitively, if $a$ decays too fast to zero (like $1/t^2$) the system will oscillate and won't converge to a critical point.
The case $a(t) = 3/t$ was studied more specifically in \cite{SBC:15} and the authors draw interesting connections between \eqref{eq:second_order} and \emph{Nesterov's algorithm} \eqref{eq:Nesterov}.
The convergence rates obtained are $\c{O}(1/(s k^2 ))$ and $\c{O}(1 /t^2 )$ respectively, which match with the discrete algorithms by using the time identification $t = \sqrt{s}k$ \cite{SBC:15}.
Extension of this work are proposed in \cite{WW:15, WWJ:16} in which the authors studied acceleration from a different continuous equation having
theoretically exponential rate of convergence.
However, a na\"ive discretization looses the nice properties of this continuous system and current work consists on finding a better one preserving the symplectic structure of the continuous flow \cite{BJW:18}. 

By nature, first order adaptive algorithms have iterates that are non-linear functions of the gradient of the objective function.
The analysis of convergence is therefore more complex, potentially because the rate of convergence might depend on the function itself.
The first known algorithm \Ada{} \cite{DS:13} consists on multiplying the gradient by a diagonal preconditioning matrix, depending on previous squared gradients. The key property to prove the convergence of this algorithm
is that the elements of the preconditioning matrix are positive and non-decreasing  \cite{2018:ward, DS:13, CLSH:18}.
Later on, two new adaptive algorithms \RMSprop{} \cite{TH:12} and \Adam{} \cite{KB:14} were proposed.
The preconditioning matrix is an exponential moving average of the previous squared gradients.
As a consequence, it is no longer non-decreasing. The proof of convergence, relying on this assumption and given in the
form of a regret bound in \cite{KB:14}, is therefore not correct \cite{RKK:18}.
A new algorithm \Ams{} proposed in \cite{RKK:18} consists on modifying the preconditioning updates to recover this property.
While converging, this algorithm looses the essence of the \Adam's algorithm.
\Adam{} is such a mysterious algorithm that many works have been devoted to understand its behaviour. Variants of \Adam{} have been proposed \cite{ZMLW:17}
as well as convergence analysis towards a critical point \cite{BDMU:18, CLSH:18}. However, conditions for convergence seem very restrictive and not easy to verify in practice.
The existing proofs of convergence on the discrete system, moreover, feel rather technical at the cost of loosing the intuition.

\section{Presentation of the model}\label{sec:ModelPresentation}

We propose a continuous model to unify the analysis of gradient based algorithms. We first introduce notation on vector's operations used in the paper.
In section \ref{ssec:Model}, we present a general system of differential equations as well as a possible discretization of it.

\subsection{Compact notation}\label{ssec:CompactNotation}

In what follows, we use several times the same non standard operations on vectors. It is convenient to fix the notation of these operations. Given two vectors $u= (u_1, \ldots, u_d)$ and $v = (v_1,\ldots, v_d)$ of $\mathbb{R}^d$ and constants $a,\, \eps \in \mathbb{R}$, we use the following notation:
\[
\begin{aligned}
u + \eps &= (u_1 + \eps, \ldots, u_d + \eps)\\
u \odot v &= (u_1 \cdot v_1, \ldots, u_d \cdot v_d)\\
u/ v &= (u_1 / v_1, \ldots, u_d / v_d)\\
[u]^{a} &= (u_1^a, \ldots, u_d^a)\\
\sqrt{u}&= (\sqrt{u_1}, \ldots, \sqrt{u_d})\\
\end{aligned}
\]

\subsection{Presentation of the continuous time model}\label{ssec:Model}
Throughout this paper we study the following dynamical system 
 \begin{equation}\label{eq:ODE}
 \left\{
  \begin{aligned}
        \dot{\theta}(t) &= - m(t) / \sqrt{v(t)+\eps}\\
      \dot{m}(t) &=  h(t)\nabla f(\theta(t)) - r(t)m(t)  \\
      \dot{v}(t) &=  p(t) \left[ \nabla f(\theta(t))\right] ^2 - q(t)
      v(t),
  \end{aligned}
\right.
 \end{equation}
where $\eps \geq 0$, the functions $h(t)$, $r(t)$, $p(t)$ and $q(t)$ are $C^1$-functions defined over $\mathbb{R}_{>0}$ and $(\theta, m, v,t) \in \mathbb{R}^{d} \times \mathbb{R}^{d} \times \mathbb{R}^{d}_{\geq 0} \times \mathbb{R}_{>0}$; if $\eps=0$, then $v \in \mathbb{R}_{>0}^d$. The above system has a momentum term $m$ and a memory term $v$. The system \eqref{eq:ODE} is supplemented with initial conditions
$\x_0 =(\theta_0, m_0, v_0)$ at time $t=t_0 \geq 0$. We denote by $\x(t)=\x(t,t_0,\x_0) =(\theta(t),m(t),v(t))$ a solution of $\eqref{eq:ODE}$ with initial condition $\x(t_0,t_0,\x_0)=\x_0$, and interval of definition $t\in [t_0,t_{\infty}[$ (extra care needs to be taken when $t_0=0$ in order to guarantee the existence of a solution, see Theorem \ref{th:existence}).

When the system does not contain a momentum term, we also consider the alternative simpler system
\begin{equation}\label{eq:ODEbis}
 \left\{
  \begin{aligned}
        \dot{\theta}(t) &= - \nabla f(\theta) / \sqrt{\omega(t)+\eps}\\
      \dot{\omega}(t) &=  p(t) \left[ \nabla f(\theta(t))\right] ^2 - q(t)\omega(t),
  \end{aligned}
\right.
 \end{equation}
whose analysis is similar (and simpler) to the first (see $\S\S$~\ref{ssec:ConvergenceBis}), but can not be derived from the first. Unless stated otherwise, our presentation deals with the system \eqref{eq:ODE}, and it is later specialized to \eqref{eq:ODEbis}.

We always make the following hypotheses.

\begin{Assumption}\label{ass:ODE}
The objective function $f$ is assumed to be a $C^2$ function defined in $\mathbb{R}^d$. The functions $h$, $r$, $p$ and $q$ are non-negative and non-increasing $C^1$-functions defined over $\mathbb{R}_{>0}$, and $h(t) \not\equiv 0$, $r(t)\not\equiv 0$. We also require that one of the following is satisfied:
\begin{itemize}
\item Either $ p(t) \not\equiv 0$, in which case we say that the system is \emph{adaptive};
\item Or $p(t)\equiv q(t) \equiv 0$, in which case we say that the system is \emph{non-adaptive}.
\end{itemize}
\end{Assumption}

It seems reasonable to imagine that there are several different choices of functions $h$, $r$, $p$ and $q$ which yield a good optimization algorithm. We provide a list of interesting cases in the table \ref{tab:1} and \ref{tab:2} below. Each choice, furthermore, might be adapted to some different ``class'' of loss function, that is, loss functions satisfying some extra property (e.g. loss functions coming from a single, double or $N$-layers in deep learning; convex functions; globally Lipschitz functions, etc). We allow, therefore, the functions $h$, $r$, $p$ and $q$ to be as general as possible, so that practitioners may test different combinations of coefficients, probably depending on the properties of the loss function. Some guidelines to chose these coefficients are provided $\S\S$~\ref{ssec:Future}.

An adaptive system has a non-trivial dynamic for the memory term $v$, which changes the rate in the dynamics of the main variable $\theta$. In practice, this corresponds to an automatic change in the learning rate, controlled only by the history of the trajectory $\x(t)$ and the loss function $f$. This justifies the name \emph{adaptive} and supports the intuition that the efficiency of the algorithm depends on the properties of the loss functions $f$ (c.f. $\S\S$~\ref{ssec:Flat}). We also consider a momentum term $m$, which allow us to accelerate the convergence of the algorithm depending on the choice of the coefficients $h$ and $r$. Intuitively, the addition of the momentum $m$ implies the existence of a special energy functional for ODE \eqref{eq:ODE} (see equation \eqref{eq:EnergyPrelim} below) which works as a ``funnel" around minimum points of $f$. The trajectories of the ODE can, therefore, be ``accelerated" without the risk of diverging to $\infty$, at the cost of an oscillatory behaviour (just as water in a funnel). In particular, this intuitively explain why it is possible to improve the rate of convergence of the Heavy ball in Nesterov. Based in these two intuitions, we can expect that the choice of $h$ and $r$ controls how fast these algorithms converge in general, while $p$ and $q$ may only slow down the algorithm in general, but accelerate it for certain ``classes" of loss functions. This intuition turns out to be precise when dealing with convex functions, as we discuss in $\S\S$~\ref{ssec:Future}.

Note that we do not suppose that $\nabla f$ is globally Lipschitz. Indeed, such an assumption is insufficient to guarantee that the ODE \eqref{eq:ODE} is itself globally Lipschitz (e.g. $f(\theta) = \theta^2/2$ implies $\nabla f(\theta)^2 = \theta^2$). The lack of this  assumption constitutes an important technical difficulty in the remaining of the paper. Under the extra hypothesis that $f$ is coercive, nevertheless, we are able to replace this hypothesis effectively (see Lemma \ref{lem:JustifiesA4} below).

\bigskip

In order to establish a relation between the continuous and the optimization algorithms,
we study the finite difference approximation of \eqref{eq:ODE} by the forward Euler method
 \begin{equation}\label{eq:discrete}
 \left\{
  \begin{aligned}
      \theta_{k+1} &= \theta_{k} - s  m_k / \sqrt{v_k + \eps}\\
      m_{k+1} &=  (1 - s r(t_{k+1})) m_{k} + s h(t_{k+1})\nabla f(\theta_{k+1})
      \\
      v_{k+1} &=  (1 - s q(t_{k+1}) )v_{k}  + s p(t_{k+1}) \left[ \nabla
      f(\theta_{k+1}) \right] ^2
  \end{aligned}
\right.
 \end{equation}
where $t_k = ks$. We chose this method because it fits well with \Adam{} discrete system. As we will see, the approximation error between the discrete system \eqref{eq:discrete} and continuous system \eqref{eq:ODE} tends to zero (with order one) when the learning rate goes to zero (see Theorem \ref{thm:ConvergenceDiscretizationRate} for a precise result). However this choice of discretization is of course non-unique, and more efficient quadrature rules could lead to more accurate numerical integration \cite{QuadratureRule:75, kythe2002computational}. The connections between our model and the discrete optimization algorithms is summarized by tables \ref{tab:1} and \ref{tab:2}, and the proof of these relations is postponed to Section \ref{sec:Appl}.

\begin{table}

\center
\begin{tabular}{| l | c |  c | c | c |}
\hline
\textbf{Equation \eqref{eq:ODE}} & $h(t)$ & $r(t)$ & $p(t)$ & $q(t)$ \\ \hline
ADAM & $\frac{\displaystyle 1- e^{-\lambda \alpha_1}}{\displaystyle \lambda(1- e^{- t \alpha_1})}$ &  $\frac{\displaystyle 1- e^{-\lambda \alpha_1}}{\displaystyle \lambda(1- e^{- t \alpha_1})}$&  $\frac{\displaystyle 1- e^{-\lambda \alpha_2}}{\displaystyle \lambda(1- e^{- t \alpha_2})}$& $\frac{\displaystyle 1- e^{-\lambda \alpha_2}}{\displaystyle \lambda(1- e^{- t \alpha_2})}$ \\ \hline
  ADAM (without rescaling) & $1/\alpha_1$ & $1/\alpha_1$ & $1/\alpha_2$ & $1/\alpha_2$\\ \hline
AdaFom  & $1/\alpha_1$ & $1/\alpha_1$ & $1/t$ & $1/t$ \\ \hline
Heavy Ball &1 & $\gamma$ & 0&0\\ \hline
 Nesterov &1 & $r/t$ & 0&0\\ \hline
\end{tabular}
\caption{\label{tab:1}where $r$, $\gamma$, $\alpha_1$, $\alpha_2$, $\lambda$, $h$, $p$ and $q$ are all constants, whose precise connection with the hyper-parameters of the algorithms is made clear in $\S\S$ \ref{ssec:Adam}, \ref{ssec:AdaForm}, \ref{ssec:HeavyBall} and \ref{ssec:Nesterov} below.}
\end{table}

\begin{table}
\centering
\begin{tabular}{| l | c |  c |}
\hline
 \textbf{Equation \eqref{eq:ODEbis}} & $p(t)$ & $q(t)$ \\ \hline
 Adagrad & 1 & 0 \\ \hline
  RMSProp& $1/\alpha_2$ & $1/\alpha_2$ \\ \hline
  Test Case & $p$ & $q$ \\ \hline
\end{tabular}

\caption{\label{tab:2}where $\alpha_2$, $p$ and $q$ are all constants, whose precise connection with the hyper-parameters of the algorithms is made clear in $\S\S$~\ref{ssec:Adagrad} and \ref{ssec:RMSProp} below.}
\end{table}

\subsection{An Energy functional of ODE \eqref{eq:ODE} and a natural assumption}\label{ssec:Energy}

A crucial property in the study of ODE \eqref{eq:ODE} is the existence of an energy functional, which is inspired from \cite[Theorem 2.1]{alvarez2000}:

\begin{equation}\label{eq:EnergyPrelim}
\begin{aligned}
    E(t,\theta,m,v) &= f(\theta) +\frac{1}{2
    h(t)}\norm{\frac{m}{\left[v+\eps\right]^{1/4}} }^2.
\end{aligned}
\end{equation}

This functional plays a crucial role in the study of the convergence of ODE \eqref{eq:ODE} in $\S$\ref{sec:MainResults}. We start by computing the derivative of the energy functional:
\[
\begin{aligned}
    \frac{d}{dt} E(t,\theta,m,v) =& - \frac{1}{h(t)}  \left( r(t) +
    \frac{h'(t)}{2h(t)} \right) \norm{\frac{m}{\left[v+\eps\right]^{1/4}}}^2 + \sum_{i=1}^d \frac{ m_i^2 \{q(t)v_i-p(t)\cdot \left[\partial_{\theta_i}f(\theta)\right]^2\}}{4 h(t) \cdot(v_i+\eps)^{3/2}},
\end{aligned}
\]
and, it easily follows from the fact that $p(t), \, h(t),\, v_i,\, \partial_{\theta_i}f(\theta)^2 \geq 0$, that:
\begin{equation}\label{eq:EnergyPrelimDerInequality}
\begin{aligned}
\frac{d}{dt} E(t,\theta,m,v) \leq -\frac{1}{2h(t)} \left[2r(t)- \frac{q(t)}{2}
    +\frac{h'(t)}{h(t)} \right]\norm{\frac{m}{\left[v+\eps\right]^{1/4}}}^2.
\end{aligned}
\end{equation}
This leads us to the following natural hypothesis, which is assumed almost everywhere in the paper:

\begin{Assumption}\label{ass:NecessaryCoeficients}
There exists $\tilde{t}>0$ such that for every $t >\tilde{t}$ we have that:
\[
2 r(t) - \frac{q(t)}{2} + \frac{h'(t)}{h(t)} \geq 0.
\]
\end{Assumption}

In practice this is a mild assumption in the hyper-parameters of the model. In terms of the algorithms in table \ref{tab:1}, it is always verified by AdaForm, Heavy Ball (with $\tilde{t}=0$) and Nesterov (with $\tilde{t}=0$); for Adam (and Adam with rescaling) it leads to the following condition on the hyper-parameters (which is usually respected by practitioners): $3 + \beta_2 > 4\beta_1$. Now, under assumption \ref{ass:NecessaryCoeficients} the derivative of $E(t,\theta,m,v)$ is non-positive, which immediately yields the following result.

\begin{lemma}\label{lem:JustifiesA4}
Suppose that assumptions \ref{ass:ODE} and \ref{ass:NecessaryCoeficients} are verified. Given a solution $\x(t)$ of the ODE \eqref{eq:ODE} such that $t_0\geq \tilde{t}$, we have that:
    \[
    f(\theta(t)) \leq E(t_0,\x_0), \quad \forall \, t \in [t_0,t_{\infty}[
       \]
In particular, if $f$ is coercive, then the curve $\theta(t)$ is bounded. Furthermore, if $f$ is a function bounded from below, say by $f_{\ast}$, then:
\[
    \frac{1}{2h(t)} \norm{\frac{m}{\left[v+\eps\right]^{1/4}}}^2 \leq E(t_0,\x_0) + f_{\ast}, \quad \forall \, t \in [t_0,t_{\infty}[
       \]
\end{lemma}

The above result shows the importance of the energy functional and has important implications (together with Theorem \ref{th:existencePrel} below). For example, if $f$ is coercive we guarantee that $\theta(t)$ is bounded, a necessary condition for the convergence of a solution of ODE \eqref{eq:ODE}.

\begin{remark}[On Assumption \ref{ass:NecessaryCoeficients}]\label{rk:OnAssNecessaryCoeficients}
Suppose that $p(t) \equiv q(t) \equiv 0$, in which case inequality \eqref{eq:EnergyPrelimDerInequality} is an equality, and consider the extreme opposite of assumption \ref{ass:NecessaryCoeficients}, that is
\[
2 r(t) + \frac{h'(t)}{h(t)} < 0, \qquad \forall \, t>0 , \quad \text{(e.g. $r(t) = 1/4t$ and $h(t) = 1/t$).}
\]
Then the differential equation \eqref{eq:ODE} does not converge to a minimum  because the energy functional \eqref{eq:EnergyPrelim} is always increasing.
\end{remark}

\subsection{Basic Properties of ODE \eqref{eq:ODE}}\label{ssec:Basic}

In order to continue our analysis of the asymptotic behaviour of solutions of the ODE \eqref{eq:ODE}, it is essential to understand their domain of definition, and to clarify the difference between continuous and discrete systems. We start by providing two results in this direction, under the assumption that the initial $t_0$ is strictly positive. First of all, under mild assumptions, we are able to guarantee that all solutions of the differential equation are defined on the interval $[t_0,\infty[$.

\begin{theorem}[Existence and uniqueness for $t_0>0$]\label{th:existencePrel}
Suppose that the ODE \eqref{eq:ODE} satisfies assumptions \ref{ass:ODE}, and that either $p(t)\not\equiv0$, or $f$ is bounded from below and assumption \ref{ass:NecessaryCoeficients} with $\tilde{t}=0$ is satisfied. Then for any $t_0 > 0$ and admissible initial condition $\x_0$, there exists a unique global solution to equation \eqref{eq:ODE} such that:
 \[
 \begin{aligned}
 \theta &\in C^2([t_0,\infty); \R^d) \quad  \text{ and } \quad m,\,v &\in C^1([t_0,\infty); \R^d).
 \end{aligned}
\]
\end{theorem}

The proof of this result is postponed to the the appendix $\S\S$\ref{sec:Existence}. In the case that assumption \ref{ass:NecessaryCoeficients} is satisfied, a very simple proof is provided in $\S\S$\ref{ssec:ExistenceSimplified}. A stronger version of this result (which includes the case $t_0 =0$) is given in Theorem \ref{th:existence} below, but we postpone its discussion to $\S$~\ref{sec:ExistenceResult} in order to keep the initial presentation as simple as possible.

We now study the validity of the approximation given by the discrete system \eqref{eq:discrete} of the differential equation \eqref{eq:ODE}, that is, we study (what is called in the literature) the ``convergence rate'' of the numerical approximation. In this work, we replace the name ``convergence rate'' by \emph{approximation rate} in order to avoid a possible confusion with the \emph{convergence of the orbits} of the system.

Let us fix the notation: consider a (final) time $T>t_0>0$ and an interval $[t_0,T]$ (on which we will study the approximation of the solution of \eqref{eq:ODE}). We recall that the step size is given by $s>0$. Consider $K_{t,s}= [(t-t_0)/s]$, the integer part of $(t-t_0)/s$. We write $t_k = t_0 + k s$ for any $k\in [\![  0, K_{t,s}  ]\!]$ and $\pi = \left\{ t_0 < t_1 < \cdots < t_K = T \right\}$ is a homogeneous partition of the interval $[t_0,T]$. Now, let $\x_0= (\theta_0,m_0,v_0)$ be a fixed admissible initial condition and consider:
\begin{itemize}
\item The sequence $\xe_k:=\x(t_k)$ for all $k\in \mathbb{N}$, where $\x(t)$ is the exact solution of ODE \eqref{eq:ODE} with initial condition $\x(t_0)=\x_0$.
    \item The sequence $(\x_k)_{k \in \mathbb{N}}$ given by the discrete system \eqref{eq:discrete} with initial condition $\x_0$;

\end{itemize}

Roughly, we are interested in estimating the distance between $(\x_k)$ and $(\xe_k)$ in terms of the learning rate $s$. In what follows, we prove that the approximation rate between the continuous and discrete dynamics is of order one. In particular, when the learning rate goes to zero, then the continuous and discrete dynamics tend to be equal. More precisely:

\begin{theorem}[Approximation rate]\label{thm:ConvergenceDiscretizationRate}
Suppose that the ODE \eqref{eq:ODE} satisfies assumptions \ref{ass:ODE} and that $p(t)\not\equiv0$. Let $T>0$ and $t_0>0$ and consider a compact set $A_0$ of $\mathbb{R}^{d}_{\geq 0}\times \mathbb{R}^d_{\geq 0} \times \mathbb{R}_{>0} $. Then, there exists a constant $C(T,A_0)$ (which only depends on $T$ and the compact $A_0$) such that for any admissible initial condition $\x_0\in A_0$, the numerical scheme satisfies
      \[
  \max_{k = 0, \ldots, K} \norm{\x_k -  \xe_k} \leq C(T, A_0) \cdot s.
        \]

\end{theorem}
The case $p=0$ corresponds to Gradient descent, Heavy Ball and Nesterov and similar results hold.

\section{Convergence analysis}\label{sec:MainResults}

In this section, we study the asymptotic behaviour of the solutions of \eqref{eq:ODE}.
Our analysis is divided in the following three steps:
\begin{itemize}
\item[(1)] \emph{Topological convergence}: Find sufficient conditions on the functions $f$ and $p,q,r,h$ in order for the solutions of equation \eqref{eq:ODE} to converge to a critical value of $f$, that is, $\nabla f(\theta(t)) \to 0$ when $t\to \infty$. In particular we do not require $f$ to be convex.

\item[(2)] \emph{Avoiding local maximum and saddles}: We want to strengthen the result of part (1) and give sufficient conditions so that the dynamics
avoid local maximum and saddles and only converge to local minimum. In other words, fix $t_0 > 0$ and denote by $S_{t_0}$ the set of initial conditions $\x_0=(\theta_0,m_0,v_0)$ such that the limit set of the associated solution $\theta(t)$ contains a critical point $ \theta_{\star}$ which is \emph{not} a local minimum.
We give, in subsection \ref{ssec:avoid}, sufficient conditions for the set $S_{t_0}$ to have Lebesgue measure zero.

\item[(3)] \emph{Rate of convergence}: Under the convexity assumption, find the rate of convergence of $f$ to a local minimum.
\end{itemize}

In the remaining of this section, we give precise statements for all of the three steps, and we will make appropriate assumptions on the objective function.
However, the following assumption is over-arches the three analysis:

\begin{Assumption}\label{ass:bounded}
The solution $\theta(t)$ of the ODE \eqref{eq:ODE} is bounded.
\end{Assumption}

This assumption is always automatically satisfied when $f$ is coercive, as it was remarked in Lemma \ref{lem:JustifiesA4}. In general, it is possible to decide when an orbit $\x(t)$ satisfies assumption \ref{ass:bounded} by algebraic manipulations in terms of its initial condition $\x_0$ and the energy functional \eqref{eq:EnergyPrelim}.

\subsection{Topological convergence}\label{ssec:TopConvergence}

In this part, we make an additional assumption on the asymptotic behaviour of the coefficients, which is designed to simplify the proof while still covering \Adam{}, and most adaptive algorithms:

\begin{Assumption}\label{ass:TopologicalConvergence}
Suppose that $\eps>0$. Consider the functions:
\[
H(s) = h(1/s), \quad R(s) = r(1/s), \quad P(s) = p(1/s), \quad Q(s) = q(1/s),
\]
and suppose that these functions are $C^1$ in $[0,\infty)$, $H(0)>0$ and $4R(0)>Q(0)$.
\end{Assumption}

Note that Assumption \ref{ass:TopologicalConvergence} is satisfied, essentially, when the coefficients of $h(t)$ and $r(t)$ do not converge to zero at infinity.
Hence, it holds for \Adam{}, AdaForm, and the Heavy ball differential equations, c.f. table \ref{tab:1}. It also has the interesting feature of being almost completely \emph{independent from the functions $p(t)$ and $q(t)$}, a flexibility which should be explored when trying to design new algorithms. Under this assumption, we prove the convergence of the dynamics in the following sense:

\begin{theorem}[Topological Convergence]\label{thm:convergence}
Suppose that assumptions \ref{ass:ODE}, \ref{ass:NecessaryCoeficients}, \ref{ass:bounded} and \ref{ass:TopologicalConvergence} are verified. Then $f(\theta(t)) \to f_{\star}$ and $m(t) \to 0$ when $t\to \infty$, where $f_{\star}$ is a critical value of $f$. Furthermore, if either $Q(0) >0$ or $p(t)\equiv q(t)\equiv 0$ and $v_0=0$, then $v(t) \to 0$.
\end{theorem}

The proof of Theorem \ref{thm:convergence} is postponed in Appendix \ref{sec:ProofTopologicalConvergence}.
Our method is inspired by the work of Alvarez \cite{alvarez2000}, based on the energy functional of the system \eqref{eq:EnergyPrelim}. We use elementary topological techniques of qualitative theory of ODE's (\`{a} la Poincar\'{e}-Bendixson), which are recalled in $\S\S$\ref{ssec:PoincareBendixson}. At the one hand, this approach avoids most estimates and analytical arguments, which are typically necessary in this kind of study, and can be easily reproduced in other systems. As an immediate advantage, we do not need assumptions such as convexity of the loss function or globally Lipschitz properties of the differential equation. At the other hand, the assumption is not optimal. For example, it is not satisfied by Nesterov's acceleration equation \eqref{eq:ODENesterov}. We believe that the optimal threshold to guarantee convergence of ODE \eqref{eq:ODE} should be given by an inequality in terms of poles of order at most one for the functions $H$ and $R$.
The idea is supported in \cite{Gadat:long_term} which shows that the function $R$ can not be a polynomial function of order bigger than $1$ in the case of the dissipative system related to accelerated dynamics.

\subsection{Avoiding local maximum and saddles}\label{ssec:avoid}

In this section, we make the following extra assumption:
\begin{Assumption}\label{ass:Morse}
A critical point $\theta_{\star}$ of $f$ is either a local-minimum or it satisfies the two following properties:
\begin{itemize}
\item[(a)] it is a strict saddle (following \cite[Definition 1]{LPPSJR:2017}), that is, there exists a strictly negative eigenvalue of the Hessian $\mathcal{H}_f(\theta_{\star})$ of $f$ at $\theta_{\star}$.
\item[(b)] is it an isolated critical point, that is, there is a neighbourhood $U$
around $\theta_{\star}$ that does not contain any other critical points.
\end{itemize}
\end{Assumption}

Now, fix a time $t_0>0$ and recall that the topological limit of a curve $\theta(t)$, called $\omega$-limit, is given by:\footnote{Let $S \subset \mathbb{R}^d$ be a set. We denote by $\overline{S}$ its closure, that is, the smallest closed set which contains $S$.}\footnote{Let $\lambda_0\in \mathbb{R}$, and suppose that for each $\lambda >\lambda_0$ there exists a set $S_{\lambda} \subset \mathbb{R}^d$. We denote by $\bigcap_{\lambda>\lambda_0} S_{\lambda}$ the intersection of all sets $S_{\lambda}$ with $\lambda>\lambda_0$.}
\[
\omega(\theta(t)) = \bigcap_{\tau> t_0} \overline{\theta([\tau,\infty))}.
\]
Consider the set of initial conditions such that the limit set of the associated orbit contains a critical point which is not a local minimum
\[
S_{t_0} := \{\x_0=(\theta_0,m_0,v_0);\, \omega(\theta(t))  \ni \theta_{\star}, \text{ where } \theta_{\star}\, \text{ is a strict saddle}\} 
\]

The main result of this subsection is the following:

\begin{theorem}[Avoiding Saddle and Local Maximum points]\label{thm:avoiding}
Suppose that assumptions \ref{ass:ODE}, \ref{ass:NecessaryCoeficients}, \ref{ass:bounded}, \ref{ass:TopologicalConvergence} and \ref{ass:Morse} are satisfied. If either $Q(0) >0$ or $p(t)\equiv q(t)\equiv 0$, then the set $S_{t_0}$ has Lebesgue measure zero for every $t_0>0$.
\end{theorem}

It follows that, if $\x_0 =(\theta_0,m_0,v_0)$ is a random initial condition, then the solution $\x(t,t_0,\x_0) = (\theta(t),m(t),\theta(t))$ converges to a local minimum of $f$ with total probability. Similar results are proved for discrete systems having isolated critical points in \cite{LSJR:16, LPPSJR:2017}, using essentially the same method as in here. More precisely, we use the theory of \emph{central-stable} manifold (for vector-fields), which is recalled in $\S\S$\ref{ssec:CentralStable}. We have supplemented our analysis by treating the case of the usual gradient flow in $\S\S$\ref{ssec:GradientFlowExampleAvoiding}, and we hope that this will help with the dissemination of the technique. We finish this section by providing a technical discussion on the Assumption \ref{ass:Morse}:

 \begin{remark}[On Assumption \ref{ass:Morse}]\label{rk:assMorse}\hfill
\begin{enumerate}

 \item Assumption \ref{ass:Morse}$(a)$ was introduced in \cite{LPPSJR:2017} and has crucial technical consequences. It allow us to use the center-stable manifold theory recalled in $\S\S$\ref{ssec:CentralStable}. Without this hypothesis, the singular points of the ODE \eqref{eq:ODE} at infinite (see equation \eqref{eq:VectorFieldInfinity}) can be arbitrarily degenerated, and there is no general singularity theory to treat these points in dimension higher than three. In order to relax such a hypothesis, it is necessary develop specific singularity techniques for equation \eqref{eq:ODE}, and we intend to pursue this direction in a future paper.

\item Assumption \ref{ass:Morse}$(b)$ allow us to exclude pathological differences between local and global center-stable manifold theory (see example \ref{ex:pathology}). An alternative to this hypothesis, is to add a globally Lipschitz assumption onto the system \eqref{eq:ODE}, and to study the relation between the Lipschitz approximation and the Hessian of the loss function $f$ (which would allow us to use the strong global result \cite[Ch. 1 Thm 1.1]{CMBook}). This is essentially what is done in \cite{PP:2016}, where the authors study the analogue problem in for a simpler ODE without assumption \ref{ass:Morse}$(b)$; more precisely, they crucially show that their system satisfies conditions that replace the globally Lipschitz assumption. We understand that a study in the generality of ODE \eqref{eq:ODE} without condition \ref{ass:Morse}$(b)$ would demand the development of specific singularity techniques for equation \eqref{eq:ODE}, and we intend to pursue this direction in a mathematical paper.
 \end{enumerate}
 \end{remark}

\subsection{Rate of convergence}\label{ssec:RateConvergence}

The study of the rate of convergence of $f(\theta(t))$ to the minimum value $f(\theta_{\star})$ usually relies on a convexity assumption and a Lyapunov energy functional (see \cite{SBC:15, AABR:02, alvarez2000, gadat2014}). It is natural to assume in this section, therefore:

\begin{Assumption}\label{ass:fConvex}
The function $f$ is convex and admits a minimum point, that is, there exists $\theta_{\star}$ such that $f(\theta) \geq f(\theta_{\star})$ for every $\theta \in \mathbb{R}^d$.
\end{Assumption}

Now, strictly speaking, we do not find a Lyapunov functional for \eqref{eq:ODE}, but a natural functional which allow us to prove convergence to a least a neighbourhood of a local
minimum. For accelerated methods, the proposed functional corresponds to the standard Lyapunov energy used in many other works c.f. \cite{SBC:15, AABR:02, alvarez2000, gadat2014}. More precisely, let $t_0 > \tilde{t}$ (as defined in assumption \ref{ass:NecessaryCoeficients}) and consider the following functions
\[
\begin{aligned}
\c{A}(t) &= \int_{t_0}^t h(s) \c{B}(s) ds \\
\c{B}(t)  &=   e^{\int^t_{t_0} r(s) ds} \int_{t}^{\infty}
    e^{-\int^s_{t_0} r(u) du}ds
\\
    \c{C}(t) & =   \frac{1}{h(t)}\int_{t_0}^t h(s) \c{B}(s) ds.
\end{aligned}
\]
The expressions of $\c{A}(t)$, $\c{B}(t)$ and $\c{C}(t)$ are simple to compute for all the expressions in table \ref{tab:1}, as we show in $\S$\ref{sec:Appl}. Note, furthermore, that these functions only depend on $h$ and $r$ (which re-enforce the heuristic that $p$ and $q$ can be chosen in a very flexible way). We are ready to introduce the energy functional used in this section:
\begin{equation}\label{eq:Lyapunov}
    \c{E}(t,m,v, \theta) = \c{E}_1(t,\theta) + \c{E}_2(t,m,v, \theta)
\end{equation}
where
\[
\begin{aligned}
    \c{E}_1(t,\theta) &= \c{A}(t) \left( f(\theta) - f(\theta_{\star})\right)\\
    \c{E}_2(t,m,v,\theta) &= \frac{1}{2}  \norm{\left[v+\eps\right]^{1/4}\left(\theta -
    \theta_{\star} \right)}^2 - \c{B}(t) \ps{\theta-\theta_{\star}}{m}+\frac{\c{C}(t)}{2}\norm{\frac{m}{\left[v+\eps\right]^{1/4}}}^2.
\end{aligned}
\]
We now need the following assumptions in order to control the behaviour of this functional:

\begin{Assumption}\label{ass:necessaryConditions}
We make the following two assumptions:
\begin{itemize}
\item[(a)] $\lim_{t\to \infty} \int_{t_0}^{t} e^{-\int^s_{t_0} r(u) du} ds <+\infty$
\item[(b)]There exists $\tilde{t}>t_0$ such that for all $t \geq \tilde{t}$
\[
    \begin{aligned}
    & \c{B}^2(t) \leq \c{C}(t)\\
    &  3\c{B}(t) \leq \c{C}(t)\left(  2r(t)  - \frac{q(t) }{2} +  \frac{ h'(t)}{h(t)}
    \right).
\end{aligned}
\]
\end{itemize}
\end{Assumption}

Note that Assumption \ref{ass:necessaryConditions}(a) is necessary for the function $\c{B}(t)$ to be well-defined, and that imposes an important constraint in the asymptotic behaviour of the function $r(t)$. More precisely, the limit $\lim_{t\to \infty}t^{1+\epsilon}r(t)$ must be zero for every $\epsilon>0$, which implies that $r(t)$ has at most a pole of order $1$ at infinity. Assumption \ref{ass:necessaryConditions}(b) provides the asymptotic control on the derivative of the energy functional \eqref{eq:Lyapunov}, and should be compared with Assumption \ref{ass:NecessaryCoeficients}. Once again, note that it is independent of the function $p$, and almost independent of $q$. We are now ready to state the main theorem of this section:

\begin{theorem}\label{thm:convergenceRate}
We assume that assumptions \ref{ass:ODE}, \ref{ass:NecessaryCoeficients},  \ref{ass:fConvex} and \ref{ass:necessaryConditions} are all satisfied.
Then for all $t \geq \tilde{t}$, where $\tilde{t}$ is given in Assumption \ref{ass:NecessaryCoeficients}, we have
\[
 f(\theta)  - f(\theta_{\star})  \leq \frac{1}{4\c{A}(t)}  \left[4 \c{E}(\tilde{t},    m(\tilde{t}), v(\tilde{t}),\theta(\tilde{t}))
    +    \int_{\tilde{t}}^t  p(u) \ps{\frac{\left[\nabla f(\theta)\right]^2}{\left[v+\eps\right]^{1/2}}}{\left[\theta - \theta_{\star}\right]^2} \,
     du  \right]
\]
where $\mathcal{E}(t,m,v,\theta)$ is the Lyapunov functional \eqref{eq:Lyapunov}.
Furthermore, under assumption \ref{ass:bounded}, suppose that either $\lim_{t\to \infty}  p(t) / q(t)  < \infty  \text{ or } p(t)\equiv q(t)\equiv 0$. Then, there exist two positive and finite constants $\c{K}_1$ and $\c{K}_2$ (which depend on $f$, $\theta_0$, $v_0$ and $\eps$) such that for all $t \geq \tilde{t}$:
    \[
 f(\theta(t))  - f(\theta_{\star})  \leq \frac{1}{\c{A}(t)} \left[\c{E}(\tilde{t},    m(\tilde{t}), v(\tilde{t}),\theta(\tilde{t}))
    +  \c{K}_1 + \c{K}_2  \int_{\tilde{t}}^t q(u) du \right].
\]
It follows that the ODE \eqref{eq:ODE} converges to the minimum point with rate of convergence of order at least:
\[
\begin{aligned}
    & \max\left\{1, \int_{t_0}^t q(u)  du\right\} / \c{A}(t).
    \end{aligned}
\]
\end{theorem}
From the previous theorem, we observe that the rate of convergence to the global minimum is determined by several factors
\begin{itemize}
\item the choices of $h, r$ which in turns define the function $\c{A}$.
\item the choices of $p, q$ which degrade the rate of convergence but allow for more flexibility
for the choice of $h,r$ in assumption \ref{ass:NecessaryCoeficients}.
\item the history of the dynamics.
\end{itemize}

It follows from the first inequality of Theorem \ref{thm:convergenceRate} that the rate of convergence of ODE \eqref{eq:ODE} depends in an essential way from the asymptotic behaviour of the term:
\begin{equation}\label{eq:CrucialTerm}
\left\| \frac{\nabla f(\theta)}{\left[v+\eps\right]^{1/2}} \right\|.
\end{equation}
This has also been independently remarked in \cite{CLSH:18}. In the later, the authors propose the algorithm AdaFom, whose coefficients are chosen in such a way that the sum of the terms \eqref{eq:CrucialTerm} is telescopic. They can, therefore, be controlled in an easy way, which allows one to obtain convergence results only under a globally Lipschitz assumption. Note that our Theorem \ref{thm:convergenceRate} recovers the expected (in the deterministic setting) rate of convergence of AdaFom, see Corollary \ref{thm:convergenceAdaFom}.

In general, nevertheless, the term \eqref{eq:CrucialTerm} can not be controlled in the same way. The second inequality of Theorem \ref{thm:convergenceRate} controls the term \eqref{eq:CrucialTerm} only in terms of the functions $h$, $r$ and $q$. We do not know if this control is optimal, but we derive some surprising features from it which we discuss in $\S\S$~\ref{ssec:Adam}, e.g. it supports practitioners usual choice of hyper-parameters for \Adam{}.

\subsection{Convergence analysis of ODE \eqref{eq:ODEbis}}\label{ssec:ConvergenceBis}

It is possible to reproduce Theorems \ref{thm:convergence}, \ref{thm:avoiding} and \ref{thm:convergenceRate} in the case of ODE \eqref{eq:ODEbis} following the same methods (but in an easier way) presented in the Appendixes  \ref{sec:ProofTopologicalConvergence}, \ref{sec:ProofAvoiding} and \ref{sec:ConvergenceRate}, provided that some changes are made in the energy functional and assumptions. More precisely, we consider the following energy functional:
\[
\begin{aligned}
E(t,\theta,\omega) &= f(\theta) \quad \implies \quad \frac{d}{dt}E(t,\theta,\omega) = - \norm{\frac{\nabla f(\theta)}{\sqrt[4]{\omega+\epsilon}}}^2 \leq 0.
\end{aligned}
\]
It follows that Assumptions \ref{ass:NecessaryCoeficients} and \ref{ass:TopologicalConvergence} are unnecessary! By the same reason, furthermore, Assumption \ref{ass:bounded} is always verified when $f$ is coercive, c.f. Lemma \ref{lem:JustifiesA4}. This implies that Theorem \ref{thm:convergence} and \ref{thm:avoiding} admit admit the exact same formulation for ODE \eqref{eq:ODEbis} without Assumptions \ref{ass:NecessaryCoeficients} and \ref{ass:TopologicalConvergence}, that is:

\begin{theorem}\label{thm:convergenceBis}
Suppose that assumptions \ref{ass:ODE} and \ref{ass:bounded} are satisfied for ODE \eqref{eq:ODEbis}. Then $f(\theta(t)) \to f_{\star}$ when $t\to \infty$, where $f_{\star}$ is a critical value of $f$. Furthermore, if either $Q(0) >0$ or $p(t)\equiv q(t)\equiv 0$ and $\omega_0=0$, then $\omega(t) \to 0$.
\end{theorem}

\begin{theorem}\label{thm:avoidingBis}
Suppose that assumptions \ref{ass:ODE}, \ref{ass:bounded} and \ref{ass:Morse} are satisfied for ODE \eqref{eq:ODEbis}. If either $Q(0) >0$ or $p(t)\equiv q(t)\equiv 0$, then the set $S_{t_0}$ has Lebesgue measure zero for every $t_0>0$.
\end{theorem}

Finally, under the convexity assumption \ref{ass:fConvex}, we consider the energy functional:
\[
\mathcal{E}(t,\theta,\omega) = t \left[f(\theta) - f(\theta_{\star})\right] + \frac{1}{2}\norm{
        \left[\omega+\epsilon\right]^{1/4} \odot \left(\theta - \theta_* \right)}.^2
\]
It follows from direct computation, using the convexity assumption, that:
\[
\frac{d}{dt}\mathcal{E}(t,\theta,\omega) \leq -\norm{\frac{\nabla f(\theta)}{\sqrt[4]{\omega+\epsilon}}}^2 \left[ t- p(t)\norm{\theta-\theta_{\ast}}^2  \right].
\]
Now, under assumption \ref{ass:bounded}, it is clear that $\frac{d}{dt}\mathcal{E}(t,\theta,\omega) $ is negative for $t>>t_0$. We have, therefore, obtained the following version of Theorem \ref{thm:convergenceRate}:

\begin{theorem}\label{thm:convergenceRateBis}
We assume that assumptions \ref{ass:ODE}, \ref{ass:fConvex} and \ref{ass:bounded} are satisfied. Then, there exists $\tilde{t}\geq t_0$ such that:
    \[
 f(\theta(t))  - f(\theta_{\star})  \leq \frac{1}{t} \left[\c{E}(\tilde{t},    \theta(\tilde{t}), \omega(\tilde{t})) \right].
\]
and it follows that the ODE \eqref{eq:ODE} converges to the minimum point with rate of convergence of order $\mathcal{O}(1/t)$.
\end{theorem}

\section{Convergence results: application to first order algorithms}\label{sec:Appl}

In this section, we specify the choice of functions $h,p,q,r$ corresponding to different optimization methods and apply each convergence theorem to them. We start by a brief discussion on the assumptions which appear in this section, and we then move on to present differential equations and convergence results on: \Adam{}, AdaFom, Heavy Ball, Nesterov, Adagrad and RMSProp. Our results on \Adam{} are new and we give full details on their proof. Some of our results on AdaFom, Heavy Ball, Adagrad and RMSProp are, up to our knowledge, also new. Their proofs can be easily formalized by repeating the arguments used for \Adam{}, and we present ``sketch of proofs" in order to provide a guideline on the necessary changes. Finally, we also recover several known results in an easy manner. In particular, sharp estimates for the rate of convergence of Nesterov are given in Corollary \ref{thm:convergenceRateNesterov} below.

\subsection{On the different assumptions}

In the convergence analysis, we recurrent make assumptions which were introduced in $\S$~\ref{sec:ModelPresentation} and \ref{sec:MainResults}. We briefly recall their meaning and situations where they are satisfied:
\begin{itemize}
\item Assumption \ref{ass:NecessaryCoeficients} is equivalent, for the models in this section, with the hypothesis that the loss function $f$ is $C^2$.

\item Assumption \ref{ass:bounded} states that the trajectory $\theta(t)$ is bounded. We recall that there are some very practical situations where this assumption is always satisfied; for example in the case of coercive objective functions, see Lemma \ref{lem:JustifiesA4}.

\item Assumption \ref{ass:Morse} gives a condition on the nature and the degeneracy of the critical points of the objective function. It is used in the study of saddle points and local maximum points of the loss functions, and it also appears in \cite{LPPSJR:2017}. Note that this assumption is satisfied for generic functions (e.g. Morse functions). See Remark \ref{rk:assMorse} for further discussion.

\end{itemize}

\subsection{\Adam{}}\label{ssec:Adam}

Adaptive Moment estimation (\Adam{}) \cite{KB:14} is a famous variant of \RMSprop{} that incorporates
a momentum equation. More precisely, it computes the exponential moving average of the gradient and the square gradient. This method combines the advantages
of \RMSprop{} \cite{TH:12} in addition to the running average for the gradient. We recall that \Adam{} has three hyper-parameters: the learning rate $s$
and the exponential rate of decay for the moment estimates $\beta_1, \beta_2
 \in (0,1)$. The parameter $\eps$ is usually set to $10^{-8}$ to avoid dividing by zero. This parameter is typically not tuned.
 The algorithm
 reads as follows: for any
constants $\beta_1, \beta_2 \in (0,1)$, $\eps > 0$
and initial vectors $\theta_0 \in \R^d, m_0 = v_0 = 0$ and for all $k \geq 1$
\begin{equation}\label{eq:original_Adam}
\left\{
\begin{aligned}
g_k &= \nabla f(\theta_{k-1}) \\
m_k &=\mu_k m_{k-1} +  (1 - \mu_k)g_k  \\
v_k  &= \nu_k v_{k-1} +  (1 - \nu_k)g_k ^ 2\\
\theta_k &= \theta_{k-1} - s \  m_k / (\sqrt{v_k} + \eps).
\end{aligned}
\right.
 \end{equation}
where the two parameters for the moving average are given by
\begin{equation*}
    \left\{
\begin{aligned}
\mu_k &= \beta_1 (1 - \beta_1 ^ {k-1}) / (1 - \beta_1 ^ k) \\
\nu_k &=  \beta_2 (1 - \beta_2 ^ {k-1}) / (1 - \beta_2 ^ k).
  \end{aligned}
\right.
\end{equation*}
We rewrite the update for the parameters $\theta$ such that
\begin{equation*}\label{theta_modified}
\theta_k = \theta_{k-1} - s \  m_k / \sqrt{v_k + \eps}.
\end{equation*}
This change does not change anything in the behaviour of the algorithm.
By modifying the order of the updates and the value of the initial conditions,
we can rewrite the above algorithm in a more suitable way for our analysis. Indeed, let $\theta_0 \in \R^d$ be such that $\nabla_{\theta} f(\theta_{0}) \neq  0$ and
$m_0 = \nabla_{\theta} f(\theta_{0}), \ v_0 = \nabla_{\theta}
f(\theta_{0} )^2$, then the following recursive update rules are equivalent to \Adam{} for all $k \geq 0$
\begin{equation} \label{eq:discreteAdam}
\left\{
\begin{aligned}
    \theta_{k+1} &= \theta_{k} - s \  m_{k} / \sqrt{v_{k} + \eps}\\
    g_{k+1} &= \nabla f(\theta_{k+1}) \\
    m_{k+1} &=\mu_{k+2} m_{k} +  (1 - \mu_{k+2})g_{k+1}  \\
    v_{k+1}  &= \nu_{k+2} v_{k} +  (1 - \nu_{k+2})g_{k+1} ^ 2
\end{aligned}
\right.
 \end{equation}
As a consequence, the initial velocity is $\dot{\theta}_0 =
-\sign(\nabla  f(\theta_{0} ))$.

\subsubsection{\Adam{} differential equation}\label{ssec:AdamODE}

Consider now the three parameter family of differential equations
\begin{equation}\label{eq:ODEAdam}
 \left\{
  \begin{aligned}
        \dot{\theta} &= - m / \sqrt{v+\eps}\\
      \dot{m} &=  g_1^{A}(t,\lambda,\alpha_1,\alpha_2)\left(\nabla f(\theta) - m\right)  \\
      \dot{v} &=  g_2^{A}(t,\lambda,\alpha_1,\alpha_2)\left( \nabla f(\theta)^2 - v \right)
  \end{aligned}
  \right.
 \end{equation}
where the coefficients in ODE \eqref{eq:ODE} are given by
\begin{equation*}\label{eq:function_adam}
h \equiv r \equiv   g_1^{A}(t,\lambda,\alpha_1,\alpha_2) , \qquad p \equiv q\equiv g_2^{A}(t,\lambda,\alpha_1,\alpha_2),
\end{equation*}
where:
\begin{equation}\label{eq:function_adam2}
g_i^{A}(t,\lambda,\alpha_1,\alpha_2) =  \frac{1- e^{-\lambda/\alpha_i}}{\lambda\left(1- e^{-t/\alpha_i}\right)}, \qquad i=1,2.
\end{equation}
and $(\lambda,\alpha_1,\alpha_2)$ are positive real numbers. Note that both functions have a simple pole at $t=0$ and, therefore, satisfy assumption  \ref{ass:Existencet0}. Now, let us consider the associated discretization \eqref{eq:discrete} with learning rate $s$ and a sub-family of discrete models parametrized by $(\beta_1,\beta_2) \in (0,1)\times (0,1)$ which are given by
\begin{equation}\label{eq:parameters}
\lambda = s, \qquad \beta_i =e^{-\lambda/\alpha_i},\, i=1,2.
\end{equation}
It easily follows that for $i = 1,2$
\[
sg_i^{A}((k+1) s,\lambda,\alpha_1,\alpha_2) = 1- \beta_1 \frac{1-\beta_1^{k}}{1-\beta_1^{k+1}} = 1-\mu_{k+1},
\]
which recovers \Adam's discrete system \eqref{eq:discreteAdam} (apart from small difference in the evaluation of $\mu$). Therefore,
\Adam{} is an Euler discretization of system \eqref{eq:ODE} for the choice of function \eqref{eq:function_adam}-\eqref{eq:function_adam2} and parameters \eqref{eq:parameters}.

\subsubsection{\Adam{} without rescaling differential equation}\label{ssec:AdamODEResc}

In the original formulation of the algorithm (as stated in \eqref{eq:original_Adam}), the
parameters $\mu$ and $\nu$ depends on the iterations $k$ to correct for the bias induced by the moving average.
These coefficients can also be taken constant $\mu = \beta_1$ and $\nu= \beta_2 $, in which case we say that the algorithm is \Adam{} \emph{without rescaling}. In this case, it is easy to verify that the differential equation:
\begin{equation}\label{eq:ODEAdamWRescalling}
 \left\{
  \begin{aligned}
        \dot{\theta} &= - m / \sqrt{v+\eps}\\
      \dot{m} &=  1/\alpha_1\left(\nabla f(\theta) - m\right)  \\
      \dot{v} &=  1/\alpha_2\left( \nabla f(\theta)^2 - v \right)
  \end{aligned}
  \right.
 \end{equation}
is the continuous counter-part of the algorithm when we consider the sub-family given by $\beta_1 =  (1 - s/ \alpha_1) $ and $\beta_2 =  (1 - s/ \alpha_2 ) $.

\subsubsection{Convergence of \Adam{}}\label{sssec:AdamConv}

\begin{corollary}[Convergence of \Adam{}]\label{thm:convergenceAdam}
Suppose $\eps>0$ and let assumptions \ref{ass:ODE} and \ref{ass:bounded} be satisfied for equation \eqref{eq:ODEAdam}. Moreover, we assume
\[
3+ \beta_2>4\beta_1, \qquad \text{where } \beta_i = \exp(-\lambda/\alpha_i), \quad i=1,2.
\]
Then the following convergence results hold true
\begin{enumerate}
    \item[(I)] \emph{Topological convergence:} $f(\theta(t)) \to f_{\star}$, $m(t) \to 0$ and $v(t) \to 0$ when $t\to \infty$, where $f_{\star}$ is a critical value of $f$.
     \item[(II)] \emph{Non-local minimum avoidance:} We assume the additional hypothesis \ref{ass:Morse} on the objective function.
    Fix $t_0>0$ and denote by $S_{t_0}$ the set of initial conditions $(\theta_0,m_0,v_0) \in \mathbb{R}^d \times \mathbb{R}^{d}_{\geq 0}$ such that $\theta_{\star} \in \omega(\theta(t))$, where $\theta_{\star}$ is not a local-minimum of $f$. Then the Lebesgue measure of $S_{t_0}$ is zero.

    \item[(III)] \emph{Rate of convergence:} Under the additional convexity assumption  \ref{ass:fConvex}, there exists a constant $\mathcal{K}>0$ which depends on $f$, $\theta_0$ and $v_0$, so that:
\[
\lim_{t\to \infty} f(\theta(t)) - f(\theta_{\star}) < \mathcal{K} \ln(1/\beta_1) \frac{1-\beta_2}{s(1-\beta_1)}.
\]
The rate of convergence to this neighbourhood, furthermore, is of order $\mathcal{O}(1/t)$.
\end{enumerate}

\end{corollary}

Note that there is an apparent paradox between point $I$ and $III$ of the Corollary: the dynamics is convergent by $I$, but the ``fast" rate of convergence (of order $\mathcal{O}(1/t)$) can only be guaranteed to a neighbourhood. This is no paradox, nevertheless, because \Adam{} might converge very \emph{slowly} once it attains the neighbourhood given by point $III$. In particular, the discrete version of \Adam{} may not converge even in the deterministic case (see Proposition \ref{prop:limitCycle} below), and this is expected because of the rate of convergence in Theorem \ref{thm:ConvergenceDiscretizationRate}.

Point $III$ has two other surprising consequences. First, it justifies the usual choice of practitioners to take the hyper parameter $\beta_2$ as close to $1$ as possible, because:
\[
\lim_{\beta_2 \to 1} 1-\beta_2 = 0
\]
while $\beta_1$ has a smaller direct impact, even if it is convenient to take it close to $1$, because:
\[
\lim_{\beta_1 \to 1}\frac{\ln(1/\beta_1)}{1-\beta_1} = 1.
\]
Second, the size of the neighbourhood of fast convergence (of order $\mathcal{O}(1/t)$) is controlled by a constant $\mathcal{K}$ which only depends on $f$ and the initial conditions. This indicates that the success of \Adam{} for certain loss functions (e.g. loss functions in deep learning) might be associated to \emph{features} on the ``class" of loss functions considered.

We now turn to the proof of the Corollary:

\begin{proof}[Proof of Corollary \ref{thm:convergenceAdam}]
The proof of \emph{(I)} directly follows from Theorem \ref{thm:convergence} provided that assumptions \ref{ass:NecessaryCoeficients} and \ref{ass:TopologicalConvergence} are satisfied.
  Hence, the proof simply consists on checking the validity of both assumptions under the condition that $3+ \beta_2>4\beta_1$.
  Let us recall that the coefficients for the \Adam{}'s differential equations are given by
  \begin{equation*}
h \equiv r \equiv   g_1^{A}(t,\lambda,\alpha_1,\alpha_2) , \qquad p \equiv q\equiv g_2^{A}(t,\lambda,\alpha_1,\alpha_2),
\end{equation*}
and $(\lambda,\alpha_1,\alpha_2)$ are positive real numbers and:
\begin{equation*}
g_i^{A}(t,\lambda,\alpha_1,\alpha_2) =  \frac{1- e^{-\lambda/\alpha_i}}{\lambda\left(1- e^{-t/\alpha_i}\right)}, \qquad i=1,2.
\end{equation*}
It is easy to check that assumptions \ref{ass:NecessaryCoeficients} and \ref{ass:TopologicalConvergence} are satisfied if there exists a $t$, large enough, such that
\[
\frac{1 - e^{-\lambda/\alpha_1}}{\lambda(1-e^{-t/\alpha_1})} - \frac{1 - e^{-\lambda/\alpha_2}}{4\lambda(1-e^{-t/\alpha_2})} >0
\]
 Taking the limit as $t$ goes to infinity in the above inequality gives
\[
1 - e^{-\lambda/\alpha_1} > \frac{1 - e^{-\lambda/\alpha_2}}{4}.
\]
We conclude using the expressions of $\beta_1$ and $\beta_2$.

The proof of \emph{(II)} follows directly from Theorem \ref{thm:avoiding} since assumptions \ref{ass:NecessaryCoeficients} and \ref{ass:TopologicalConvergence} are satisfied
under the condition $3 + \beta_2 > 4\beta_1$.

In order to prove \emph{(III)}, let us check the hypotheses of Theorem \ref{thm:convergenceRate}. We compute explicitly the functions
\begin{align*}
\c{A}(t) &
=
\frac{1- e^{-\lambda/\alpha_i}}{\lambda}\int_{t_0}^t \frac{e^{s/\alpha_1}}{e^{s/\alpha_1}-1}  \c{B}(s)ds
\\
\c{B}(t)  &=  (e^{t/\alpha_1}-1) \int_{t}^{\infty} \frac{1}{e^{s/\alpha_1}-1} ds \\
    \c{C}(t) &
    =\frac{e^{t/\alpha_1}-1}{e^{t/\alpha_1}}\int_{t_0}^t \frac{e^{s/\alpha_1}}{e^{s/\alpha_1}-1}  \c{B}(s)ds
\end{align*}
so, by direct computation via L'H\^opital's rule:
\begin{align*}
\lim_{t\to \infty}\c{A}(t)/t &= \alpha_1 \frac{1 - e^{-\lambda/\alpha_1}}{\lambda} \\
\lim_{t\to \infty} \c{B}(t) &= \alpha_1 \\
\lim_{t\to \infty}\c{C}(t)/t &=  \alpha_1
\end{align*}
and it easily follows that assumption \ref{ass:necessaryConditions} is verified. Finally, by using L'H\^opital's rule, we get:
\[
\lim_{t\to \infty}\int_{t_0}^t q(s)ds /\c{A}(t) = \alpha_1^{-1}\frac{ 1 - e^{-\lambda/\alpha_2}}{1 - e^{-\lambda/\alpha_1}}
\]
which yields the result.
\end{proof}

\subsection{AdaFom}\label{ssec:AdaForm}
In \cite{CLSH:18}, the authors propose a variation of \Adam{} algorithm which can be guaranteed to have good convergence rate. We will provide the differential equation associated to this algorithm, and recover its expected deterministic convergence rate below (based on \cite[Corollary 3.2]{CLSH:18}). The algorithm reads as follows: for any constants $\beta_1\in (0,1)$, $\eps \geq 0$
and initial vectors $\theta_0 \in \R^d, m_0 = v_0 = 0$ and for all $k \geq 1$
\begin{equation}\label{eq:original_AdaFom}
\left\{
\begin{aligned}
g_k &= \nabla f(\theta_{k-1}) \\
m_k &=\mu_k m_{k-1} +  (1 - \mu_k)g_k  \\
v_k  &= (1-1/k) v_{k-1} +   1/k\, g_k ^ 2\\
\theta_k &= \theta_{k-1} - s \  m_k / (\sqrt{v_k} + \eps).
\end{aligned}
\right.
 \end{equation}
where the parameter for the moving average are given by (just as in \Adam{}):
\begin{equation*}
\begin{aligned}
\mu_k &= \beta_1 (1 - \beta_1 ^ {k-1}) / (1 - \beta_1 ^ k).
  \end{aligned}
\end{equation*}

Consider now the two parameter family of differential equations
\begin{equation}\label{eq:ODEAdaFom}
 \left\{
  \begin{aligned}
        \dot{\theta} &= - m / \sqrt{v+\eps}\\
      \dot{m} &=  g^{A}(t,\lambda,\alpha)\left(\nabla f(\theta) - m\right)  \\
      \dot{v} &=  \frac{1}{t}\left( \nabla f(\theta)^2 - v \right)
  \end{aligned}
  \right.
 \end{equation}
where:
\begin{equation*}
g^{A}(t,\lambda,\alpha) =  \frac{1- e^{-\lambda/\alpha}}{\lambda\left(1- e^{-t/\alpha}\right)}.
\end{equation*}
and $(\lambda,\alpha)$ are positive real numbers. Just as in the case of \Adam{},   consider the associated discretization \eqref{eq:discrete} with learning rate $s$ and a sub-family of discrete models parametrized by $\lambda = s$ and $\beta = e^{-\lambda/\alpha}$. The reader may verify, following the same steps of the analysis of \Adam{} that the discretization of this sub-family recovers AdaFom algorithm \eqref{eq:original_AdaFom}. We now turn to the convergence analysis :

\begin{corollary}[Convergence of AdaFom]\label{thm:convergenceAdaFom}
Suppose $\eps>0$ and let assumptions \ref{ass:ODE} and \ref{ass:bounded} be satisfied for equation \eqref{eq:ODEAdaFom}. Then the following convergence results hold true
\begin{enumerate}
    \item[(I)] \emph{Topological convergence:} $f(\theta(t)) \to f_{\star}$ and $m(t) \to 0$ when $t\to \infty$, where $f_{\star}$ is a critical value of $f$.

    \item[(III)] \emph{Rate of convergence:} Under the additional convexity assumption  \ref{ass:fConvex}, $f(\theta(t)) \to f(\theta_{\star})$ with the rate $\c{O}(\ln(t)/t)$.
\end{enumerate}
\end{corollary}

Note that Theorem \ref{thm:avoiding} can not be applied to ODE \eqref{eq:ODEAdaFom} in a direct way because $v(t)$ may not go to $0$ (in particular, in the notations of Theorem \ref{thm:avoiding}, $Q(0)=0$).
\medskip

\begin{proof}[Sketch of the proof of Corollary \ref{thm:convergenceAdaFom}]
By the choice of functions $h(t)$, $r(t)$, $p(t)$ and $q(t)$, it is easy to see that Assumptions \ref{ass:NecessaryCoeficients} and \ref{ass:TopologicalConvergence} are always verified. Part $(I)$ is, therefore, direct consequences from Theorem \ref{thm:convergence}. Next, the computation of $\c{A}$, $\c{B}$ and $\c{C}$ are independent on $p(t)$ and $q(t)$, so they are analogous to the one's obtained for \Adam{}. It follows that assumption \ref{ass:necessaryConditions} is verified. Finally, by Theorem \ref{thm:convergenceRate}, the rate of convergence is controlled by the asymptotic behaviour of:
\[
\int_{t_0}^t q(s)ds /\c{A}(t) = \ln(t/t_0)/\c{A}(t),
\]
which can easily be verified to be of order $\mathcal{O}(\ln(t)/t)$.
\end{proof}

\subsection{Heavy Ball}\label{ssec:HeavyBall}

We consider the Heavy ball second order differential equation \cite{alvarez2000}
\begin{equation}\label{eq:ODEHeavyBall}
\ddot{x} + \gamma \dot{x} + \nabla f(x) =0,
\end{equation}
where $\gamma>0$. By taking $\theta = x$ and $m = - \dot{x}$ (and $v \equiv 1$), we obtain the system \eqref{eq:ODE} with
\[
h(t)\equiv 1, \qquad r(t) \equiv \gamma, \qquad  \text{and} \qquad p(t) \equiv q(t) \equiv 0.
\]
Equation \eqref{eq:discrete} simplifies to
\begin{equation}\label{eq:HeavyBall_discrete}
 \left\{
  \begin{aligned}
      \theta_{k+1} &= \theta_{k} - s  m_k \\
      m_{k+1} &=  (1 - s \gamma ) m_{k} + s  \nabla f(\theta_{k+1})
  \end{aligned}
\right.
 \end{equation}
which corresponds to the classical Heavy ball methods with damping coefficient $\beta = 1 -s \gamma$, momentum variable $n_k  = s m_k$ and learning rate $\alpha = s^2$.
Implicit discretization has also been considered in \cite{alvarez2000}.

From our analysis, we recover results given in \cite{LPPSJR:2017, GFJ:14} for the discrete update rules \eqref{eq:HeavyBall_discrete}:

\begin{corollary}
\label{thm:convergenceHB}
Suppose that assumptions \ref{ass:ODE} and \ref{ass:bounded} are satisfied for equation \eqref{eq:HeavyBall_discrete}. Then
\begin{enumerate}
    \item[(I)] \emph{Topological convergence:} $f(\theta(t)) \to f_{\star}$ and $\omega(t) \to \omega_{\infty}>0$ when $t\to \infty$, where $f_{\star}$ is a critical value of $f$.
    \item[(II)]  \emph{Non-local minimum avoidance:} We assume the additional hypothesis \ref{ass:Morse} on the objective function.
    Fix $t_0>0$ and denote by $S_{t_0}$ the set of initial conditions $(\theta_0,\omega_0) \in \mathbb{R}^d \times \mathbb{R}^{d}_{\geq 0}$ such that $\theta_{\star} \in \omega(\theta(t))$, where $\theta_{\star}$ is not a local-minimum of $f$.
    Then the Lebesgue measure of $S_{t_0}$ is zero.
    \item[(III)] \emph{Rate of convergence:} Under the additional convexity assumption  \ref{ass:fConvex}, $f(\theta(t)) \to f(\theta_{\star})$ with the rate $\c{O}(1/t)$.
    \end{enumerate}
\end{corollary}
\begin{proof}[Sketch of the proof]
By the choice of functions $h(t)$, $r(t)$, $p(t)$ and $q(t)$, it is easy to see that Assumptions \ref{ass:NecessaryCoeficients} and \ref{ass:TopologicalConvergence} are always verified. Part $(I)$ and $(II)$ are, therefore, direct consequences from Theorems \ref{thm:convergence} and \ref{thm:avoiding}. Next, by direct computation, we get:
\begin{align*}
\c{A}(t) = \gamma (t-t_0) \quad \c{B}(t)  =  \gamma \quad \c{C}(t) =  \gamma (t-t_0).
\end{align*}
It follows that assumption \ref{ass:necessaryConditions} is verified. Finally, by Theorem \ref{thm:convergenceRate}, the rate of convergence is controlled by the asymptotic behaviour of $1/\c{A}(t)$, which is of order $\mathcal{O}(\ln(t)/t)$.
\end{proof}

\subsection{Nesterov}\label{ssec:Nesterov}

Following \cite{SBC:15}, we consider the Nesterov second order differential equation, parametrized by the constant $r>0$,
\begin{equation}\label{eq:ODENesterovSecond}
\ddot{x} + \frac{r}{t} \dot{x} + \nabla f(x) =0.
\end{equation}
Similarly as in the Heavy Ball case, we define $\theta = x$ and $m = - \dot{x}$ and write the above equation as a system \eqref{eq:ODE} with
\[
h(t)\equiv 1, \qquad r(t) = r/t, \qquad  and \qquad p(t) \equiv q(t) \equiv 0.
\]
In \cite{SBC:15}, the authors studied a slightly different forward Euler scheme and proved that the difference between the numerical scheme and the Nesterov algorithm goes to zero in the limit $s \to 0$. This effectively replaces our Theorem \eqref{thm:ConvergenceDiscretizationRate}, so that Nesterov differential equation can be assumed to be given by:
\begin{equation}\label{eq:ODENesterov}
 \left\{
  \begin{array}{ll}
        \dot{\theta} &= - m \\
      \dot{m} &=  \nabla f(\theta) -  r/t \cdot  m
  \end{array}
\right.
\end{equation}
where $r>0$. We are ready enunciate the main convergence result for Nesterov's differential equation, which have been previously proved in \cite{SBC:15, Attouch:2017, Attouch:2018}. 

\begin{corollary}[Convergence Rate of Nesterov]
\label{thm:convergenceRateNesterov}
Suppose that equation \eqref{eq:ODENesterov} satisfies assumptions \ref{ass:ODE}, \ref{ass:bounded} and \ref{ass:fConvex}. Then $f(\theta) \to f(\theta_{\star})$ when $t\to \infty$ with rate of convergence:
\[
\begin{aligned}
\c{O}(1/t^2), & \text{ if }r\geq 3\\
\c{O}(1/t^{2r/3}), & \text{ if }r\leq 3.
\end{aligned}
\]
\end{corollary}
\begin{proof}
The proof for $r\leq 3$ is given in subsection \ref{ssec:NesterovBounds}. We assume in here that $r\geq 3$. Recall that:
\[
\begin{aligned}
h(t) =1 & \quad &r(t) = r/t & \quad & p(t)=q(t)=0
\end{aligned}
\]
From direct computation, we get:
\begin{align*}
\c{A}(t) = (t^2-t_0^2)/2(r-1) \quad \c{B}(t)  =  t/(r-1) \quad \c{C}(t) =  (t^2-t_0^2)/2(r-1)
\end{align*}
It easily follows that, whenever $r\geq 3$, the inequalities of assumption \ref{ass:necessaryConditions} are verified, and the result follows from Theorem \ref{thm:convergenceRate}.
\end{proof}

\subsection{\Ada{}}\label{ssec:Adagrad}

\Ada{}  \cite{DHS:11} was designed to incorporate knowledge of the geometry of the data
previously observed during the training. In his deterministic version (full batch), for all $k \in \N$,
\begin{equation} \label{eq:adagrad}
\left\{
\begin{aligned}
    v_{k+1} & =  \sum^k_{j=0} \left[\nabla f(\theta_j) \right]^2\\
    \theta_{k+1}& =  \theta_k -s \nabla_{\theta}
    f(\theta_k) / \sqrt{v_{k+1}}.
\end{aligned}
\right.
\end{equation}
with initial conditions $v_0 = 0$ and $\theta_0 \in \R^d$.
Here, we recall that $ \left[\nabla_{\theta} f(\theta) \right]^2$ denotes the element-wise product
$\nabla_\theta f(\theta)\odot\nabla_{\theta} f(\theta)$.
The adaptive part in the algorithm comes from the term $\sqrt{v_k}$ which is precisely
the preconditioning matrix used to scale the gradients. The algorithm \eqref{eq:adagrad} can be equivalently described by
\begin{equation}
\left\{
\begin{aligned}
    \theta_{k+1} &= \theta_{k} - s  \nabla f(\theta_{k}) / \sqrt{ G_{k} }  \\
    G_{k+1} &= G_{k} + [ \nabla f(\theta_{k+1}) ]^2.
\end{aligned}
\right.
\end{equation}
with initial condition $\theta_0$ and $G_0 =  \nabla f(\theta_{0})^2$. By setting $\alpha = s^2$ and $\omega_k = \alpha G_k$, it is easy to conclude that the  \Ada's update rule can be written as:
\begin{equation}\label{eq:DiscreteAdagrad}
\left\{
\begin{aligned}
	\theta_{k+1}& =  \theta_k - \alpha    \nabla f(\theta_k) / \sqrt{ \omega_k }  \\
    \omega_{k+1} & = \omega_k + \alpha [ \nabla f(\theta_{k+1} ) ]^2
\end{aligned}
\right.
\end{equation}
with the re-scaled initial condition $\omega_0 = \alpha [ \nabla f(\theta_{0} ) ]^2$ and $\theta_0 \in \R^d$. It is now easy to see that \eqref{eq:DiscreteAdagrad} is a forward Euler discretization of the system of differential equations \eqref{eq:ODEbis} (which we call the \Ada{} differential equation)
\begin{equation}\label{eq:ODEAdagrad}
 \left\{\begin{aligned}
  \dot{\theta}(t) &= -  \nabla f(\theta(t))/\sqrt{\omega(t)}\\
  \dot{\omega}(t) &= \left[ \nabla f(\theta(t))\right]^2,
 \end{aligned}\right.
 \end{equation}
with initial condition given by $\theta_0 \in \R^d$, $\omega_0 =  \alpha [\nabla f(\theta_{0})]^2$ and with $q \equiv 0, \eps = 0$ and $p \equiv 1$. Our methods lead to the following result (whose proof is a direct consequence of the results in $\S\S$~\ref{ssec:ConvergenceBis}):

\begin{corollary}[Convergence of Adagrad]
\label{thm:convergenceAdagrad}
Suppose that assumptions \ref{ass:ODE} and \ref{ass:bounded} are satisfied for equation \eqref{eq:ODEAdagrad}. Then
\begin{enumerate}
    \item[(I)] \emph{Topological convergence:} $f(\theta(t)) \to f_{\star}$ and $\omega(t) \to \omega_{\infty}>0$ when $t\to \infty$, where $f_{\star}$ is a critical value of $f$.
    \item[(II)]  \emph{Non-local minimum avoidance:} We assume the additional hypothesis \ref{ass:Morse} on the objective function. Fix $t_0>0$ and denote by $S_{t_0}$ the set of initial conditions $(\theta_0,\omega_0) \in \mathbb{R}^d \times \mathbb{R}^{d}_{\geq 0}$ such that $\theta_{\star} \in \omega(\theta(t))$, where $\theta_{\star}$ is not a local-minimum of $f$. Then the Lebesgue measure of $S_{t_0}$ is zero.
    \item[(III)] \emph{Rate of convergence:} Under the additional convexity assumption  \ref{ass:fConvex}, $f(\theta(t)) \to f(\theta_{\star})$ with the rate $\c{O}(1/t)$.
    \end{enumerate}
\end{corollary}

\subsection{RMSProp}\label{ssec:RMSProp}

The only difference between RMSProp and \Ada{} is how the preconditioning matrix is computed. In \RMSprop, it consists of an exponentially decaying
moving average, rather than a sum of the previous gradients, which has many similarities to the update rule of \Adam{}
\begin{equation} \label{eq:rmsprop}
\left\{
\begin{aligned}
    v_{k+1} &= \beta v_{k} + (1-\beta)\nabla_{\theta} f(\theta_k)^2\\
    \theta_{k+1}& =  \theta_k -s \nabla_{\theta}
    f(\theta_k) / \sqrt{v_{k+1}}.
\end{aligned}
\right.
\end{equation}
In particular, following the same analysis as performed in $\S\S$~\ref{ssec:Adam}, we derive the differential equation:
\begin{equation}\label{eq:ODERSMProp}
 \left\{\begin{aligned}
  \dot{\theta}(t) &= - \nabla f(\theta(t))/\sqrt{\omega(t)}\\
  \dot{\omega}(t) &= \frac{1}{\alpha} \cdot \left( \left[ \nabla f(\theta(t))\right]^2 - \omega(t)\right),
 \end{aligned}\right.
 \end{equation}
where $\beta = (1-s/\alpha)$, with initial condition given by $\theta_0 \in \R^d$ and $\omega_0 =  s [\nabla f(\theta_{0})]^2$. This differential equation allow us to prove the following result (whose proof is a direct consequence of the results in $\S\S$~\ref{ssec:ConvergenceBis}):

\begin{corollary}[Convergence of RMSProp]
\label{thm:convergenceRMSProp}
Suppose that assumptions \ref{ass:ODE} and \ref{ass:bounded} are satisfied for equation \eqref{eq:ODERSMProp}. Then
\begin{enumerate}
    \item[(I)] \emph{Topological convergence:} $f(\theta(t)) \to f_{\star}$ and $\omega(t) \to 0$ when $t\to \infty$, where $f_{\star}$ is a critical value of $f$.
    \item[(II)]  \emph{Non-local minimum avoidance:} We assume the additional hypothesis \ref{ass:Morse} on the objective function.
    Fix $t_0>0$ and denote by $S_{t_0}$ the set of initial conditions $(\theta_0,\omega_0) \in \mathbb{R}^d \times \mathbb{R}^{d}_{\geq 0}$ such that $\theta_{\star} \in \omega(\theta(t))$, where $\theta_{\star}$ is not a local-minimum of $f$. Then the Lebesgue measure of $S_{t_0}$ is zero.
    \item[(III)] \emph{Rate of convergence:} Under the additional convexity assumption  \ref{ass:fConvex}, $f(\theta(t)) \to f(\theta_{\star})$ with the rate $\c{O}(1/t)$.
    \end{enumerate}
\end{corollary}

\section{Existence and uniqueness of solutions: $t_0=0$}\label{sec:ExistenceResult}

In this section we improve Theorem \ref{th:existencePrel}, and we consider the case that $t_0=0$. Recall that our functions $h(t)$, $r(t)$, $p(t)$ and $q(t)$ are allowed to have poles in the origin (c.f. \ref{tab:1}) in order to capture the phenomena present in accelerated methods (c.f. Nesterov) and on rescaling due to bias correction (c.f. ADAM). This impose some technical difficulties, as illustrated by the following example:

\begin{example}\label{ex:ExistenceIssue}
Consider the differential equation, with $L\neq 0$:
 \[
 \dot{\theta}(t) = \frac{L}{t} \theta(t), \quad \theta(0)=\theta_0
 \]
Now the solutions of this equation with initial condition at $t_0=1$ are given by
\[
\theta(t) = \theta(1) \cdot t^L.
\]
This allow us to make the following considerations:
\begin{itemize}
\item[(a)] Suppose $\theta_0 \neq 0$. Then there is no solution for the system.
\item[(b)] Suppose $\theta_0 =0$. Then there always exist a solution given by $\theta(t) \equiv 0$. Uniqueness, nevertheless, only holds if $L< 0$. Indeed, if $L> 0$, for every $\theta(1) \in \mathbb{R}$ the solution $\theta(t) = \theta(1) \cdot t^L$ converges to $0$ when $t\to 0$.
\end{itemize}
\end{example}

We, therefore, need to demand extra assumptions on the coefficients of our model and the initial conditions in order to guarantee the existence and uniqueness of the solution at time $t=0$:

\begin{Assumption}\label{ass:Existencet0}
We assume one of the following condition
    \begin{enumerate}

        \item The functions $h,r,p,q$ have a simple pole at $t=0$.
        \item If $h \in C^1([0, + \infty)) $ (resp $p \in C^1([0, +\infty))$), then $r$ (resp. $q$) can have (at most) a simple pole at zero.
        \item In any other cases, all functions are assumed to be $C^1$ on
            $[0,\infty) $.
    \end{enumerate}
    In cases $(1)$ and $(2)$, furthermore, we demand the following two extra-conditions:
    \begin{itemize}
    \item[(a)] the initial conditions must be taken as:
    \[
    m_0 = \nabla f(\theta_0) \lim_{t \to 0^+} h(t)/r(t) ,\quad  v_0 = [\nabla f(\theta_0)]^2 \lim_{t \to 0^+} p(t)/q(t).
    \]
    \item[(b)] We assume that $p(t) \not\equiv 0$ and that there exists a small time $\hat{t} $ such that
    \[
     2r(t) - q(t) \geq 0, \qquad \forall t < \hat{t},
    \]
\end{itemize}
\end{Assumption}

Conditions (a) and (b) of Assumption \ref{ass:Existencet0} should be compared with the issues (a) and (b) raised in Example \ref{ex:ExistenceIssue}, respectively. In terms of the algorithms in table \ref{tab:1}, the assumption is always verified for Heavy-Ball, Nesterov, AdaForm and ADAM without rescaling. For ADAM, nevertheless, it is necessary to add mild assumptions on the hyper-parameters such as $\beta_1\geq 0.21$. We are now ready to enunciate our main result:

\begin{theorem}\label{th:existence}
Suppose that the ODE \eqref{eq:ODE} satisfies assumptions \ref{ass:ODE}, and that either $p(t)\not\equiv0$, or assumption \ref{ass:NecessaryCoeficients} with $\tilde{t}=0$ is satisfied. For any $t_0 > 0$ and admissible initial condition $\x(t_0)$,
 there exists a unique global solution to equation \eqref{eq:ODE} such that:
 \[
 \begin{aligned}
 \theta &\in C^2([t_0,\infty); \R^d) \quad  \text{ and } \quad m,\,v &\in C^1([t_0,\infty); \R^d).
 \end{aligned}
\]
Suppose, furthermore, that assumption \ref{ass:Existencet0} is also satisfied. Then, there exists a unique global solution to equation \eqref{eq:ODE} such that:
 \[
 \begin{aligned}
 \theta &\in C^2((0,\infty); \R^d) \cap C^1 ([0,\infty); \R^d)
\quad  \text{ and } \\ m,\,v &\in C^1((0,\infty); \R^d) \cap C([0,\infty); \R^d).
 \end{aligned}
\]
\end{theorem}

The proof is postponed in Appendix \ref{sec:Existence}, and it is technically much harder than the proof of Theorem \ref{th:existencePrel} because of the pole at the origin.

\section{Considerations and guidelines for adaptive algorithms}\label{sec:considerations}

We make here empirical observations, discuss some limitations, and provide some extra guidelines for the use of adaptive algorithms, based on our previous analysis.

\subsection{The discrete dynamics does not necessarily converge.}
One strong limitation of \Adam{} is the existence of discrete limit cycles in the sense that the algorithm produces oscillations that never damp out.
If the discrete dynamics reaches such an equilibrium, the difference $f(\theta_k) - f(\theta_{\star})$ can not converge arbitrarily close to zero
with an increasing number of steps.
However, it reaches a neighborhood of the critical point whose radius is determined by the learning rate $s$.
Decaying the learning rate is therefore necessary to obtain convergence of the dynamics
Numerically, we found that \Adam{} with $\beta_1 > 0$ suffers from the same phenomena but the limit cycles are more difficult to establish.
We believe that the existence of such cycles depend on the local curvature of the function $f$ near the optimum.

\begin{proposition}[Existence of a discrete limit cycle for \Adam{}]\label{prop:limitCycle}
Let $\beta_1 =  0$ and $f(\theta) = \theta^2/2$. Then there exists a discrete limit cycle for \eqref{eq:discreteAdam}.
    \end{proposition}
    \begin{proof}
        Let us assume that there exists a $k$ such that $\theta_k  = s/2 $ and that $v_k = (s/2)^2$, where $s$ is the learning rate. It easily follows from the update rules that
        \begin{align*}
            \theta_{k+1} &= \theta_k -s \frac{\nabla f(\theta_k) }{\sqrt{v_k}} =
            \frac{s}{2} - s  = - \theta_k \\
            v_{k+1}  &= (s/2)^2
        \end{align*}
        Therefore
        $\theta_{k+2} = -\frac{s}{2} + s  = \theta_k $ and the system has entered a discrete equilibrium.
    \end{proof}
We illustrate this behavior in Figure \ref{fig:limit_cycle} on the strongly convex toy function $f(x,y) = x^2  + y^2$.
\begin{figure}[t!]
    \centering
    \begin{subfigure}[b]{0.45\textwidth}
        \caption{Discrete limit cycles for Adam}
        \includegraphics[width=\textwidth]{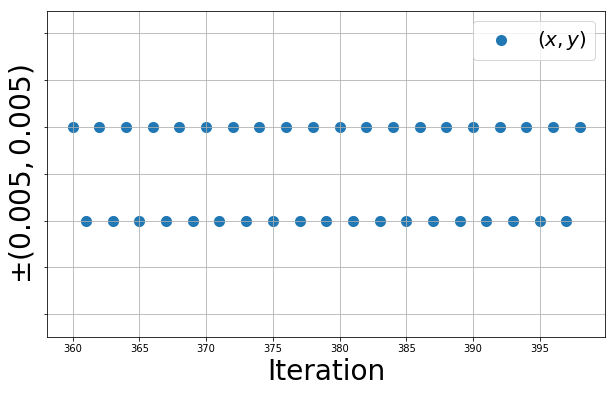}
    \end{subfigure}
    \begin{subfigure}[b]{0.45\textwidth}
        \caption{Convergence}
        \includegraphics[width=\textwidth]{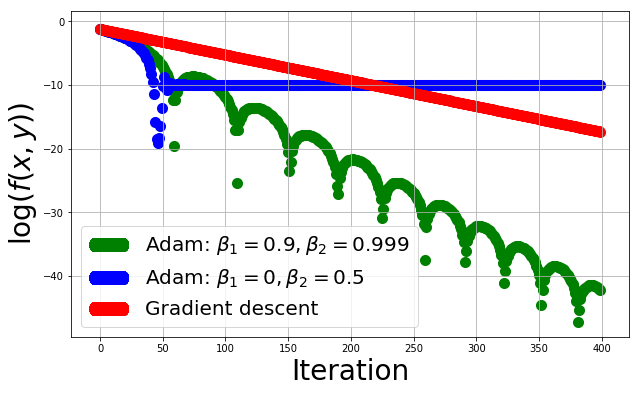}
    \end{subfigure}
    \caption[Illustration of discrete limit cycles for the \Adam{}'s algorithm.]
    {\small Illustration of discrete limit cycles for the \Adam's algorithm with $\eps = 10^{-8}, \beta_2 = 0.5, s=10^{-2}$.
    {\bf a)} Limit cycle of period two for \Adam. The algorithm oscillates between two points $(0.005, 0.005)$ and $(-0.005, -0.005)$. {\bf b)} Plot of the logarithm of $f$ versus the number of iterations.
    The loss plateau after $50$ iterations.
    }
 \label{fig:limit_cycle}
\end{figure}

It is important to note that the value of the gap between $f(\theta_k)$ and $f(\theta_{\star})$ depends on the learning rate. Choosing a smaller learning rate reduce the gap, but doesn't remove it.

\subsection{The hyper-parameters $\beta_1, \beta_2$ in ADAM should be tuned in terms of the learning rate}
The second observation is related to the hyper-parameters of the optimizers
and give important guidance on how to tune them.
As observed in section \ref{ssec:Adam}, the parameters $\beta_1$ and $\beta_2$ are chosen as functions of the learning rate $s$ and parameters $\alpha_1$ and $\alpha_2$. It is often the case in practice (in particular in stochastic optimization)
to decay the learning rate during the training process.
By doing so, the discrete dynamics is completely modified unless the $\beta$'s are adjusted to keep the parameters $\alpha_i$ constant. Therefore, once a particular choice of hyper-parameters seems promising, a decay in the learning rate should be accompanied by changing the hyper-parameters $\beta_i$ according to the formula \eqref{eq:parameters}, which we recall here
\[
 \beta_i = e^{-s/\alpha_i},\, i=1,2.
\]
Indeed, by doing so the underlying dynamics is preserved. We illustrate this in plot {\bf b)}, Figure \ref{fig:discretization}, where we compute the logarithm of the error between different trajectories
\begin{enumerate}
    \item The reference dynamics: $\beta_2 = 0.99$ and $s=0.001$. According to formula \eqref{eq:parameters}, $\alpha_2  = - 0.001 / \log(0.99) \approx 0.0995$
    \item Second dynamics: $\beta_2 = 0.99$ and $s=0.01$
    \item Third dynamics: same learning rate as the second dynamics $s=0.01$ but we adjust the hyper-parameter: $\beta_2 = \exp(- 0.01 / 0.0995) = 0.90438$.
\end{enumerate}
   \begin{figure}[t!]
    \centering
    \begin{subfigure}[b]{0.45\textwidth}
        \caption{Trajectories}
        \includegraphics[width=\textwidth]{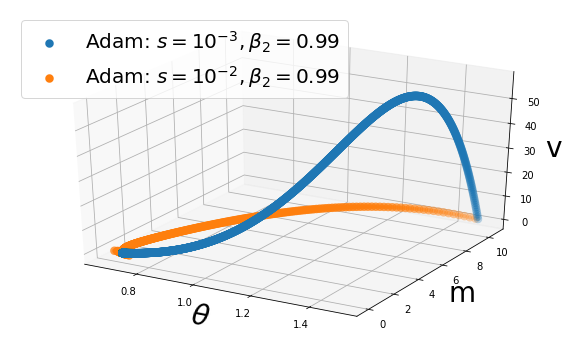}
    \end{subfigure}
    \begin{subfigure}[b]{0.45\textwidth}
        \caption{Comparison of $L^2$ norms}
        \includegraphics[width=\textwidth]{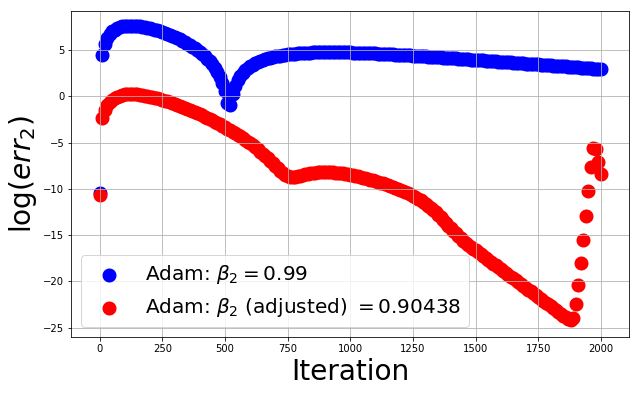}
    \end{subfigure}
    \caption[discretization]
    {\small Fixing $\beta_2$ and changing the learning rate $s$ lead to different dynamics.
    {\bf a)} Trajectories of \Adam{} (1) $\&$ (2) when only the learning rate is changed.
    {\bf b)} Comparison of the error between trajectories (1) $\&$ (2) and (1) $\&$ (3). As expected the discrepancies between (1) $\&$ (3) is very small.}
      \label{fig:discretization}
\end{figure}

\subsection{Convergence properties of \Adam{} depend on the class of loss functions}\label{ssec:Flat}

The convergence analysis in the convex case (see Theorem \ref{thm:convergenceRate}) seem to indicate that \Adam{}
is a rather slow algorithm because quick convergence (in the order $o(1/t)$) is only guaranteed in a neighbourhood of the global minimum. Under a closer look, however, we note that the size of the neighbourhood might be very small depending on the class of loss functions because of its impact to the term \eqref{eq:CrucialTerm} and the constant $\mathcal{K}$. In particular, we believe that there are situations where \Adam{} should perform consistently better than other algorithms, provided that the hyper-parameters are well-chosen; e.g. for flat loss functions (see figure \ref{fig:convergence}). We intend to deepen this remark in forthcoming works.

\begin{figure}[t!]
    \centering
    \begin{subfigure}[b]{0.45\textwidth}
        \caption{$f(x,y) = (x+y)^4 +  (x/2-y/2)^4$}
        \includegraphics[width=\textwidth]{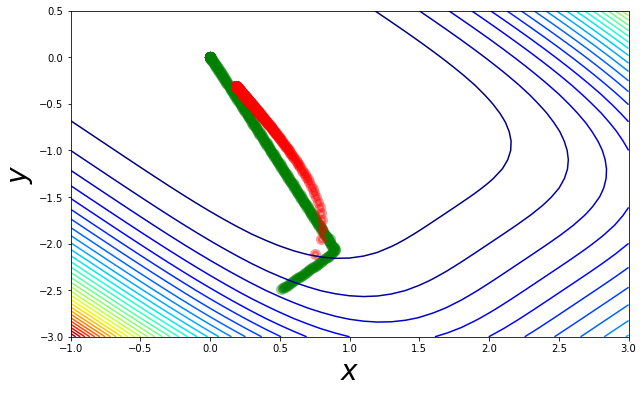}
    \end{subfigure}
    \begin{subfigure}[b]{0.45\textwidth}
        \caption{Rate of convergence}
        \includegraphics[width=\textwidth]{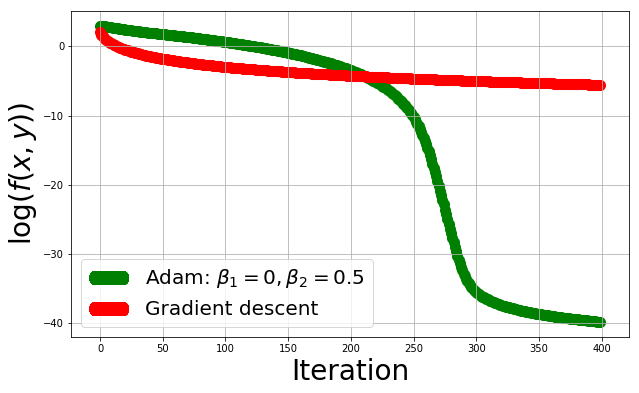}
    \end{subfigure}
    \caption[]
    {\small Comparison between gradient descent and \Adam. Gradient Descent converges faster
    initially when the gradients are large but \Adam{} outperforms Gradient descent after entering the flat region.
    Both trajectories start from the point $(0.5, -2.5)$.
    }
    \label{fig:convergence}
\end{figure}
    \begin{figure}[t!]
    \centering
    \begin{subfigure}[b]{0.45\textwidth}
        \caption{$f(x,y) = (x+y)^4 +  (x/2-y/2)^4$}
        \includegraphics[width=\textwidth]{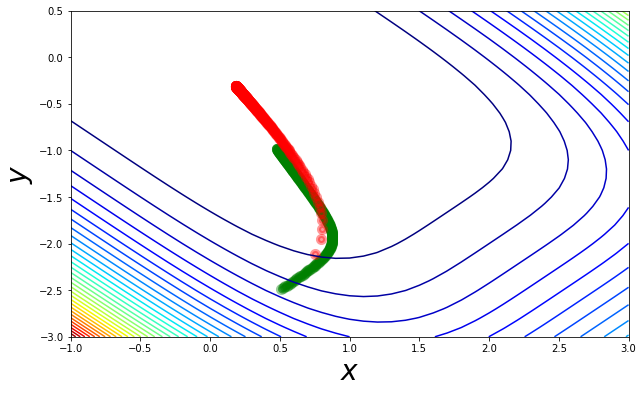}
    \end{subfigure}
    \begin{subfigure}[b]{0.45\textwidth}
        \caption{Rate of convergence}
        \includegraphics[width=\textwidth]{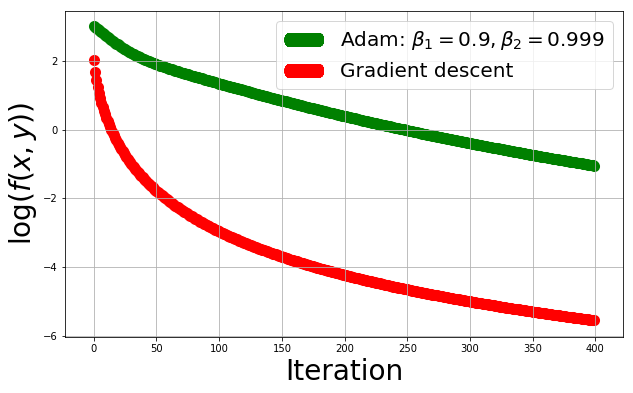}
    \end{subfigure}
    \caption[]
    {\small Comparison between gradient descent and \Adam. Gradient Descent outperforms \Adam{} in this example because $\beta_1, \beta_2$ are large and \Adam{} keeps memory of the past large gradients. Both trajectories start from the point $(0.5, -2.5)$.}
    \label{fig:convergencebis}
\end{figure}

\subsection{On future algorithms}\label{ssec:Future}

Existence and uniqueness of solutions for ODE \eqref{eq:ODE} holds for almost arbitrary functions $h(t)$, $r(t)$, $p(t)$ and $q(t)$. Our convergence analysis, nevertheless, imposes important restriction on the choice of these functions (see Assumptions \ref{ass:TopologicalConvergence}, \ref{ass:Morse}, \ref{ass:necessaryConditions}), which we now discuss.
From Theorem \ref{thm:convergenceRate}, the rate of convergence to the global minimum is at least given by the decay of the function
\[
 \max\left\{1, \int_{t_0}^t q(u)  du\right\} / \c{A}(t),
\]
where $\mathcal{A}$ depends only on $h$ and $r$.
It is natural to seek for functions $h(t)$, $r(t)$, $p(t)$ and $q(t)$
such that this upper-bound decays faster to zero, while satisfying the conditions \ref{ass:NecessaryCoeficients} and \ref{ass:necessaryConditions}.

Note that in Theorem \ref{thm:convergenceRate}, we obtained tighter estimates on the rate of convergence depending on the history of the gradient and the variables $v$ and $\theta$. This suggest that the efficiency of the algorithm
does not only depend on the choice of the functions $h(t)$, $r(t)$, $p(t)$ and $q(t)$ but also on
the path the dynamics is taking and therefore on properties of the loss.

\medskip
\noindent
\textbf{On the flexibility of the coefficients $h(t)$ and $r(t)$.} We start by recalling that assumption \ref{ass:necessaryConditions}(a) implies a strong constraint on the function $r(t)$. Indeed, it is necessary that $\lim_{t\to \infty} r(t) \cdot t^{1+\epsilon} =0$ for all $\epsilon>0$. This is a strong restriction, which have consequences to the choice of $h(t)$. In order to illustrate this, let us consider two examples $r_1(t) =r_1$ and $r_2(t)=r_2/t$ with $r_2>1$, which are the natural allowed examples.
With these choices of functions, we seek for $h$ such that $\c{A}$ has at least linear growth.
In the notation of $\S\S$~\ref{ssec:RateConvergence}, we have that:
\[
\mathcal{B}_1(t) = 1/r_1, \quad \mathcal{B}_2(t) = t/(r_2-1)
\]
 This imposes strong restrictions on $h(t)$, indeed:
\[
\lim_{t\to \infty}h_1(t) >0, \quad \lim_{t\to \infty} t \cdot h_2(t)>0.
\]
It follows that the natural reasonable choices are one of the following three:
\begin{enumerate}
\item $r(t) =r $, $h(t)=h$ (c.f. Heavy Ball, \Adam{} and AdaFom). Under these conditions, one can expect a convergence rate of order at least $\c{O}(1/t)$ and convergence guarantees (including avoidance of saddle points) even in the non-convex setting;
\item $r(t) = r/t$, $h(t) =h$ (c.f Nesterov). Under these conditions, one can expect a fast convergence rate of order at least $\c{O}(1/t^2)$, but we obtain fewer guarantees in the non-convex setting because of the convergence to zero of the function $r$;
\item $r(t)=r/t$ and $h(t)=h/t$.  Under these conditions, one can expect a convergence rate of order at least $\c{O}(1/t)$, but we obtain fewer guarantees in the non-convex setting;
\end{enumerate}
Note that assumptions \ref{ass:NecessaryCoeficients} and \ref{ass:necessaryConditions} impose further relations on the hyper-parameters $r$ and $h$.

\medskip
\noindent
\textbf{On the flexibility of the coefficients $p(t)$ and $q(t)$.} In sharp contrast with the previous analysis, the coefficients $p(t)$ and $q(t)$ require fewer restrictions. In particular, note that assumptions \ref{ass:NecessaryCoeficients}, \ref{ass:TopologicalConvergence} and \ref{ass:necessaryConditions} are almost independent on these coefficients (besides mild constraints on $q(t)$).

Our analysis leads to an interesting dilemma, nevertheless, which deserves further investigation. At the one hand, in Theorem \ref{thm:avoiding} (which guarantee avoidance of saddle points) it is necessary to assume that $\lim_{t \to \infty} q(t) >0$ (c.f. \Adam{}). On the other hand, in order to obtain faster convergence, Theorem \ref{thm:convergenceRate} indicates that the function $q(t)$ should have a fast decay to zero to control the growth of $\int_{t_0}^t q(u)du$. Indeed:
\begin{itemize}
\item[$(i)$] if $\lim_{t \to \infty} q(t) >0$ (c.f \Adam{}), then the denominator $\int_{t_0}^t q(u)du$  has linear growth which degrades the rate of convergence of the algorithm.

\item[$(ii)$] if $\lim_{t \to \infty} t \cdot q(t) >0$ (c.f. AdaFom) then the denominator has logarithmic growth.

\item[$(iii)$] if $\lim_{t \to \infty} t^{1+\epsilon} \cdot q(t)= 0$ for some $\epsilon>0$, then there is no loss in the expected convergence rate.
\end{itemize}

It is, of course, interesting that an optimization algorithm avoids saddle point and converges as fast as possible. This is an intriguing point, which we feel that deserves further empirical investigation.

\medskip
\noindent
\textbf{Final remarks.} From the analysis outlined above, we feel that the combination of choices $1.ii$ (AdaFom), $1.iii$ and $2.ii$ are promising, at least from the perspective of our current analysis, and deserve further empirical investigation.

Moreover, we hope to explore further different choices of functions $p(t)$ and $q(t)$ as well
as the design of hybrid algorithms, which are also supported by this analysis.

\section{Conclusion and final discussion}

The main objective of this work is to provide a theoretical framework to study adaptive algorithms. The proposed continuous dynamical system \eqref{eq:ODE} is flexible enough to encompass commonly used adaptive algorithms  (as we show in Section \ref{sec:Appl}), but
stays specific enough to allow simple proofs and guidelines. Our work shows that adaptive dynamics converge to a critical locus of the loss function but possibly at a slower rate than non-adaptive algorithms. Due to the nature of the adaptivity, the convergence rate is not just a non-increasing function of time but also depends on the gradient history. The performance of adaptive algorithms is linked to
the trajectory taken by the dynamics and properties of the loss function.
We also analyze how different choices of coefficients $h$, $r$, $p$ and $q$ impact on the convergence of the dynamics
and we suggested several possible algorithms to be tested (see $\S\S$~\ref{ssec:Future}). It supports our interest in linking specific choices of adaptive algorithms (more precisely, specific choices of the coefficients $h$, $r$, $p$ and $q$) with properties from the loss function. We intend to pursue this direction in future works.

The deterministic convergence analysis leads to natural conjectures on the convergence in the discrete and stochastic setting. In particular, we believe that the Lyapunov functional \eqref{eq:Lyapunov} can be adapted to the stochastic discrete framework \cite{GPS:16}. We note that, nevertheless, a precise correspondence between results valid for a continuous ODE and the stochastic discrete counterparts is far from being obvious. Indeed, recall that \Adam{} and \RMSprop{} are not always converging in the stochastic setting, even for a convex loss function \cite{RKK:18}. We expect, therefore, new restrictions on the coefficients $h(t)$, $r(t)$, $p(t)$ and $q(t)$, as well as on the loss function and the learning rate. We believe that those conditions will be different compared to SGD.

Finally, the mathematical tools used in this work are theoretically simple, and we hope that our presentation helps to diffuse these techniques. Some beautiful mathematical open problems have naturally appeared (see, for example, the discussion in Remark \ref{rk:assMorse}) and we intend to work on them in future projects.

\medskip

\paragraph{Acknowledgements.} We would like to thank Daniel Panazzolo for a private communication concerning central-stable manifolds.

\appendix

\section{Existence and uniqueness of solutions}\label{sec:Existence}

In this section, we prove Theorems \ref{th:existencePrel} and \ref{th:existence}. We start by a simple case, in order to exemplify the main ideas of the proof, before we turn to the general more technical analysis.

\subsection{The Cauchy problem for $t_0 > 0$ under assumption \ref{ass:NecessaryCoeficients}.}\label{ssec:ExistenceSimplified}

We compute elementary bounds for the solutions of the ODE \eqref{eq:ODE}, under the additional assumption \ref{ass:NecessaryCoeficients} with $\tilde{t}=0$ and that $f$ is bounded from below (or under Assumption \ref{ass:bounded}).

Indeed, let $t_0>0$ and an initial condition $\x(t_0)=\x_0$ be fixed. Because of Assumption \ref{ass:necessaryConditions}, we are in conditions of Picard Theorem, so there exists a solution $\x(t)$ with initial condition $\x(t_0)=\x_0$ and interval of definition $[t_0,t_{\infty}[$. The crucial point with this assumption \ref{ass:NecessaryCoeficients} with $\tilde{t}=0$, is that we can apply Lemma \ref{lem:JustifiesA4} in order to find a constant $\mathcal{K}(\x_0,t_0)$ which depends on the initial condition, such that:
\begin{equation}\label{eq:FirstInequalityDerTheta}
\norm{\frac{m(t)
         }{\sqrt{v(t) + \eps}}}
        \leq \mathcal{K}(\x_0,t_0), \quad \forall t\in [t_0,t_{\infty}[,
\end{equation}
which implies that $\dot{\theta}(t)$ is uniformly bounded. We conclude that:
\begin{equation}\label{eq:FirstInequalityTheta}
 \norm{\theta(t)}
        \leq \norm{\theta_0} + \mathcal{K}(\x_0,t_0)(t-t_0), \quad \forall t\in [t_0,t_{\infty}[.
\end{equation}
It follows that: either $t_{\infty}=\infty$, in which case we are done, or $t_{\infty}<\infty$ and $\theta(t)$ satisfies assumption \ref{ass:bounded}. In order to conclude, it is enough to treat this last case:

\begin{lemma}\label{lem:IntervalDefinition}
Let $t_0 > 0$ and $\x_0=(\theta_0,m_0,v_0)$ be fixed. Under assumptions \ref{ass:ODE} and \ref{ass:bounded}, there exists an unique solution $\x(t) =(\theta(t),m(t),v(t))$ of \eqref{eq:ODE} with initial condition $\x(t_0)=\x_0$, and which is defined for all $t$ in $[t_0,\infty)$. Furthermore, we have $v(t) \geq 0$ for all $t\in [t_0,\infty)$. Furthermore, denoting by $L_g = \sup \{ \norm{\nabla f(\theta(t))}; t\geq t_0\}$, we get:
\begin{equation}\label{eq:BoundsMV}
\begin{aligned}
\norm{m(t)} &\leq \norm{m(t_0)}+ L_g d \int_{t_0}^t h(s) ds,\\ \norm{v(t)}  &\leq \norm{v(t_0)}+ L^2_g d \int_{t_0}^t p(s) ds
\end{aligned}
\end{equation}
where we recall that $d$ stands for the dimension of the space. If we suppose that $r(t) \not\equiv 0$ and $q(t) \not\equiv 0$, furthermore, then:
\[
\begin{aligned}
\norm{m(t)} &\leq \norm{m(t_0)} + L_g d \sup_{s\in [t_0,t]}\left\{ \frac{h(s)}{r(s)} \right\}, \\
\norm{v(t)}  &\leq \norm{v(t_0)}+  L^2_g d \sup_{s\in [t_0,t]}\left\{ \frac{p(s)}{q(s)} \right\}
\end{aligned}
\]
\end{lemma}

\begin{proof}
By assumption \ref{ass:ODE} and classical ODE's, there exists a solution $\x(t)=(\theta(t),m(t),v(t))$ of system \eqref{eq:ODE} with maximal interval of definition $[t_0,T)$ and initial conditions $\x(t_0)=\x_0=(\theta_0,m_0,v_0)$. Now, consider the functions:
\[
\begin{aligned}
a(t) = \exp\left(\int_{t_0}^t r(s) ds\right), \quad b(t) = \exp\left(\int_{t_0}^t q(s) ds\right)
\end{aligned}
\]
which are increasing functions bigger than $1$ (for all $t \geq t_0$). We note that:
\[
\begin{aligned}
\frac{d }{dt}(m\cdot a(t)) = a(t)h(t) \nabla f(\theta), \quad \frac{d }{dt}(v\cdot b(t)) = b(t) p(t) \nabla f(\theta)^2.
\end{aligned}
\]
In particular, we easily conclude that
\[
\begin{aligned}
m(t) &= \frac{1}{a(t)}\left(m_0+ \int_{t_0}^t a(s) h(s) \nabla f(\theta) ds \right)\\
v(t) &=  \frac{1}{b(t)}\left(v_0+ \int_{t_0}^t b(s) p(s) \nabla f(\theta)^2 ds \right)
\end{aligned}
\]
Next, under assumption \ref{ass:bounded}, we can assume that $|\nabla f (\theta(t))| \leq L_g$ for some positive real number $L_g$. It is now easy to get inequalities \eqref{eq:BoundsMV}, which lead to:
\[
\begin{aligned}
|m_i(t)| &\leq |m_i(t_0)|+ L_g \,h(t_0)\, (t-t_0)\\
|v_i(t)| &\leq |v_i(t_0)|+ L_g^2 \,p(t_0)\, (t-t_0)\\
|v_i(t)| &\geq \frac{1}{b(t)} v_0 > 0
\end{aligned}
\]
and we easily conclude that $T=\infty$. Finally, if $r(t)\not\equiv 0$, we get by direct integration:
\[
\begin{aligned}
\left|m_i(t)\right|  &\leq \left|m_i(t_0)\right| +L_g \frac{1}{a(t)}\int_{t_0}^t r(s) a(s) \frac{h(s)}{r(s)} ds\\
& \leq \left|m_i(t_0)\right| +L_g \sup_{s\in [t_0,t]}\left\{ \frac{h(s)}{r(s)} \right\} \frac{1}{a(t)} \int_{t_0}^t r(s) a(s) ds\\
& =\left|m_i(t_0)\right| + L_g \sup_{s\in [t_0,t]}\left\{ \frac{h(s)}{r(s)} \right\} \frac{a(t) - a(t_0)}{a(t)} \leq L \sup_{s\in [t_0,t]}\left\{ \frac{h(s)}{r(s)} \right\}
\end{aligned}
\]
A similar computation holds whenever $q(t) \not\equiv 0$, which concludes the Lemma.
\end{proof}

In order to prove Theorems \ref{th:existencePrel} without assumption \ref{ass:NecessaryCoeficients} with $\tilde{t}=0$, it is necessary to obtain the estimate \eqref{eq:FirstInequalityDerTheta}. This is possible whenever $p(t) \not\equiv 0$, as we will show in Lemma \ref{lemma:bound_m_over_sqrtv} below. Before turning to its proof, nevertheless, we recall some variations of Gronwall's Lemma which are used in our study.

\subsection{Gronwall's Lemma}

\begin{lemma}[Gronwall's Lemma]
  Let $T>0$, $\lambda \in L^1(0,T)$, $\lambda \geq 0$ almost everywhere and $C_1, C_2 \geq 0$.
  Let $\varphi \in L^1(0, T)$, $\varphi \geq 0$ almost everywhere, be such that $\lambda \varphi \in L^1(0, T)$
  and
  \[
  \varphi(t) \leq C_1 + C_2\int_0^t \lambda(s) \varphi(s)ds
  \]
  for almost every $t \in (0, T)$. Then we have
  \[
  \varphi(t) \leq C_1\exp\left( C_2 \int_0^t \lambda (s)ds\right)
  \]
 \end{lemma}

\begin{lemma}\label{nonlinear_gronwall}
    Let $\varphi: [t_0, t_1] \to \R_{>0}$ be absolutely continous stricly non-negative function
    and suppose $\varphi$ obeys the differential inequality for $0\leq \alpha \leq 1$
    \[
        \varphi'(t) \leq \beta(t) \varphi^{\alpha}(t)
    \]
for almost every $t \in [t_0, t_1]$, where $\beta$ is continuous. Then for all $t \in [t_0, t_1]$  and all $0\leq \alpha < 1$
    \[
        \varphi(t) \leq \left[ \varphi(t_0)^{1-\alpha} + \int_{t_0}^t
        (1-\alpha)\beta(s) ds \right]^{1/(1- \alpha)}.
    \]
    If $\alpha =1$ then
    \[
        \varphi(t) \leq  e^{\int_{t_0}^t \beta(s)ds} \varphi(t_0).
    \]
\end{lemma}
\begin{proof}
Note that
    \[
    [ \varphi^{1-\alpha} ] ' = (1-\alpha) \varphi^{-\alpha} \varphi' \leq
    (1-\alpha) \varphi^{-\alpha}(t) \beta(t) \varphi^{\alpha}(t)  = (1-\alpha)
    \beta(t)
    \]
Integrating over time gives
\[
    \varphi(t) \leq \left[ \varphi(t_0)^{1-\alpha} + \int_{t_0}^t (1-\alpha)
    \beta(s)ds  \right]^{1/(1-\alpha)}
    \]
\end{proof}

\begin{lemma}\label{nonlinear_gronwall2}
    Let $\varphi: [t_0, t_1]  \to \R_{>0} $ be absolutely continous stricly non-negative function
    and suppose $\varphi$ obeys the differential inequality for $0\leq \alpha < 1$
    \[
        \varphi'(t) \leq \gamma(t)\varphi(t) +  \beta(t) \varphi^{\alpha}(t)
    \]
for almost every $t \in [t_0, t_1]$, where $\beta, \gamma$ are continuous. Then for all $t \in [t_0, t_1]$
    \[
        \varphi(t) \leq \left[ e^{(1-\alpha) \int_{t_0}^t \gamma(s)ds } \varphi(t_0)^{1-\alpha} + \int_{t_0}^t
        (1-\alpha)e^{(1-\alpha) \int_{s}^t \gamma(u)du } \beta(s) ds \right]^{1/(1- \alpha)}.
    \]
\end{lemma}

\begin{proof}
Note that
    \begin{align*}
        [ \varphi^{1-\alpha} ] ' &= (1-\alpha) \varphi^{-\alpha} \varphi'\\& \leq
    (1-\alpha) \varphi^{-\alpha}(t) \left( \gamma(t) \varphi(t)+ \beta(t) \varphi^{\alpha}(t)
        \right) \\& = (1-\alpha) \gamma(t) \varphi^{1-\alpha}(t) +  (1-\alpha)
    \beta(t)
\end{align*}
and we conclude using the standard Gronwall Lemma applied to
  $ \varphi^{1-\alpha}$ .
\end{proof}

\subsection{A priori estimates and global solution}\label{ssec:a_priori}

We now turn to more precise estimates of the functions $\x(t)=(\theta(t),m(t),v(t))$ in order to prove existence and uniqueness of the solution of \eqref{eq:ODE}. We start by controlling the derivative of $\theta$:

\begin{lemma}\label{lemma:bound_m_over_sqrtv}
Let $0<t_0<T <\infty$ be fixed and suppose that $p(t) \not\equiv 0$. For any $s, t \in [t_0, T]$ such that $ s \leq  t$, we have that:
    \begin{align*}
        \norm{\frac{m(t)
         }{\sqrt{v(t) + \eps}}}^2
        \leq e^{ \int_s^t  q(u)- 2r(u) du }
        \norm{\frac{m(s)}{\sqrt{v(s) + \eps}}}^2  +
d \int_s^t e^{ \int_u^t q(a)
         - 2r(a) da }   \frac{h^2(u)}{ p(u) }     du
    \end{align*}
\end{lemma}
\begin{proof}
It follows from direct computation that:
     \begin{align*}
         \frac{d}{dt} \frac{1}{2} \norm{\frac{m
         }{\sqrt{v + \eps}}}^2
         &= h(t)\ps{\frac{m}{\sqrt{v+ \eps}}}{
             \frac{\nabla f\left(\theta\right) }{\sqrt{v + \eps}}}
         - r(t) \norm{\frac{m}{\sqrt{v + \eps}}}^2  \\&
         - \frac{1}{2} p(t) \norm{\frac{m \odot  \nabla
         f(\theta)}{v + \eps}}^2
         + \frac{q}{2} \norm{\frac{m \odot \sqrt{v}}{v +\eps}}^2
     \end{align*}
Now, note that
 \begin{align*}
             h(t)\ps{\frac{m}{\sqrt{v+ \eps}}}{
             \frac{\nabla f\left(\theta\right) }{\sqrt{v + \eps}}}
- \frac{p(t) }{2} \norm{\frac{m \odot \nabla
         f(\theta)}{v + \eps}}^2
    &= h(t)\sum_{i=1}^d \frac{m_i\partial_i f\left(\theta\right)
}{v_i+ \eps}
     - \frac{p(t) }{2}  \sum_{i=1}^d \left( \frac{m_i
     \partial_i
         f(\theta)}{v_i + \eps} \right)^2\\
     & = - \frac{p(t) }{2} \sum_{i=1}^d \left( \frac{m_i
     \partial_i
         f(\theta)}{v_i + \eps}  -
         \frac{h(t)}{p(t)} \right)^2 +
     \frac{h^2}{2p(t)}d \\
     & = - \frac{p(t) }{2} \norm{ \frac{m \odot \nabla
         f(\theta )}{v + \eps}-
     \frac{h(t)}{p(t)} }^2  + \frac{h^2(t)}{2 p(t)} d
    \end{align*}
     where $d$ is the dimension of the state space. Hence,
         \begin{align*}
         \frac{d}{dt} \frac{1}{2} \norm{\frac{m
             }{\sqrt{v + \eps}}}^2 &=
     - \frac{p(t) }{2} \norm{ \frac{m \odot \nabla
         f(\theta)}{v + \eps}-
     \frac{h(t)}{p(t)} }^2  + \frac{h^2(t)}{2
             p(t)} d \\&
         - r(t) \norm{\frac{m}{\sqrt{v + \eps}}}^2
               + \frac{q(t)}{2} \norm{\frac{m \odot
         \sqrt{v}}{v +\eps}}^2 \\
             &\leq   \frac{h^2(t)}{2
             p(t)} d+
             \left( \frac{q(t)}{2} - r(t) \right) \norm{\frac{m}{\sqrt{v + \eps}}}^2
     \end{align*}
     and we easily conclude from Gronwall's Lemma.
 \end{proof}

The above Lemma allow us to control $\theta(t)$. The next Lemma replaces Assumption \ref{ass:bounded} in the existence proof:

\begin{lemma}\label{norm_theta}
Let $0<t_0<T <\infty$ be fixed. For any $s, t \in [t_0, T]$ such that $ s \leq  t$:
    \[
         \norm{\theta(t)}^2 \leq   \left[  \norm{\theta(s)}    +  \int_s^t  \norm
        {\frac{m}{\sqrt{v + \eps } }}du \right]^{2} .
        \]
    and, in particular, if $p(t) \not\equiv 0$:
    \begin{align*}
        \norm{\theta(t)}& \leq   \norm{\theta(s)}    +
        \norm{\frac{m(s)}{\sqrt{v(s) + \eps}}}\int_{s}^t  e^{
            \frac{1}{2} \int_{s}^u  q(a)- 2r(a) da }du
        \\& +
        \sqrt{d} \int_{s}^t \left( \int_{s}^u e^{ \int_a^u  q(b)
         - 2r(b) db }   \frac{h^2(a)}{ p(a) }     da \right)^{1/2} du
\end{align*}
\end{lemma}
\begin{proof}
From the Cauchy-Schwarz inequality, we obtain
\[
\frac{d}{dt}\frac{1}{2} \norm{\theta(t)}^2 = - \ps{\theta}{ \frac{m}{\sqrt{v + \eps}}} \leq \norm{\theta} \norm{\frac{m}{\sqrt{v + \eps}}}.
\]
    We apply Lemma \ref{nonlinear_gronwall} to $\varphi(t) = \frac{1}{2}\norm{\theta(t)}^2$,
    $ \beta(t) = \sqrt{2} \norm {\frac{m}{\sqrt{v  + \eps } }}$ and $\alpha =1/2$ in order to get the first inequality. The second inequality follows from the first together with the estimate of Lemma \ref{lemma:bound_m_over_sqrtv}.
\end{proof}

We complete this section by improving the estimates on $m(t)$ and $v(t)$ which were given in Lemma \ref{lem:IntervalDefinition}:

    \begin{lemma}\label{apriori_estimates}
        Let $0<t_0<T <\infty$ be fixed. For any $s, t \in [t_0, T]$ such that $ s \leq  t$:
 \begin{align*}
        \norm{m(t)}^2 \leq \left[ e^{ - \int_{t_0}^t r(s)ds }
         \norm{m(t_0)}+ \int_{t_0}^t
         e^{- \int_{s}^t r(u)du } h(s) \norm{\nabla
        f(\theta) } ds \right]^{2}.
    \end{align*}
    \end{lemma}

\begin{proof}
From Cauchy Schwarz,
\begin{align*}
         \frac{d}{dt} \frac{1}{2} \norm{m(t)}^2
          &= h(t)\ps{m }{\nabla f\left(
          \theta \right)} - r(t) \norm{m}^2\\
          & \leq  h(t) \norm{m } \norm{\nabla f\left(
          \theta \right)} - r(t) \norm{m}^2
    \end{align*}
We now just need to apply Lemma \ref{nonlinear_gronwall2} in order to conclude.
\end{proof}

    \begin{lemma}\label{lem:EstimatesV}
        Let $0<t_0<T <\infty$ be fixed. For any $s, t \in [t_0, T]$ such that $ s \leq  t$:
 \begin{align*}
     \norm{v^{1/2}(t)}^2 \leq  e^{ - \int_{t_0}^t q(s)ds }
         \norm{v^{1/2}(t_0)}^2 + \int_{t_0}^t
         e^{- \int_{s}^t q(u)du } p(s) \norm{\nabla
        f(\theta) }^2 ds,
    \end{align*}
and
\[
v(t) \geq e^{- \int_{t_0}^t q(s)ds} v_0. 
\]
    \end{lemma}

    \begin{proof}
    Note that
        \begin{align*}
            \frac{d}{dt}  \norm{v^{1/2}(t)}^2
            & = \ps{  \frac{p(t) \left[ \nabla f(\theta) \right]^2 - q(t) v}{v^{1/2}}
             }{v^{1/2}} \\
             & = p(t) \norm{\nabla f(\theta)}^2  - q(t) \norm{v^{1/2}}^2
        \end{align*}
        and the first inequality easily follows from the Gronwall lemma.  In order to get the second inequality, note that
\[
\dot{v} = p(t) [\nabla f(\theta)]^2- q(t)v  \geq - q(t)v,
\]
and we just need to apply Gronwall's Lemma once again to conclude.
    \end{proof}

\subsection{Existence and uniqueness for $t_0>0$: proof of Theorem \ref{th:existencePrel}}

Indeed, let $t_0>0$ and an initial condition $\x(t_0)=\x_0$ be fixed. Because of Assumption \ref{ass:necessaryConditions}, we are in conditions of Picard Theorem, so there exists a solution $\x(t)$ with initial condition $\x(t_0)=\x_0$ and interval of definition $[t_0,t_{\infty}[$. If $t_{\infty} = \infty$, we are done, so suppose by contradiction that $t_{\infty} < \infty$. Then: if $p(t)\not\equiv 0$, by Lemma \ref{norm_theta} we conclude that $\theta(t)$ is bounded; otherwise, assumption \ref{ass:NecessaryCoeficients} with $\tilde{t}=0$ is satisfied and $f$ is bounded by below, and by inequality \eqref{eq:FirstInequalityTheta} we conclude that $\theta(t)$ is bounded. In either case we are in conditions to apply Lemma \ref{lem:IntervalDefinition} which implies that $t_{\infty}=\infty$, a contradiction.

\subsection{Existence and uniqueness for $t_0=0$}
In the previous section, we proved that for all $T > 0 $, there exists a unique solution
to the system \eqref{eq:ODE} in the space $C^1([t_0,T]; \R^n)$
for any strictly positive time $t_0$. The purpose of this section is to extend this result to solutions starting at $t_0=0$.
Classical results on differential equations do not apply directly here because
the functions $h, r, k, q$ are
allowed to have a pole of order one at $t=0$ (see Assumption $\ref{ass:Existencet0}$).

We follow a standard argument in dynamical systems: we will approximate the solution of ODE \eqref{eq:ODE} by a sequence of functions with good convergence properties. In this section, we limit ourselves to describing which is the sequence of functions, and we leave the ``good convergence properties" for later on. Indeed, we consider the orbits $\x_{\delta}$, for $\delta >
0$, which are solution to the equation
 \begin{equation}\label{eq:ODEsmoothed}
 \left\{
  \begin{aligned}
      \dot{\theta_{\delta}}(t) &= - m_{\delta}(t) / \sqrt{v_{\delta}(t) + \eps}\\
      \dot{m_{\delta}}(t) &=  h_{\delta}(t)\nabla f(\theta_{\delta}(t)) - r_{\delta}(t)m_{\delta}(t)  \\
      \dot{v_{\delta}}(t) &=  p_{\delta}(t) \left[ \nabla f(\theta_{\delta}(t)\right] ^2 - q_{\delta}(t)
      v_{\delta}(t),
  \end{aligned}
\right.
 \end{equation}
where
\[
    h_{\delta}(t) = h(\max \left(\delta, t\right))
    \]
and similar formulas hold for $r_{\delta}, p_{\delta}$ and $q_{\delta}$. Those functions are continuous and locally Liptchitz in time, and defined for every $t>0$ by the previous section. Note that for every $T>0$, $\x_{\delta} \in C([0,T]; \R^n)$, and is $C^1$ everywhere outside $t=\delta$.

In order to show that this family of functions converge, we use Arzela-Ascoli Theorem. The next section is dedicated to proving that the hypothesis of Arzela-Ascoli are verified.

\subsubsection{Equicontinuity and uniform boundedness}
We prove in this section, that the family of functions $\x_{\delta}$ is equicontinuous and uniformly bounded, where $\x_{\delta}$ is the solution to \eqref{eq:ODEsmoothed}. This allow us to apply Arzela-Ascoli in the end of this subsection in order to get the candidate for a solution of ODE \eqref{eq:ODE}.

The key result is the following proposition whose proof is left to subsection \ref{sse:supporting_proofs}
\begin{proposition}\label{Holder_continuity}
   If assumptions \ref{ass:Existencet0} is satisfied, then there exists a positive constant $C_2(T)$, independent of $\delta$, such that for all $t, s \in [0, T]$
     \[
         \norm{\x_{\delta}(t) - \x_{\delta}(s)}^2 \leq C_2(T) (t-s)^2.
    \]
 \end{proposition}

As a consequence of the previous Proposition, we can control the norm of the solution $\x_{\delta}$; this is done in terms of an special norm (instead of the usual one). More precisely, let us recall the notion of fractional Sobolev space.
For a real number $0 < \delta < 1$ and $p \geq 1$, we denote by $W^{\alpha,
p}([0, T])$ the fractional Sobolev space of functions $u \in L^p(0,T)$
satisfying
\[
    \int_0^T \int_0^T \frac{\norm{u(t) -u(s)}^p}{\abs{t-s}^{\delta p +1}}dsdt
    < +\infty
\]
The space $C^{\gamma}([t_0,T]; \R^{3d})$ is the space of H\"older continuous function of order
$\gamma>0$ on $[t_0, T]$ with values in $\R^{3d}$. It follows that
 \begin{lemma}\label{lem:UniformBoundedFamily}
 If assumptions \ref{ass:Existencet0} is satisfied, there exists a positive constant $C_3(T)$, independent of $\delta$, such that 
 \[
     \norm{\x_{\delta}}^2_{W^{\gamma, 2}} \leq C_3(T)
    \]
for any $\gamma <1$.
 \end{lemma}
 \begin{proof}
     The proof is a direct consequence of Lemma \ref{Holder_continuity}. Indeed
    \begin{align*}
        \int_0^T \int_0^T \frac{ \norm{\x_{\delta}(t) -
         \x_{\delta}(s)}^2}{\abs{t - s}^{2\gamma + 1}}  dsdt &\leq C_2(T)
         \int_0^T \int_0^T \frac{(t-s)^2}{\abs{t - s}^{2\gamma + 1}}dsdt  < +\infty
    \end{align*}
where the last inequality holds if and only is $\gamma <1$.
 \end{proof}

We now use the Sobolev embedding $W^{\gamma, 2}([0,T])
\xhookrightarrow{} C^{\alpha}([0,T]) $ for $\gamma -\alpha >1/2$ and $\gamma <1$,
which implies $\alpha <1/2$.
It follows that the family $\x_{\delta} \in C^{\alpha}([0, T], \R^n)$. From Lemma \ref{lem:UniformBoundedFamily}, we conclude that the family is uniformly bounded. Finally, the family is equi-continuous because of the definition of the norm in $C^{\alpha}$ and its uniform bound in $\delta$.

Applying Arzela Ascoli Theorem, we deduce that there exists a converging sub-sequence (still denoted $\x_{\delta}$) in $C([0, T], \R^n)$. We denote by $\widehat{\x}$ its
limit and we prove in the next section that $\widehat{\x}$
satisfies Equation \eqref{eq:ODE}.

\subsubsection{Identification of the limit and uniqueness of the solution}

\paragraph{Existence}
 The convergence of the initial
conditions are a direct consequence of the uniform convergence (which implies point-wise convergence at every point). Now fix $T>0$; it is clear that the ODE \eqref{eq:ODEsmoothed} converges uniformely to ODE \eqref{eq:ODE} when $\delta \to 0$ in a neighbourhood of $t=T$ (indeed, for $\delta<<T$, the two differential equations are equal). Since $x_{\delta}$ converges uniformly to $\hat{x}$, we conclude that $\hat{x}$ is a solution of of ODE \eqref{eq:ODE} in a neighbourhood of $t=T$. Since $T>0$ was arbitrary, we conclude the result.

\paragraph{Uniqueness}
We proceed by contradiction. Assume there exist two solutions
$\x = (\theta, m, v)$ and $\y = (\psi, n, w)$ to the system \eqref{eq:ODE}.

An easy computation shows that for all $0 \leq t \leq T $ (because $v$ and $w$ are lower bounded, see Lemma \ref{lem:EstimatesV})
\begin{align*}
    \norm{\theta(t) - \psi(t)} & \leq  \int_0^t \norm{  \frac{m}{\sqrt{v + \eps}} -
    \frac{n}{\sqrt{w+\eps }} } ds  \\&
     \leq  C \int_0^t \norm{m-n} + \norm{n}\norm{w-v} ds
\end{align*}
By continuity of the solution of equation \eqref{eq:ODE} on $[0,T]$, we know that there exists a constant $\widetilde{C}$ such that
for all $s \leq t$
\[
\norm{n(s)} \leq \widetilde{C}
\]
and therefore
\begin{align}\label{eq:theta_uniq}
    \norm{\theta(t) - \psi(t)} \leq C \int_0^t \norm{m-n} + \widetilde{C} \norm{y-v} ds.
\end{align}
Now, consider the functions
\[
\begin{aligned}
a_{\eta}(t) = \exp\left(\int_{\eta}^t r(s) ds\right), \quad b_{\eta}(t) = \exp\left(\int_{\eta}^t q(s) ds\right)
\end{aligned}
\]
which are increasing functions bigger than $1$ (for all $t \geq \eta >0$). We note that:
\[
\begin{aligned}
\frac{d }{dt}(m\cdot a_{\eta}(t)) = a_{\eta}(t)h(t) \nabla f(\theta), \quad \frac{d }{dt}(v\cdot b_{\eta}(t)) = b_{\eta}(t) p(t) \nabla f(\theta)^2.
\end{aligned}
\]
In particular, we easily conclude that
\[
\begin{aligned}
m(t) &= \frac{1}{a_{\eta}(t)}\left(m(\eta)+ \int_{\eta}^t a_{\eta}(s) h(s) \nabla f(\theta) ds \right)\\
v(t) &=  \frac{1}{b_{\eta}(t)}\left(v(\eta)+ \int_{\eta}^t b_{\eta}(s) p(s) \nabla f(\theta)^2 ds \right)
\end{aligned}
\]
It follows from Assumptions \ref{ass:Existencet0} and inequality \eqref{eq:theta_uniq}, that for all $\eta \leq t \leq T $,
\begin{align*}
   & \norm{m(t) - n(t)}
    \\&= \norm{\frac{1}{a_{\eta}(t)}\left(m(\eta) - n(\eta) \right) + \frac{1}{a_{\eta}(t)} \int_{\eta}^t
       a_{\eta}(s) h(s) \left( \nabla f(\theta)  -  \nabla f(\psi)  \right)ds }\\&
               \leq   \norm{m(\eta) - n(\eta) }
               +  C_1 \int_{\eta}^t   h(s) \int_{0}^s \norm{m-n} + \widetilde{C}\norm{v-w} du ds \\&
               \leq  \norm{m(\eta) - n(\eta) }
               +  C_1  \left(\sup_{0\leq u \leq t}\norm{m-n} + \widetilde{C}\sup_{0\leq u \leq t}\norm{v-w} \right) \int_{\eta}^t  s\cdot  h(s) ds
\end{align*}
By continuity of the process $m$ and $n$, the fact that $m_0= n_0$ and the continuity of $s \mapsto sh(s)$ on $[0,t]$, we obtain by taking the limit when $\eta$ goes to zero that, apart from increasing $C_1$,
\begin{align*}
   & \norm{m(t) - n(t)}
        \leq   C_1 t  \left(\sup_{0\leq u \leq t}\norm{m-n} + \widetilde{C}\sup_{0\leq u \leq t}\norm{v-w} \right).
\end{align*}
Similarly there is a constant $C_2$ such that
\begin{align*}
   \norm{v(t) - w(t)}
    \leq
         C_2 t  \left(\sup_{0\leq u \leq t}\norm{m-n} + \widetilde{C}\sup_{0\leq u \leq t}\norm{v-w} \right) .
\end{align*}

Hence, by combining all bounds there exists two constants, still denoted $C_1$ and $C_2$ such that
\begin{align*}
    \norm{m(t) - n(t)} &+ \norm{v(t)-w(t)} + \norm{\theta(t) - \psi(t)} \\
    & \leq C_1 t  \sup_{0 < u \leq t} \norm{m - n}+  C_2t\sup_{0 < u \leq t}
    \norm{v - w} .
\end{align*}

Since there exists a $t>0$ such that $C_1t$ and $C_2t$ are strictly smaller than $1$, this inequality yields a contradiction. We conclude that the solution must be unique.

\subsubsection{Proof of Proposition  \ref{Holder_continuity}}\label{sse:supporting_proofs}

We start by a preliminary estimate, which extends Lemma \ref{lemma:bound_m_over_sqrtv} to an uniform bound in terms of $\delta$:

 \begin{lemma}\label{uniform_bound_m_over_v}
Suppose that $p(t)\not\equiv 0$. There exists two constants $K_1$ and $K_2$ such that for all $t \in [0, T]$ and all $\delta >0$ sufficiently small
    \begin{align*}
      \norm{\frac{m_{\delta}(t)
         }{\sqrt{v_{\delta}(t) + \eps}}}^2
          \leq K_1 \norm{\frac{m_0}{\sqrt{v_0 + \eps}}}^2+ K_2.
    \end{align*}
 \end{lemma}

\begin{proof}
From Lemma \ref{lemma:bound_m_over_sqrtv} and assumption \ref{ass:Existencet0} (which implies that $\delta h_{\delta}(\delta)$,  $\delta
q_{\delta}(\delta)$, $\delta r_{\delta}(\delta)$ and
$\frac{h_{\delta}(\delta)}{p_{\delta}(\delta)}$ are bounded for $\delta<1$), there exits a constants $K_1\geq 0$ and $K_2\geq 0$ such that for every $\delta \leq 1$ and $t<\delta$, we have:
    \begin{align}\label{eq:bound_smoothed}
        \nonumber \norm{\frac{m_{\delta}(t)
         }{\sqrt{v_{\delta}(t) + \eps}}}^2
        & \leq e^{ \int_0^{\delta} q_{\delta}(\delta)-
        2r_{\delta}(\delta) du }
        \norm{\frac{m_0}{\sqrt{v_0 + \eps}}}^2  +
        d \int_0^{\delta} e^{ \int_u^{\delta} q_{\delta}(\delta)
         - 2 r_{\delta}(\delta) da }   \frac{h_{\delta}^2(\delta)}{  p_{\delta}(\delta) }
         du\\
        & = e^{ \delta \left(  q_{\delta}(\delta)-
        2r_{\delta}(\delta) \right) }
        \norm{\frac{m_0}{\sqrt{v_0 + \eps}}}^2  +
        d \frac{e^{ \delta \left(   q_{\delta}(\delta)
         - 2r_{\delta}(\delta)\right) }-1  }{  q_{\delta}(\delta)  -
         2 r_{\delta}(\delta) }  \frac{h_{\delta}^2(\delta)}{
         p_{\delta}(\delta) }\\ \nonumber
         & \leq K_1 \norm{\frac{m_0}{\sqrt{v_0 + \eps}}}^2+ K_2.
    \end{align}
   Moreover from Lemma \ref{lemma:bound_m_over_sqrtv} and assumption \ref{ass:Existencet0} (which implies that $q_{\delta}(u)-2r_{\delta}(u) <0$, $h_{\delta}(t)/p_{\delta}(u)$ and $h_{\delta}(t)/r_{\delta}(u)$ are bounded for $\delta$ and $u$ small), there exits a constants $\widetilde{K}_1\geq 0$ and $\widetilde{K}_2\geq 0$ such that for every $\delta >0$ small enough and $t\geq \delta$, we have:
    \begin{align}\label{eq:bound_smoothed2}
        \nonumber \norm{\frac{m_{\delta}(t)
         }{\sqrt{v_{\delta}(t) + \eps}}}^2
        & \leq e^{ \int_{\delta}^{t} q_{\delta}(u)-
        2r_{\delta}(u) du }
        \norm{\frac{m_\delta(\delta)}{\sqrt{v_{\delta}(\delta) + \eps}}}^2  +
         d \int_{\delta}^t e^{ \int_u^{t} q_{\delta}(a)
         - 2 r_{\delta}(a) da }   \frac{h_{\delta}^2(u)}{  p_{\delta}(u) }
         du\\
        & \leq        \norm{\frac{m_\delta(\delta)}{\sqrt{v_{\delta}(\delta) + \eps}}}^2  +
        d \sup_{\delta<u<t} \frac{h_{\delta}(u)}{p_{\delta}(u)} \sup_{\delta<u<t} \left|\frac{h_{\delta}(u)}{q_{\delta}(u)
         - 2 r_{\delta}(u)}\right|\\ \nonumber
         & \leq \widetilde{K}_1 \norm{\frac{m_\delta(\delta)}{\sqrt{v_{\delta}(\delta) + \eps}}}^2+ \widetilde{K}_2,
    \end{align}
    Now, combining the two inequalities, (and apart from increasing $\widetilde{K}_1$ and $\widetilde{K}_2$) we get:
\[
\norm{\frac{m_{\delta}(t)
         }{\sqrt{v_{\delta}(t) + \eps}}} \leq \widetilde{K}_1 \norm{\frac{m_0}{\sqrt{v_0 + \eps}}}+ \widetilde{K}_2.
\]
as we wanted to prove.
\end{proof}

We are now ready to prove Proposition \ref{Holder_continuity}. The proof uses the integral formulation and weighted space. First, we define the following norm for all $0 < t \leq T$
     \begin{align*}
         N(t,\delta)  &= \sup_{0 < u \leq t} \norm{h_{\delta}(u) \nabla
    f(\theta_{\delta}(u)) - r_{\delta}(u)m_{\delta}(u) }   \\& +
         \sup_{0 < u \leq t}  \norm{p_{\delta}(u)\left[ \nabla f(\theta_{
         \delta}(u)) \right]^2  - q_{\delta}(u) v_{\delta}(u)  }\\
         & +   \sup_{0 < u \leq t} \norm{ \frac{m_{\delta}(u)}{\sqrt{v_{\delta}(u) +
         \eps}}}.
     \end{align*}
We claim that there exists a constant $C(T)$ (independent of $\delta$) such that $N(t,\delta)\leq C(T)$ for all $t\in (0,T]$. Note that Proposition \ref{Holder_continuity} immediately follows from the claim and the following inequality
\[
\|\x_{\delta}(t) - \x_{\delta}(s)\|^2 \leq \left(\int_s^t \|\dot{\x}_{\delta}(u)\|du \right)^2 \leq N(T,\delta)^2(t-s)^2.
\]
We now turn to the proof of the claim.

    \paragraph{The case $t \leq \delta$.}
    For all $t \leq \delta$, the
    functions $r_{\delta}$ and $q_{\delta}$ are constant and the
    equations for $m_{\delta}$ and $v_{\delta}$, given by system \eqref{eq:ODEsmoothed}, have the equivalent Duhamel formulation given by
     \begin{align}\label{Duhamel_local_time}
         m_{\delta}(t) &= e^{- t r_{\delta}(\delta)} m_0 + e^{- t r_{\delta}(\delta)}
         \int_0^t e^{u r_{\delta}(\delta)} h_{\delta}(\delta) \nabla
         f(\theta_{\delta}(u))  du \\
         v_{\delta}(t) &= e^{- t q_{\delta}(\delta)} v_0 + e^{- t q_{\delta}(\delta)}
         \int_0^t e^{u q_{\delta}(\delta)} p_{\delta}(\delta)  \left[\nabla
         f(\theta_{\delta}(u))\right]^2  du.
     \end{align}

From Lemma \ref{uniform_bound_m_over_v}, we know that $\norm{m_{\delta}(t)/
          \sqrt{v_{\delta}(t) + \eps}}^2$ is uniformly bounded with respect to $\delta$.
Moreover, $\norm{\theta_{\delta}(t)}$ is uniformly
bounded; indeed

\begin{align}\label{eq:bound_Thetasmoothed}
  \norm{\theta_{\delta}(t) - \theta_0} \leq \int_0^t \norm{\frac{m_{\delta}(u)
         }{\sqrt{v_{\delta}(u) + \eps}}}du \leq t \left(K_1 \norm{\frac{m_0}{\sqrt{v_0 + \eps}}}+ K_2\right).
\end{align}

Next, consider the first term which appears in $N(t,\delta)$. From the Duhamel formulation \eqref{Duhamel_local_time}, the triangle inequality and the fact that the initial condition $m_0 = \nabla f(\theta_0) \lim_{t \to 0^+} h(t)/r(t)$, we obtain an upper bound of the form
\begin{align*}
\norm{h_{\delta}(\delta)\nabla f(\theta_{\delta}(t))  -r_{\delta}(\delta)
     m_{\delta}(t)  } \leq N_1 + N_2 + N_3 
\end{align*}
where:
    \begin{align*}
     N_1 &= \norm{h_{\delta}(\delta)\left( \nabla f(\theta_{\delta}(t)) - \nabla f(\theta_0)\right) } \\
      N_2 &= \norm{r_{\delta}(\delta) e^{-t r_{\delta}(\delta)}  \left( m_0 - \frac{h_{\delta}(\delta)}{r_{\delta}(\delta)}\nabla
      f(\theta_0)\right) }
      \\ N_3&=  \norm{r_{\delta}(\delta) e^{- t r_{\delta}(\delta)} \int_0^t e^{u
        r_{\delta}(\delta)} h_{\delta}(\delta) \left( \nabla  f(\theta_{\delta}(u))
        - \nabla  f(\theta_0) \right)du}
 \end{align*}
 We now show that each one of these terms are bounded uniformly in terms of $\delta$.

The term $N_1$ is bounded because $f$ is $C^2$ (and therefore, the gradient is locally Lipschitz) and $\theta_{\delta}(t)$ is uniformly bounded by inequality \eqref{eq:bound_Thetasmoothed}; in particular, denote by $L$ the Lipschitz constant of $\nabla f$ in the compact set containing all solutions $\theta_{\delta}(t)$ for  bounded $t$. More precisely, by the Duhamel formula \eqref{Duhamel_local_time} and Lemma \ref{uniform_bound_m_over_v}
 \begin{align*}
     N_1 \leq \delta h_{\delta}(\delta) L \left(K_1 \norm{\frac{m_0}{\sqrt{v_0 + \eps}}}^2+ K_2\right),
 \end{align*}
 and we easily conclude that $N_1$ is uniformly bounded by assumption \ref{ass:Existencet0}.

The term $N_2$ is bounded from the choice of the initial condition and the fact that $h(t)/r(t)$ is a $C^1$ function. More precisely
\[
N_2 \leq \delta r_{\delta}(\delta) e^{-t r_{\delta}(\delta)} \norm{ \nabla
      f(\theta_0)} \left|  \frac{h_{\delta}(\delta)}{r_{\delta}(\delta)} - \lim_{t\to 0} \frac{h(t)}{r(t) } \right| \delta^{-1}.
\]
and we can easily conclude that $N_2$ is uniformely bounded by usual calculus and assumption \ref{ass:Existencet0}.

The term $N_3$ is bounded in a similar way as $N_1$ using assumption \ref{ass:Existencet0}, inequality \eqref{eq:bound_Thetasmoothed} and Lemma \ref{uniform_bound_m_over_v}. More precisely:
 \begin{align*}
     N_3 &\leq  r_{\delta}(\delta) h_{\delta}(\delta) L \left(K_1 \norm{\frac{m_0}{\sqrt{v_0 + \eps}}}^2+ K_2\right) \int_0^t u \,du\\
     &\leq \frac{\delta^2 r_{\delta}(\delta) h_{\delta}(\delta)}{2} L \left(K_1 \norm{\frac{m_0}{\sqrt{v_0 + \eps}}}^2+ K_2\right).
 \end{align*}

Gathering all bounds, we easily conclude that there exists a constant $C_1$ such that, for every $\delta\leq 1$:
  \begin{align*}
     & \sup_{0 < u \leq t} \norm{h_{\delta}(u)\nabla
    f(\theta_{\delta}(u)) - r_{\delta}(u)m_{\delta}(u) }\leq C_1.
 \end{align*}

 From a similar argument, we obtain that there exists a constant $C_2$ such that, for every $\delta\leq 1$:
 \[
 \sup_{0 < u \leq t} \norm{p_{\delta}(u)\left[ \nabla f(\theta_{
         \delta}(u)) \right]^2  - q_{\delta}(u)v_{\delta}(u)  } \leq C_2.
         \]
We conclude that there exists a constant $C$ such that $N(t,\delta) <C$ for every $t\leq \delta$ and $\delta\leq 1$.

\paragraph{The case $t > \delta$} The proof uses the same arguments as in the case of $t\leq \delta$  using the appropriate integral formulation and Lemma \ref{uniform_bound_m_over_v}. We omit the details here.

\section{Convergence of the Euler discretization}\label{sec:Discretization}

\subsection{A priori estimates for the discrete model}

Just as in the previous section, we need to obtain estimates and bounds for the discrete system in order to compare it with its continuous counterparts. In other words, we need to re-do subsection \ref{ssec:a_priori}, but for the discrete system \eqref{eq:discrete}.

 \begin{proposition}\label{bound_on_v}
For all $k= 0,\cdots, K-1$
 \[
m_{k+1}=\sum_{i = 0}^{k}  s h_{i+1} \nabla f(\theta_{i+1})  \prod_{j =
         i+1}^{k} (1 - sr_{j+1}) +  \prod_{j =
         0}^{k} (1 - sr_{j+1}) m_0
\]
and
 \[
v_{k+1}=\sum_{i = 0}^{k}  s p_{i+1} [\nabla f(\theta_{i+1})]^2  \prod_{j =
         i+1}^{k} (1 - sq_{j+1}) +  \prod_{j =
         0}^{k} (1 - sq_{j+1}) v_0,
\]
  where we used the notation $h_i = h(t_i)$ (and similarly for the other functions).
Let assume that the learning rate $s$ satisfies $ sr_1 < 1$ and $sq_1<1$.
Hence, the numerical scheme preserves the strict positivity of $v$ i.e if we
     assume $v_{0} > 0$, then for all $k= 0,\cdots, K$,  $v_k > 0$.
 \end{proposition}
 \begin{proof}
     The proof follows directly by induction and the iterative formula for $m$ and $v$ given in \eqref{eq:discrete}. Indeed, if $k=0$, then the above formula recovers \eqref{eq:discrete}. So, suppose by induction that this formula is true for $k'<k$. We have that:
     \[
     \begin{aligned}
     m_{k+1} &= (1-sr_{k+1})m_k + sh_{k+1}\nabla f(\theta_{k+1})\\
     & = \left[\sum_{i = 0}^{k}  s h_{i+1} \nabla f(\theta_{i+1})  \prod_{j =
         i+1}^{k} (1 - sr_{j+1}) +  \prod_{j =
         0}^{k-1} (1 - sr_{j+1}) m_0 \right] + sh_{k+1}\nabla f(\theta_{k+1})\\
         &=\sum_{i = 0}^{k}  s h_{i+1} \nabla f(\theta_{i+1})  \prod_{j =
         i+1}^{k} (1 - sr_{j+1}) +  \prod_{j =
         0}^{k} (1 - sr_{j+1}) m_0
     \end{aligned}
     \]
     and the same also holds for $v_{k+1}$, proving the assertion.
 \end{proof}

\begin{proposition}\label{bound_on_m_over_v}
Let assume that the learning rate $s$ satisfies $ sr_1 < 1$, $sq_1<1$ and that $p(t) \not\equiv 0$.  For all $k = 0, \cdots, K-1$, the following bound holds
\begin{align*}
\norm{\frac{m_{k+1}}{\sqrt{v_{k+1}}}}^2
    &\leq  \norm{\frac{m_{0}}{\sqrt{v_{0}}}}^2 \prod_{i = 0}^{k }
    \max\left( \frac{(1 - s r_{i+1})^2  }{ (1 - s q_{i+1}) }, 1 \right) \\& \qquad + sd \sum_{i = 0}^{k}\frac{h_{i+1}^2}{
          p_{i+1} }
      \prod^{k}_{j=i}
    \max\left( \frac{(1 - s r_{j+1})^2  }{ (1 - s q_{j+1}) }, 1 \right).
\end{align*}
  Moreover
   \[
 \prod_{i = 0}^{K-1 }
    \frac{(1 - s r_{i+1})^2
      }{(1 - sq_{i+1} ) } \leq \exp\left( -2\int_{t_1}^{T} r(t) dt +  \frac{1}{1-sq_1} \int_{t_0}^{T} q(t) dt\right).
  \]
\end{proposition}

\begin{proof}
From the discrete updates for $m_{k+1}$ and $v_{k+1}$ given by equation \eqref{eq:discrete}, we easily
 observe that the following identity holds true
 \begin{align*}
     &m^2_{j, k+1}v_{j,k} - m^2_{j, k}v_{j,k+1}\\
     &=  \left( (1 - s r_{k+1})^2 m^{2}_{j,k} + s^2
      h_{k+1}^2 [\partial_j f_{k+1}]^2   +   2s h_{k+1} (1 - s r_{k+1}) m_{j,k}  \partial_j f_{k+1}
     \right) v_{j,k} \\& -  m^2_{j, k} (1 - s q_{k+1} )v_{j,k}  - s
      p_{k+1} m^2_{j, k}[\partial_j f_{k+1} ]^2 \\
     & =  v_{j,k}\left( (1 - s r_{k+1})^2 -  (1 - s q_{k+1}
      ) \right)\left( m^{2}_{j,k}  + s \frac{h^2_{k+1}}{p_{k+1}}v_{j,k}\right)
     \\& +  s \frac{h_{k+1}^2}{ p_{k+1} }v_{j,k+1} v_{j,k}  - s p_{k+1} \left( m_{j, k}\partial_j f_{k+1} -   \frac{h_{k+1}}{p_{k+1}}
     (1 - s r_{k+1}) v_{j,k}  \right) ^2.
 \end{align*}
 Thus, dividing both side of  the previous equality by $v_{j,k+1}v_{j,k}$, we obtain for all $k \geq 0$
 \begin{align*}
     \norm{\frac{m_{k+1}}{\sqrt{v_{k+1}}}}^2& -
     \norm{\frac{m_{k}}{\sqrt{v_{k}}}}^2
      =
     \sum_{j=1}^d \frac{m^2_{j, k+1}}{v_{j,k+1} } - \frac{m^2_{j, k}}{v_{j,k}}\\
     & \leq
     \sum_{j=1}^d \frac{(1 - s r_{k+1})^2 -  (1 - s q_{k+1}
      )}{ v_{j,k+1} }\left( m^{2}_{j,k}   + sv_{j,k}  \frac{h^2_{k+1}}{p_{k+1}} \right)
      +  sd \frac{h_{k+1}^2}{ p_{k+1} }.
 \end{align*}

We consider two different cases: if $ (1 - s r_{k+1})^2 -  (1 - s q_{k+1})\leq 0 $ then,
 \[
\norm{\frac{m_{k+1}}{\sqrt{v_{k+1}}}}^2
\leq
     \norm{\frac{m_{k}}{\sqrt{v_{k}}}}^2
      +
    sd \frac{h_{k+1}^2}{ p_{k+1} }.
    \]
On the other hand, if $ (1 - s r_{k+1})^2 -  (1 - s q_{k+1}) \geq  0$, then from the update rule for $v_{k+1}$, given by equation \eqref{eq:discrete}, and the fact that
 $v_{k+1 } \geq (1 - sq_{k+1})v_k $, we get
 \begin{align*}
     \norm{\frac{m_{k+1}}{\sqrt{v_{k+1}}}}^2& -
     \norm{\frac{m_{k}}{\sqrt{v_{k}}}}^2 \\
     & \leq
     \sum_{j=1}^d \frac{(1 - s r_{k+1})^2 -  (1 - s q_{k+1}
      )}{(1 - sq_{k+1} ) v_{j,k} }\left( m^{2}_{j,k}   + sv_{j,k}  \frac{h^2_{k+1}}{p_{k+1}} \right)
      +  sd \frac{h_{k+1}^2}{ p_{k+1} }\\
     & =
     \frac{(1 - s r_{k+1})^2 -  (1 - s q_{k+1}
      )}{(1 - sq_{k+1} )  } \left(  \norm{ \frac{m_{k}}{\sqrt{v_k}}}^2   + s d  \frac{h^2_{k+1}}{p_{k+1}} \right)
      +  sd \frac{h_{k+1}^2}{ p_{k+1} }.
 \end{align*}

Thus
 \[
     \norm{\frac{m_{k+1}}{\sqrt{v_{k+1}}}}^2 \leq
 \frac{(1 - s r_{k+1})^2  }{ (1 - s q_{k+1}) }  \left( \norm{\frac{m_{k}}{\sqrt{v_{k}}}}^2  + s d \frac{h^2_{k+1}}{p_{k+1}} \right)
 \]
 Combining the upper bounds obtained in the two cases
 \[
     \norm{\frac{m_{k+1}}{\sqrt{v_{k+1}}}}^2 \leq
 \max\left( \frac{(1 - s r_{k+1})^2  }{ (1 - s q_{k+1}) }, 1 \right) \left( \norm{\frac{m_{k}}{\sqrt{v_{k}}}}^2  + s d \frac{h^2_{k+1}}{p_{k+1}} \right)
 \]

By induction we get that
\begin{align*}
\norm{\frac{m_{k+1}}{\sqrt{v_{k+1}}}}^2
    &\leq  \norm{\frac{m_{0}}{\sqrt{v_{0}}}}^2 \prod_{i = 0}^{k }
    \max\left( \frac{(1 - s r_{i+1})^2  }{ (1 - s q_{i+1}) }, 1 \right) \\& \qquad + sd \sum_{i = 0}^{k}\frac{h_{i+1}^2}{
          p_{i+1} }
      \prod^{k}_{j=i}
    \max\left( \frac{(1 - s r_{j+1})^2  }{ (1 - s q_{j+1}) }, 1 \right).
\end{align*}
We now prove that the previous upper bound is bounded by a constant depending only on the final time $T$.
For all $x \in (0,1)$, we know that $-x \geq \log(1-x) \geq -x/(1 -x)$. Therefore
 \begin{align*}
  \log\left(    \prod_{i = 0}^{k }
    \frac{(1 - s r_{i+1})^2
      }{(1 - sq_{i+1} )  } \right) &= \sum_{i = 0}^{k } 2  \log (1 - s r_{i+1}) - \log (1 - s q_{i+1}) \\
      &\leq  \sum_{i = 0}^{k} -2s r_{i+1} + \frac{sq_{i+1}}{1-sq_1}.
 \end{align*}
From assumption \ref{ass:Existencet0}, the functions $q$ and $r$ are non-increasing. Then for all $t \in [t_i, t_{i+1}]$, we have $q_i \geq  q(t) \geq   q_{i+1}$ and similarly for $r$.
Integrating over $[t_i, t_{i+1}]$ and summing from zero to $K-1$ gives
 \[
 \int_{t_0}^{T} q(t) dt \geq  \sum_{i = 0}^{K-1} s q_{i+1}.
 \]
 Similarly integrating over $[t_i, t_{i+1}]$ and summing from one to $K-1$ gives
  \[
   \int_{t_1}^{T} r(t) dt \leq  \sum_{i = 1}^{K-1} s r_{i} \leq  \sum_{i = 0}^{K-1} s r_{i+1}.
 \]
 We conclude that
 \[
 \prod_{i = 0}^{K-1 }
    \frac{(1 - s r_{i+1})^2
      }{(1 - sq_{i+1} ) } \leq \exp\left( -2\int_{t_1}^{T} r(t) dt +  \frac{1}{1-sq_1} \int_{t_0}^{T} q(t) dt\right).
  \]
\end{proof}

Finally, from the previous estimates, we obtain a moment bound on $\theta$. This estimate is important
since the bound only depends on the final time $T$ and the norm of the initial solution but not on the norm of $\theta$ itself.
\begin{proposition}\label{bound_theta}
 For all $k= 0,\cdots, K-1$
    \[
    \norm{\theta_{k+1}} \leq \norm{\theta_{0}} + s\sum_{i= 0}^k \norm{\frac{m_i}{\sqrt{v_i + \eps}}}.
    \]

Then, from Proposition \ref{bound_on_m_over_v}, if $p(t)\not\equiv 0$ we conclude that  there exists a constant $C(T ) > 0$, such that for all $s$
    \[
       \sup_{0 \leq k \leq K_{T, s}} \norm{\theta_k}  \leq
        C(T)(1 +\norm{\x_{0}}) .
    \]
\end{proposition}

Using the previous estimates given by Propositions \ref{bound_on_v}, \ref{bound_on_m_over_v} and \ref{bound_theta} and the fact that $f$ is $C^2$,
we obtain the following bounds for the solution of the
numerical scheme \eqref{eq:discrete}.

\begin{proposition}[Bounds for the solution of the numerical scheme]\label{discrete_bounds}
  
    Let $A_0 \subset \mathbb{R}^d_{\geq 0}\times \mathbb{R}^d_{\geq0}\times \mathbb{R}^d_{>0}$ be a compact set. There exists a constant $C(T,A_0) > 0$, such that for all $s$ and all initial condition $\x_0\in A_0$
    \[
        \sup_{0 \leq k \leq K_{T, s}} \norm{\x_k}   \leq
        C(T,A_0) \left(1 + \norm{\x_{0}} \right) .
    \]
\end{proposition}

We are ready to prove the main Theorem of this section.

\subsection{Proof of Theorem \ref{thm:ConvergenceDiscretizationRate}}

We denote by $B_{\ell} = \left\{x \in \R^{3d}; \norm{x} \leq \ell \right\}$.
    From section \ref{ssec:a_priori} and Proposition \eqref{discrete_bounds},
    we know that there exists a constant $\ell$ such that $\x_k$ and $\tilde{\x}_k := \x(t_k)$ remain in $B_{\ell}$ for all $t \in [0,T]$ and $k = 0,\cdots, K_{T,s}$.
    Moreover we considered the numerical approximation only for $t_0>0$ and
    therefore all functions $p,q,r,h$ are continuously
    differentiable. The global error of the dynamics is now given by:
   \begin{equation}\label{eq:global_error}
   \begin{aligned}
   \max_{0 \leq k \leq K}\|\x_k - \tilde{\x}_k\| =
   \max_{0 \leq k \leq K} \left(E^\theta_k + E^m_k
   +E^v_k\right)
   \end{aligned}
 \end{equation}
 where $E^m_k =  \|m_k - \tilde{m}_k\|, E^\theta_k =  \|\theta_k - \tilde{\theta}_k\|$ and $E^v_k =  \|v_k - \tilde{v}_k\|$
 represent the global discretization error for each variable of the system.
   We control each term in the global error \eqref{eq:global_error} using the integral formulation
   and a priori estimates on both the discrete and continuous dynamics \eqref{eq:ODE} and \eqref{eq:discrete}.
We integrate Equation for $m$ in \eqref{eq:ODE} on $[t_{k-1}, t_k]$ and by subtracting the equation for $m_{k}$ in \eqref{eq:discrete}, we obtain
\[
\tilde{m}_{k} - m_{k} = \tilde{m}_{k-1} - m_{k-1} + \int_{t_{k-1}}^{t_{k}} h(u)\nabla f(\theta(u)) - h(t_{k})\nabla f(\theta_{k})
- r(u)m(u) + r(t_{k})m_{k-1}du.
\]
From the decompositions
\[
\begin{aligned}
h(u)\nabla f(\theta(u)) &- h(t_{k})\nabla f(\theta_{k}) \\&=
\left(h(u) - h(t_{k}) \right)\nabla f(\theta(u)) + h(t_{k}) \left(\nabla f(\theta(u)) - \nabla f(\tilde{\theta}_{k}) \right) +  h(t_{k}) \left(\nabla f(\tilde{\theta}_k) - \nabla f(\theta_{k}) \right)
\end{aligned}
\]
and
\[
\begin{aligned}
r(u)m(u) &- r(t_{k})m_{k-1} \\ &= \left(r(u) - r(t_{k}) \right) m(u) + r(t_{k}) \left(m(u) - \tilde{m}_{k-1} \right)
+ r(t_{k}) \left(\tilde{m}_{k-1} - m_{k-1} \right),
\end{aligned}
\]
we get by definition of $E^m_{k}$ and from the triangle inequality
\[
\begin{aligned}
E^m_{k} \leq  E^m_{k-1} &+ \norm{\int_{t_{k-1}}^{t_{k}} \left(h(u) - h(t_{k}) \right)\nabla f(\theta(u)) + h(t_{k}) \left(\nabla f(\theta(u)) - \nabla f(\tilde{\theta}_{k}) \right) du } \\ &
+ \norm{\int_{t_{k-1}}^{t_{k}} \left(r(u) - r(t_{k}) \right) m(u) + r(t_{k}) \left(m(u) - \tilde{m}_{k-1} \right)du} \\&
+ s \abs{h(t_{k})}\norm{\nabla f(\tilde{\theta}_k) - \nabla f(\theta_{k})} + s \abs{r(t_{k})} \norm{\tilde{m}_{k-1} - m_{k-1} }.
\end{aligned}
\]
From the $C^1$ continuity of $h, r, \theta, m$,
the local Lipschitz assumption on $\nabla f$, Lemmas \ref{norm_theta} and \ref{bound_theta} (stating that both $\tilde{\theta}_k$ and $\theta_k$ are bounded for finite $T$), we deduce that there exist two constants $M_1$ and $M_2$ independent of $s$, but depending on $T$ and the initial condition, such that
\begin{equation}\label{global_error_m}
\begin{aligned}
  E^m_{k} \leq  \left(1+ s \abs{r(t_{k})}\right)E^m_{k-1} + M_1 s^2 +
 s M_2 \abs{h(t_{k})}E^\theta_{k}
\end{aligned}
\end{equation}
Similar reasoning holds for $v$ and
there exist two constants $N_1$ and $N_2$ independent of $s$, but depending on $T$ and the initial condition, such that
\begin{equation}\label{gloabl_error_v}
\begin{aligned}
  E^v_{k} \leq  \left(1+ s \abs{q(t_{k})}\right)E^v_{k-1} + N_1 s^2 +
 s N_2 \abs{p(t_{k})}E^\theta_{k}
\end{aligned}
\end{equation}
It remains to consider the discretization error for $\theta$. We have
\[
    \begin{aligned}
  \tilde{\theta}_k - \theta_k = \tilde{\theta}_{k-1} - \theta_{k-1}  - \int_{t_{k-1}}^{t_k} \frac{m(u)}{\sqrt{v(u) + \eps}} -  \frac{m_{k-1}}{\sqrt{v_{k-1} + \eps}}du.
    \end{aligned}
    \]
The discretization error will be expressed
in terms of $ m(u)/ \sqrt{
      v(u) + \eps} -   m_{k-1} / \sqrt{v_{k-1}+\eps}$ which we decomposed as follow
\begin{align*}
&\frac{ m(u)}{\sqrt{
      v(u) + \eps}} -   \frac{m_{k-1}}{\sqrt{v_{k-1}+\eps}} \\
   &= \frac{ m(u) - \tilde{m}_{k-1}}{\sqrt{
         v(u) + \eps}} +  \frac{ \tilde{m}_{k-1} - m_{k-1}}{\sqrt{ v(u) + \eps}} +
  \frac{m_{k-1}}{\sqrt{v_{k-1}+ \eps}} \frac{ v_{k-1} - \tilde{v}_{k-1} + \tilde{v}_{k-1}  -v(u) }{ \sqrt{v(u)  +
  \eps}\left( \sqrt{v_{k-1}+ \eps} + \sqrt{v(u) + \eps} \right) }.
\end{align*}
From Propositions \ref{lemma:bound_m_over_sqrtv} and  \ref{bound_on_m_over_v},
we deduce that there exists three constants $C_1, C_2, C_3$, depending on $\eps, T$ and
the initial data but not on $s$, such that
\[
\begin{aligned}
E^\theta_k \leq E^\theta_{k-1}  + C_1s^2 + C_2sE^m_{k-1} + C_3sE^v_{k-1}
\end{aligned}
\]
Combining all estimates, there exist two constants $\tilde{C}_1$ and $\tilde{C}_2$ such that
\[
\begin{aligned}
\norm{\x_k - \tilde{\x}_k} \leq  \left(1 + s \tilde{C}_1 \right)\norm{\x_{k-1} - \tilde{\x}_{k-1}} + \tilde{C}_2s^2.
\end{aligned}
\]
Applying this formula recursively,  we obtain for all $k= 0, \cdots, K$
\[
\begin{aligned}
\norm{\x_k - \tilde{\x}_k} &\leq  \left(1 + s \tilde{C}_1 \right)^k\norm{\x_{0} - \tilde{\x}_{0}} +  \tilde{C}_2 s^2 \sum_{i=0}^{k-1}\left(1 + s \tilde{C}_1 \right)^i.
\end{aligned}
\]
Since $\x_{0} = \tilde{\x}_{0}$ and noticing that the remaining term is a geometric series
\[
\begin{aligned}
\norm{\x_k - \tilde{\x}_k} &\leq
 s \frac{\tilde{C}_2}{ \tilde{C}_1} \left(\left(1 + s \tilde{C}_1 \right)^k - 1 \right)
\leq s \frac{\tilde{C}_2}{\tilde{C}_1} \left( \exp\left(T \tilde{C}_1 \right) - 1 \right)
\end{aligned}
\]

\section{Proof of Theorem \ref{thm:convergence}}\label{sec:ProofTopologicalConvergence}

The proof follows from a topological study of the vector-field associated to the ODE \eqref{eq:ODE} together with the control given by the Lyapunov energy \eqref{eq:EnergyPrelim} studied below. We first introduce the basic notions from qualitative theory of ODE's in $\S\S$ \ref{ssec:PoincareBendixson} necessary for this work, before turning to the main proof in $\S\S$ \ref{ssec:ProofConvergence} (specialists may go directly to the proof). The $\S\S$ \ref{ssec:GradientFlowExampleConvergence} concerns the simple example of the gradient flow.

\subsection{Background: qualitative theory of ODE's}\label{ssec:PoincareBendixson}
We start introducing some of the concepts of qualitative theory of ODE's used in this work. Consider the Cauchy problem:
\[
\dot{\x} = F(\x,t), \quad \x(t_0) = x_0
\]
where $\x = (x_1, \ldots, x_n) \in \mathbb{R}^n$ and $F \in C^1(\mathbb{R}^{n+1})$. Note that this system is not autonomous, that is, the expression of the differential equation depends on the time $t$. We can always transform it into an autonomous system by adding the differential equation $\dot{t} = 1$, where all derivatives are now taken in respect to a variable $\tau$ (the new time), that is:
\[
\left\{\begin{aligned}
\dot{\x} &= F(\x,t),\\
\dot{t} &= 1
\end{aligned}\right.
\quad\left\{\begin{aligned}
 \x(\tau_0) &= x_0, \\
 t(\tau_0)  &= t_0
\end{aligned}\right.
\]
Now, we take $\y = (\x,t) \in \mathbb{R}^{n+1}$, $G(\y) = (F(\y),1)$ and $\y(\tau_0)=\y_0=(\x_0,t_0)$ so that we obtain an autonomous system:
\[
\dot{\y} = G(\y), \quad \y(\tau_0) = \y_0.
\]
An autonomous system can be described by a vector-field. In this case, we have:
\[
\partial := \sum_{i=1}^{n+1} G_i(\y) \partial_{y_i}
\]
where $(\partial_{y_1}, \ldots,\partial_{y_{n+1}})$ are global sections which generate the tangent space $T\mathbb{R}^{n+1}$ and $G(\y) = (G_1(\y),\ldots,G_{n+1}(\y))$. Note that a differentiable curve $\y:[t_0,t_{\infty}[ \to \mathbb{R}^{n+1}$ is a solution of the differential equation if, and only if, the vector $\dot{y}(\tau) $ is equal to $\partial(\y(\tau))$ at every point $\tau\in [t_0,t_{\infty}[$.

Now, let us assume that a solution $\y(\tau)$ of the differential equation is defined for every $\tau\in [t_0,\infty[$ (for the case of ODE \eqref{eq:ODE}, this is proved in Theorem \ref{th:existence}). In order to study the asymptotic behaviour of $\y(\tau)$ when $\tau \to \infty$, we consider its topological limit, that is:
\[
\omega(\y(\tau)) := \bigcap_{\alpha>t_0} \overline{\y([\alpha,\infty))}
\]
In particular, note that if $\y(\tau) \to p$ when $\tau\to \infty$, then $\overline{\y([\tau,\infty))} = \y([\alpha,\infty)) \cup \{p\}$ for every $\alpha>\tau_0$, so that $\omega(\y(\tau)) = \{p\}$. Classical Poincar\'{e}-Bendixson theory provides a topological description of the set $\omega(\y(\tau))$. There are three crucial Properties used in this work:
\begin{properties}\label{properties:PB}\hfill
\begin{itemize}
\item[(a)] $\omega(\y(\tau))$ is either empty, or an union of orbits of $\partial$;
\item[(b)] The $\omega$-limit of the orbits in $\omega(\y(\tau))$ must be contained in $\omega(\y(\tau))$, that is, $\omega(\omega(\y(\tau)))\subset\omega(\y(\tau))$;
\item[(c)] If $\y(\tau)$ is bounded (that is, its image is bounded) then $\omega(\y(\tau))$ is non-empty compact and connected set.
\end{itemize}
\end{properties}

In what follows we use these three properties in order to study the asymptotic behaviour of the orbits of ODE \eqref{eq:ODE}.

\subsection{Simplified example: Study of the gradient flow}\label{ssec:GradientFlowExampleConvergence}
Consider the differential equation:
\[
\dot{\theta}(t) = -\nabla  f(\theta(t)) \quad \text{ with initial condition } \quad \theta(0) = \theta_0 \in \mathbb{R}^d,
\]
which is an autonomous differential equation. In this case, the analogue of the Lyapunov energy \eqref{eq:EnergyPrelim} is given by:
\[
E(\theta) = f(\theta), \quad \implies \quad \frac{d}{dt} E(\theta) = - \|\nabla  f(\theta)\|^2\leq 0
\]
We are now ready to turn to the Poincar\'{e}-Bendixson approach of this problem. Since the system is autonomous, we can consider the vector-field associated to the gradient system, which is given by:
\[
\partial = - \sum_{i=1}^d \partial_{\theta_i}[f(\theta)] \partial_{\theta_i}.
\]
We now study the asymptotic behaviour of the the trajectory via its topological limit $\omega(\theta(t))$, under the assumption that $\theta(t)$ is bounded c.f assumption \ref{ass:bounded}. It follows from property \ref{properties:PB}(c) that $\omega(\theta(t))$ is non-empty, and from property \ref{properties:PB}(a) that it is the union of orbits. Now, since the function $E(\theta(t))$ is continuous, monotone and bounded (because $\theta(t)$ is bounded), we get that the limit
\[
\lim_{t\to \infty} E(\theta(t)) =: E_{\infty}
\]
exists. It follows from property \ref{properties:PB}(b), furthermore, that $\omega(\theta(t)) \subset (E(\theta) = E_{\infty})$, and this can only happen if the derivative of $E$ along any orbit in $\omega(\theta(t))$ are zero. By the derivative expression of the derivative of $E$, we get that $\omega(\theta(t)) \subset  (\nabla f(\theta) =0)$. We easily conclude that $f(\theta(t))$ converges to a critical value of $f$.

\subsection{Proof of Theorem \ref{thm:convergence}}\label{ssec:ProofConvergence}
Consider the \emph{autonomous} system associated to \eqref{eq:ODE} (we recall that this means that we treat the time $t$ as a variable with differential equation $\dot{t}=1$, and that the system now has time $\tau$), that is, the following vector field defined in $\mathbb{R}^{3d+1}$:
\[
\begin{aligned}
\partial = \partial_t - \sum_{i=1}^d&\frac{m_i}{\sqrt{v_i+\eps}} \partial_{\theta_i} + (h(t)\partial_{\theta_i} f(\theta) - r(t) m_i)\partial_{m_i} \\
&+ (p(t)\left[\partial_{\theta_i} f(\theta)\right]^2 - q(t) v_i)\partial_{v_i}.
\end{aligned}
\]
which is well-defined for every $(\theta,m,v,t) \in \mathbb{R}^d\times \mathbb{R}^d\times \mathbb{R}^{d}_{\geq 0}\times \mathbb{R}_{>0}$ (because we assume that $\eps>0$, c.f. assumption \ref{ass:TopologicalConvergence}). In order to study the convergence of the vector field when $t\to \infty$, we perform the change of coordinates $s = 1/t$ (in this case, studying the behaviour at $t\to \infty$ is replaced to studying the behaviour  when $s\to 0$), which yields:
\begin{equation}\label{eq:VectorFieldInfinity}
\begin{aligned}
\partial = -s^2\partial_s + \sum_{i=1}^d&-\frac{m_i}{\sqrt{v_i+\eps}} \partial_{\theta_i} + (H(s)\partial_{\theta_i} f(\theta) - R(s) m_i)\partial_{m_i} \\
&+ (P(s)\left[\partial_{\theta_i} f(\theta)\right]^2 - Q(s) v_i)\partial_{v_i}.
\end{aligned}
\end{equation}
and note that the above vector-field is kept autonomous. We recall that the time of the associated differential equation is now denoted by $\tau$ (that is, a solution of this vector field is a curve $\y(\tau)= (\theta(\tau),m(\tau),v(\tau),s(\tau))$ such that $\dot{\y}(\tau) = \partial(\y(\tau))$). We now fix an orbit
\[
\y(\tau)= (\theta(\tau),m(\tau),v(\tau),s(\tau))
\]
with initial conditions $\y(\tau_0) = (\theta(\tau_0),m(\tau_0),v(\tau_0),1/\tau_0)$. By the Lemma \ref{lem:IntervalDefinition}, we know that $\y(\tau)$ is bounded and $v(\tau) >0$ for all $\tau \in [\tau_0,\infty)$. Denote by $\omega(\y(\tau))$ the topological limit of $\y(\tau)$. By assumption \ref{ass:bounded} we know that $\theta(\tau)$ is bounded, and by Lemma \ref{lem:IntervalDefinition} we conclude that $m(\tau)$ and $v(\tau)$ are also bounded. It follows from property \ref{properties:PB}(c) that $\omega(\y(\tau))$ is non-empty, and from property \ref{properties:PB}(a) that it is the union of orbits of $\partial$. Furthermore, from the expression $s = 1/t$ (and the fact that the solutions are defined for all $\tau\in [\tau_0,\infty)$, which implies that $t$ takes all values in $[\tau_0,\infty)$) we know that $\omega(\y(\tau)) \subset (s=0)$.

We now consider the energy functional $E$ given in \eqref{eq:EnergyPrelim} but in this new coordinate system. More precisely, consider the functional:
\[
\widetilde{E}(\y) = f(\theta) +\frac{1}{2 H(s)}\norm{\frac{m}{\left[v+\eps\right]^{1/4}}}^2,
\]
which by assumption \ref{ass:TopologicalConvergence} is everywhere well-defined (because $H(0)>0$). It follows from direct computation that:
\[
\frac{d}{d \tau}\widetilde{E}(\y) \leq  - \frac{1}{2H(s)}\left[ 2R(s)-\frac{Q(s)}{2} - s^2\frac{H'(s)}{H(s)}  \right]\norm{\frac{m}{\left[v+\eps\right]^{1/4}}}^2.
\]
which is everywhere non-positive by assumption \ref{ass:NecessaryCoeficients} and \ref{ass:TopologicalConvergence}. Now, since $E(\x(\tau))$ is bounded from below (because $E$ is continuous and $\y(\tau)$ is bounded), we conclude that the limit:
\[
\lim_{\tau\to \infty} \widetilde{E}(\y(\tau)) = \widetilde{E}_{\infty}
\]
exists. In particular, it follows from property \ref{properties:PB}(b) that $\omega(\y(\tau)) \subset (\widetilde{E}(\y) = \widetilde{E}_{\infty})$. This implies that $\omega(\y(\tau))$ must be contained in the set of zero derivative of $\widetilde{E}(\y)$. By assumption \ref{ass:TopologicalConvergence} this implies that $\omega(\y(\tau)) \subset (m=0)$. Now note that:
\[
\partial \cdot m_i = H(s)\partial_{\theta_i} f(\theta) - R(s) m_i,
\]
and since $H(0) \neq 0$ (assumption \ref{ass:TopologicalConvergence}), we conclude that $\omega(\y(\tau)) \subset (\nabla f(\theta)=0)$. Finally, if either $P(s) \equiv Q(s) \equiv 0$ or $Q(0) > 0$, by the expression:
\[
\partial \cdot v_i = P(s)\partial_{\theta_i} f(\theta)^2 - Q(s) v_i,
\]
we conclude that $\omega(\x(\tau)) \subset (v=0)$. We conclude easily.

\section{Proof of Theorem \ref{thm:avoiding}}\label{sec:ProofAvoiding}

The proof follows from central-stable manifold theory applied to the singular points of the vector-field \eqref{eq:VectorFieldInfinity}. We first recall the version of the central-stable manifold Theorem in $\S\S$ \ref{ssec:CentralStable} necessary for this work, before proving Theorem \ref{thm:avoiding} in $\S\S$ \ref{ssec:ProofAvoiding} (specialists may go directly to the proof). The $\S\S$ \ref{ssec:GradientFlowExampleAvoiding} concerns the simpler example of the gradient flow.

\subsection{Central-Stable manifold}\label{ssec:CentralStable}

In order to explain the notion of the central-stable manifold, let us consider the following-example:
\[
\left\{
\begin{aligned}
\dot{x}(t) &= x(t)\\
\dot{y}(t) &= y(t)^2
\end{aligned}\right.
\]
Note that the only singular point (that is, a point where the solution is constant in time) of this system is $(x,y)=(0,0)$. Now, consider the Taylor expansion of this system at the origin:
\[
\left(\begin{matrix}
\dot{x}\\
\dot{y}
\end{matrix}\right)
= \left[\begin{matrix}
1&0\\
0& 0
\end{matrix}\right] \cdot \left(\begin{matrix}
x\\
y
\end{matrix} \right)  + \left(\begin{matrix}
0\\
y^2
\end{matrix}\right)
\]
and note that the linearisation of the system at the origin has one eigenvalue equal to $1$ (in particular, positive) and another equal to $0$. As we will see, the existence of an eigenvalue which is positive guarantee's the existence of an \emph{unstable} manifold, which is crucial in order to prove the \emph{existence} of a proper sub-manifold $\Sigma$, the so-called central-stable manifold, which contains all orbits which converge to $(0,0)$.

More precisely, a solution of the ODE with $t_0=0$ and initial condition $(x_0,y_0) \neq (0,0)$ is given by:
\[
x(t) =  x_0 \cdot e^t, \quad y(t) = \frac{y_0}{1- ty_0}.
\]
In particular, note that if $y_0 \neq 0$, then any solution $(x(t),y(t))$ is \emph{divergent}, that is, there exists a neighbourhood $U$ of the singular point $(0,0)$ (e.g. $U= (x^2+y^2 <1)$) and a positive time $\tilde{t}>t_0=0$ (e.g. $\tilde{t}=\ln(1/|x_0|)$ if $y_0\leq 0$ and $\tilde{t}=(1-y_0)/y_0$ if $y_0>0$) such that $(x(\tilde{t}),y(\tilde{t})) \not\in U$. The heuristic reason for such a behaviour is exactly the existence of a positive eigenvalue for the linearization of the differential system, which means that there exists a ``direction" (in this case, $x$), or more precisely an unstable sub-manifold, whose dynamic is divergent and which dominates almost every other orbit.

There may exist a proper sub-manifold, nevertheless, where the unstable effect of the unstable manifold does not influence that dynamics. In this example it is given by $\Sigma = (y=0)$ and we call it the central-stable manifold. Note that every solution $(x(t),y(t))$ which converges to $(0,0)$ (or, more generally, that does not diverge) are contained in $\Sigma$ (but $\Sigma$ may contain solutions which diverge). The strong regularity of the central stable manifold in this example is an accident due to the expression of the differential equation; a more intriguing example is given by the Euler equation $(x-y)\partial_x + y^2\partial_y$, but we won't enter this discussion in here.

We are ready to present the main result of this section, which formalizes the above considerations for general differential equations. The following result is a local version of \cite[Ch. 1 Thm 4.2]{CMBook}, by using the cut-off technique given in \cite[Ch. 1 Lem. 3.1]{CMBook}; c.f. \cite[Ch. 1, Thm 1.1 and 3.2]{CMBook}.

\begin{theorem}\label{thm:LocalCentralStable}
Consider the differential equation
\[
\dot{x} = Ax + F(x)
\]
defined over $\mathbb{R}^n$ for some $n\in \mathbb{N}$, where $A$ is a matrix which contains at least one positive eigenvalue, and $F(x)$ is a $C^k$ function, for some $k\geq 1$, such that $F(0) =0$ and $DF(0)=0$. Then there exists a neighbourhood $U$ of $0$ and a $C^k$ sub-manifold $\Sigma$ (the center-stable manifold) such that:
\begin{enumerate}
\item The manifold $\Sigma$ is invariant by the differential equation everywhere over $U$;
\item The manifold contains the origin $0$ and has dimension at most $n-1$;
\item If $\x_0\in U\setminus \Sigma$, then there exists $\widetilde{t}_0>t_0$ such that $\x(\widetilde{t}_{0}) \notin U$, where $\x(t)$ denotes the solution of the differential equation with initial condition $\x(t_0)=\x_0$.
\end{enumerate}
\end{theorem}

Global versions of this result do exist, but they demand globally Lipschitz assumptions, and a control of the Lipschitz constant in terms of the linearisation $A$ of the system, see \cite[Ch. 1, Thm 1.1]{CMBook}. We finish this section with an example to illustrate why, in general, the result is only local.

\begin{example}\label{ex:pathology}
 Consider the differential equation:
\[
\begin{aligned}
\dot{x} &= y - (y^2-\sin(x)^2)\sin(x)\cos(x)\\
\dot{y} &= \sin(x)\cos(x) + (y^2-\sin(x)^2))y.
\end{aligned}
\]
and let us consider the orbits of the differential equation whose limit set contain the singularity $(\pi,0)$. Note the following contrast:

\begin{itemize}
\item Local: consider a (very) small neighborhood $U$ of $(\pi,0)$, then the only solutions which contain $(\pi,0)$ in their limit set have initial conditions in:
\[
(x(t_0),y(t_0)) \in U \cap \{y^2= \sin(x)^2\}
\]
and all other solutions ``leave" $U$ in finite time.
\item Global: Every solution $(x(t),y(t))$ with initial condition
\[
(x(t_0),y(t_0)) \in \{ -\pi<x<\pi,\, y^2< \sin(x)^2\} \setminus \{  (0,0)\}
\]
converges to the set $\{ -\pi\leq x \leq \pi, \,y^2 = \sin(x)^2\}$, which contains the singular point $(\pi,0)$.
\end{itemize}
It follows that there are global convergence behaviours which are not controlled by the locally defined central-stable manifold.
\end{example}

\subsection{Simplified example: Study of the gradient flow}\label{ssec:GradientFlowExampleAvoiding}
Consider the vector-field associated to the gradient descent (see $\S\S$ \ref{ssec:GradientFlowExampleConvergence}):
\[
\partial = - \sum_{i=1}^d \partial_{\theta_i}[f(\theta)] \partial_{\theta_i}
\]
and note that the singular points are the critical values of $f$, that is, \[
Sing(\partial) = \{\theta; \, \nabla f(\theta)=0\}.\]
We also know from $\S\S$ \ref{ssec:GradientFlowExampleConvergence} that if $\theta(t)$ is a bounded solution, then $f(\theta(t)) \to f_{\ast}$ converges to a critical value of $f$ because $\omega(\theta(t)) \subset (\nabla f(\theta) = 0) = Sing(\partial)$. Now, let us consider:
\[
B:= \{\theta;\, \theta \in Sing(\partial) \text{ and } \theta \text{ is not a local minimum of }f \}
\]
By the assumption \ref{ass:Morse}(b), the set $B$ is discrete and, therefore, a countable union of isolated points. So, let us consider the set:
\[
S_{0} := \{\theta_0 \in \mathbb{R}^d; \, \theta(t_0)=\theta_0, \text{ and } \omega(\theta(t)) \cap B \neq \emptyset \}
\]
It follows from Property \ref{properties:PB}(3) that $\omega(\theta(t))$ is a connected set, so that:
\[
S_{0} = \{\theta_0 \in \mathbb{R}^d; \, \theta(t_0)=\theta_0, \text{ and } \omega(\theta(t)) \subset B \}
\]
We now inquire about the size (in terms of measure theory) of $S_{0}$. In order to study this, let us start with a local analysis for a fixed $\theta_{\ast} \in B$. Consider the linearisation of the vector-field $\partial$ at a singular point $\theta_{\ast} \in B$, which is given by
\[
A = - \mathcal{H}_f(\theta_{\ast})
\]
where $\mathcal{H}_f$ is the Hessian of $f$. By the assumption \ref{ass:Morse}(a), the linearisation $A$ has one strictly positive eigenvalue, so we may use the central-manifold theory. More precisely, by Theorem \ref{thm:LocalCentralStable}, there exists a neighbourhood $U_{\theta_\ast}$ of $\theta_{\ast}$, and a proper sub-manifold $\Sigma_{\theta_\ast} \subset U_{\theta_\ast}$ such that every orbit in $U_{\ast}$ which converges to $\theta_{\ast}$ must be contained in $\Sigma_{\theta_\ast}$. Note that the Lebesgue measure of $\Sigma_{\theta_\ast}$ is zero. Consider the set $\Sigma$ given by the union of all orbits with initial conditions in $\Sigma_{\theta_{\star}}$, for every $\theta_{\star}\in B$. Since $B$ is a countable set, we conclude that the Lebesgue measure of $\Sigma$ must be zero (since each $\Sigma_{\theta_\ast}$ has Lesbeague measure zero). Now, since $\theta_{\ast}$ is an isolated singularity of $\partial$ and the $\omega$-limit of an arbitrary orbit $\theta(t)$ with initial condition in $S_0$ is connected, we conclude that if $\omega(\theta(t)) = \theta_{\ast}$, then $\omega(t) \subset \Sigma_{\theta_{\star}}$ for $t>>t_0$. It easily follows that $S_0 \subset \Sigma$, and we conclude that $S_0$ has measure zero.

\subsection{Proof of Theorem \ref{thm:avoiding}}\label{ssec:ProofAvoiding}

Recall the vector field $\partial$ defined in \eqref{eq:VectorFieldInfinity}, which describes the ODE \eqref{eq:ODE}. We consider the set:
\[
B = \{\theta_{\star} \in \mathbb{R}^d;\, \nabla f(\theta_{\star}) = 0, \text{ and $\theta_{\star}$ is not a local minimum of }f \}
\]
By assumption \ref{ass:Morse}(b) the set $B$ is discrete and, therefore, a countable union of isolated points of $\mathbb{R}^d$. It follows from Theorem \ref{thm:convergence} that the set:
\[
C:=\{ \y=(\theta,m,v,s);\, \theta \in B\text{ and } \y \text{ is a singularity of }\partial\}
\]
is a countable union of isolated points, all of each have the form $\y_{\ast}=(\theta_{\star},0,0,0)$, where $\theta_{\star} \in B$. We now consider the set:
\[
S := \{\y_0=(\theta_0,m_0,v_0,s_0);\, \omega(\y(\tau))\cap C \neq \emptyset \text{, where }\y(\tau_0)=\y_0\}
\]
It follows from Property \ref{properties:PB}(3) that $\omega(\y(\tau))$ is a connected set, so that:
\[
S := \{\y_0=(\theta_0,m_0,v_0,s_0);\, \omega(\y(\tau)) \subset C \text{, where }\y(\tau_0)=\y_0\}
\]
We now make a local argument valid for each singular point in $C$ in order to show that $S$ is locally a sub-manifold; indeed, fix $\y_{\ast} \in S$. Consider the linearization of $\partial$ at the singular point $\y_{\ast}=(\theta_{\star},0,0,0)$, which is the $3d+1$ square matrix:
\[
Jac(\partial)(\y_{\ast})=
\left[
\begin{matrix}
0 & -\eps^{-1/2} Id & 0 & 0\\
H(0)\mathcal{H}_f(\theta_{\star}) & -R(0)Id & 0 & 0\\
0 &0 & -Q(0)Id & 0\\
0 &0 & 0 & 0
\end{matrix}
\right],
\]
where $Id$ denotes the Identity of a $d$-square matrix, and $\mathcal{H}_f(\theta_{\star})$ is the Hessian of $f$ at $\theta_{\star}$. It follows from direct computation that the eigenvalues $\lambda$ of this matrix are: $0$ with order $1$, $-Q(0)$ with order $d$ and the solutions of the quadratic equations:
\begin{equation}\label{eq:eigenvalues}
\eta_i = -\frac{\eps^{1/2}}{H(0)}(R(0) + \lambda)\lambda, \quad i=1,\ldots, d
\end{equation}
where $\{\eta_1,\ldots,\eta_d\}$ are the eigenvalues of $\mathcal{H}_f(\theta_{\star})$. By assumption \ref{ass:Morse}, we can suppose without loss of generality that $\eta_1<0$, and we easily conclude by equation \eqref{eq:eigenvalues} that there exists one strictly positive eigenvalue $\lambda$ of $Jac(\partial)(\y_{\ast})$. By Theorem \ref{thm:LocalCentralStable}, there exists an open neighbourhood $U_{\y_{\ast}}$ of $\y_{\ast}$ and a $C^1$ manifold $\Sigma_{\y_{\ast}} \subset U_{\y_{\ast}}$ such that every orbit $\y(\tau)$ with initial condition in $U_{\y_{\ast}}\setminus \Sigma_{\y_{\ast}}$, leaves $U_{\y_{\ast}}$ in finite time. Note that the Lebesgue measure of $\Sigma_{\y_\ast}$ is zero. Consider the set $\Sigma$ given by the union of all orbits with initial conditions in $\Sigma_{\y_{\star}}$, for every $\y_{\star}\in C$. Since $C$ is a countable set, we conclude that the Lebesgue measure of $\Sigma$ must be zero (since each $\Sigma_{\y_\ast}$ has Lesbeague measure zero). Now, since $\y_{\ast}$ is an isolated singularity of $\partial$ and the $\omega$-limit of an arbitrary orbit $\y(\tau)$ with initial condition in $S$ is connected, we conclude that if $\omega(\y(\tau)) = \y_{\ast}$, then $\y(\tau) \subset \Sigma_{\y_{\star}}$ for $\tau >>\tau_0$. It easily follows that $S \subset \Sigma$, and we conclude that $S$ has measure zero.

Finally, let $t_0>0$ be fixed and denote by $S_{t_0}= S \cap \{s=1/t_0\}$. Now, $S$ has volume zero and contains orbits of $\partial$, all of each are transverse to the set $\{s = 1/t_0\}$. It follows that the volume of $S_{t_0} \subset \mathbb{R}^{3d}$ is zero by transversality, and we conclude easily.

\section{Exploring a (generalized) Lyapunov}
\label{sec:ConvergenceRate}

In what follows we prove Theorem \ref{thm:convergenceRate} in $\S\S$ \ref{ssec:ProofConvergenceRate}, and we extend the study in $\S\S$ \ref{ssec:NesterovBounds}. In this section we do not introduce the techniques used, since they are standard in computer science.

\subsection{Proof of Theorem \ref{thm:convergenceRate}}\label{ssec:ProofConvergenceRate}

Let $\theta_{\star}$ be a minimum point of $f$ (which exists by the convexity Assumption \ref{ass:fConvex}). We recall that we consider the energy functional \eqref{eq:Lyapunov} given by:
\begin{align*}
    \c{E}(t,m,v, \theta) = \c{E}_1(t,\theta) + \c{E}_2(t,m,v, \theta)
\end{align*}
where
\begin{align*}
    \c{E}_1(t,\theta) &= \c{A}(t) \left( f(\theta) - f(\theta_{\star})\right)\\
    \c{E}_2(t,m,v,\theta) &= \frac{1}{2}  \norm{\left[v+\eps\right]^{1/4}\left(\theta -
    \theta_{\star} \right)}^2 - \c{B}(t) \ps{\theta-\theta_{\star}}{m}+\frac{\c{C}(t)}{2}\norm{\frac{m}{\left[v+\eps\right]^{1/4}}}^2.
\end{align*}
This functional is used as a Lyapunov function to prove convergence to a neighborhood of the global minimum.
We first compute its time derivative and we find
conditions on the functions $\c{B}$ and $\c{C}$, as well as the coefficients $h,p,q,r$, so that $\frac{d}{dt} \c{E}$ is bounded. The conditions must also guarantee that $\c{E}$ is positive.
From the convexity assumption on the objective function $f$, we get
\begin{equation}\label{eq:derivativeE1}
    \frac{d}{dt} \c{E}_1(t,\theta)  \leq  \c{A}'(t)  \ps{\nabla f(\theta)}{
        \theta -\theta_{\star}} - \c{A}(t) \ps{\nabla
    f(\theta)}{\frac{m}{\left[v+\eps\right]^{1/2}}}
\end{equation}
Next, we derive each term of $\mathcal{E}_2$.
\begin{align*}
    \frac{1}{2} \frac{d}{dt}& \norm{\left[v+\eps\right]^{1/4} \left(\theta
    -\theta_{\star}\right)}^2 \\ &=
     - \ps{m}{ \theta - \theta_{\star}}
      -  \frac{q(t)}{4} \ps{ \frac{v}{\left[v+\eps\right]^{1/2}} \left(\theta -
    \theta_{\star}\right)}{\theta - \theta_{\star}} \\&
    + \frac{p(t)}{4} \ps{\frac{\left[\nabla f(\theta)\right]^2}{\left[v+\eps\right]^{1/2}} \odot \left(\theta -\theta_{\star}\right)}{\theta - \theta_{\star}}
\end{align*}
\begin{align*}
\frac{d}{dt}& \c{B}(t) \ps{ \theta -\theta_{\star}}{  m }  \\
    &= -  \c{B}(t)  \norm{ \frac{m}{\left[v+\eps\right]^{1/4}} }^2
    + \c{B}(t)h(t)\ps{ \nabla
    f(\theta)}{ \theta -\theta_{\star}}\\&
    +  ( \c{B}'(t)  - \c{B}(t)r(t))\ps{ \theta - \theta_{\star}}{m}
    \end{align*}
    \begin{align*}
    \frac{d}{dt}& \frac{\c{C}(t)}{2} \norm{\frac{m}{\left[v+\eps\right]^{1/4}}}^2\\
    &= h(t) \c{C}(t)\ps{\nabla f(\theta) }{ \frac{m}{ \left[v+\eps\right]^{1/2}}}
    \\&     + \left( -r(t) \c{C}(t) +  \c{C}'(t)/2  \right)\norm{\frac{m}{\left[v+\eps\right]^{1/4}}}^2  \\& +  \frac{\c{C}(t)
    q(t) }{4}\norm{\frac{m \odot \left[v\right]^{1/2}}{\left[v+\eps\right]^{3/4}}}^2
     - \frac{\c{C}(t) p(t)}{4} \norm{ \frac{\nabla f(\theta) \odot
     m}{\left[v+\eps\right]^{3/4}}}^2
\end{align*}
By adding all of the above computations, we get that:
\[
\begin{aligned}
0&\leq \c{E}_1(t,\theta), \,\c{E}_2(t,m,v, \theta)  \\
 \frac{d}{dt} \c{E}(t,m,v, \theta)  &\leq \frac{p(t)}{4} \ps{\frac{\left[\nabla f(\theta)\right]^2}{\left[v+\eps\right]^{1/2}} \odot \left(\theta -\theta_{\star}\right)}{\theta - \theta_{\star}},
 \end{aligned}
\]
if all the following sufficient conditions are satisfied
\begin{align}
    & \c{A}(t) \geq 0, \quad \c{A}'(t) \geq 0,\\
    & \c{A}'(t)=  h(t)\c{B}(t)  \label{eqAprime} \\
    & \c{A}(t) =  h(t) \c{C}(t) \label{eqA} \\
    & \c{B}'(t) - \c{B}(t) r(t) = -1 \label{eqB} \\
    & \c{B}(t) \leq \frac{\c{C}(t)}{3}\left(2r(t) - \frac{q(t) }{2} +\frac{h'(t)}{h(t)} \right)  \label{in1} \\
    &  \c{B}^2(t) \leq \c{C}(t) \label{in2}.
\end{align}
It is now easy to see that equations \eqref{eqAprime}, \eqref{eqA} and \eqref{eqB} are equivalent to the choices of $\c{A}$, $\c{B}$ and $\c{C}$ chosen in subsection \ref{ssec:RateConvergence}, and that assumption \ref{ass:necessaryConditions} implies inequalities \eqref{in1} and \eqref{in2} are satisfied. It is now immediate from the Fundamental Theorem of calculus (and the fact that $\mathcal{E}_2 \geq 0$) that:
\[
f(\theta)  - f(\theta_{\star})  \leq \frac{\c{E}(t_0,    m_0, v_0,
    \theta_0)}{\c{A}(t)}
    +  \frac{  \int_{t_0}^t  p(u) \ps{\frac{\left[\nabla f(\theta)\right]^2}{\left[v+\eps\right]^{1/2}}}{\left[\theta - \theta_{\star}\right]^2}
     du }{4 \c{A}(t)}.
\]
which proves the first part of the Theorem. Next, under assumption \ref{ass:bounded}, Lemma \ref{lem:IntervalDefinition} implies that there exists a finite constant:
\[
    \c{K} = \sup_{t\in \R_+ } \norm{\left[v+\eps\right]^{1/4} \left( \theta -
    \theta_{\star} \right)}_{\infty}^2,
\]
and we note that:
\[
p(t) \ps{\frac{\left[\nabla f(\theta)\right]^2}{\left[v+\eps\right]^{1/2}}}{\left[\theta - \theta_{\star}\right]^2} \leq \c{K} p(t) \norm{\frac{\nabla f(\theta)}{\sqrt{v+\eps}}}^2 \leq \c{K} p(t) \norm{\frac{\nabla f(\theta)}{\sqrt{v}}}^2 .
\]
Now, from the expression of ODE \eqref{eq:ODE} in $v$ we get:
\[
\frac{d}{dt} \ln(v) + q(t) = p(t) \frac{\left[\nabla f(\theta)\right]^2}{v}
\]
which implies that:
\[
p(t) \ps{\frac{\left[\nabla f(\theta)\right]^2}{\left[v+\eps\right]^{1/2}}}{\left[\theta - \theta_{\star}\right]^2} \leq \c{K} \left( d \cdot q(t) + \sum_{i=1}^d\frac{d}{dt} \ln(v_i) \right).
\]
and it follows that:
\[
\int_{t_0}^t  p(u) \ps{\frac{\left[\nabla f(\theta)\right]^2}{\left[v+\eps\right]^{1/2}}}{\left[\theta - \theta_{\star}\right]^2}
     du \leq \int_{t_0}^t  \c{K} \left( d \cdot q(t) + \sum_{i=1}^d\frac{d}{dt} \ln(v_i)  \right)
     du
\]
The second inequality now easily follows from the fact that $v(t)$ is bounded by Lemma \ref{lem:IntervalDefinition}.

\subsection{Proof of Corollary \ref{thm:convergenceRateNesterov}}\label{ssec:NesterovBounds}

We may also consider the slightly more general energy functional:
\[
\c{E}_2(t,m,v,\theta) = \frac{\c{D}(t)}{2}  \norm{\left[v+\eps\right]^{1/4}\left(\theta -
    \theta_{\star} \right)}^2 - \c{B}(t) \ps{\theta-\theta_{\star}}{m}+\frac{\c{C}(t)}{2}\norm{\frac{m}{\left[v+\eps\right]^{1/4}}}^2  ,
\]
where $\c{D}(t)$ is a positive function. If we assume that $\c{D}(t)$ is bounded, we are able to follow the same reasoning of the previous section. In this case, we need to add the sufficient condition $\c{D}(t)' \leq 0$, and equality \eqref{eqB} and inequality \eqref{in2} are now given by:
\begin{align}
    & \c{B}'(t) - \c{B}(t) r(t) = -\c{D}(t) \label{eqBN}\\
    & \c{B}^2(t) \leq \c{D}(t)\c{C}(t) \label{in2N}
\end{align}
In particular, this implies that:
\[
\c{B}(t) =  e^{\int^t_{t_0} r(s) ds} \int_{t}^{\infty}
    \c{D}(s)e^{-\int^s_{t_0} r(u) du}ds
\]
while the equations for $\c{A}(t)$ and $\c{C}(t)$ are unchanged. Since $\c{D}(t)$ has negative derivative, in general, this computation can not lead to a stronger convergence rate than the one obtained in the previous section. Nevertheless, it does allow one to obtain convergence rates for parameters which are inaccessible in the previous section. Indeed, using this more general energy functional, we prove a convergence result for Nesterov when $0<r<3$:

\begin{proof}[End of proof of Corollary \ref{thm:convergenceRateNesterov}]
Let $t_0=1$ and $\c{D}(t) = t^{-\alpha}$ for some positive $\alpha$ which satisfies $2>\alpha>1-r$. Then:
\[
\begin{aligned}
\c{B}(t) &= t^r \int_{t}^{\infty}  s^{-r-\alpha} = \frac{t^{1-\alpha}}{r+\alpha-1}\\
\c{A}(t) &= \c{C}(t) = \frac{t^{2-\alpha}-1}{(2-\alpha)(r+\alpha-1)}
\end{aligned}
\]
Therefore, from inequality \eqref{in1} we get:
\[
\frac{t^{1-\alpha}}{r+\alpha-1} \leq \frac{t^{2-\alpha}-1}{(2-\alpha)(r+\alpha-1)} \left( \frac{2r}{3t} \right) \iff 2-\frac{2r}{3} \leq \alpha
\]
while from \eqref{in2N} we obtain:
\[
\frac{t^{2-2\alpha}}{(r+\alpha-1)^2} \leq \frac{t^{2-2\alpha}-1}{(2-\alpha)(r+\alpha-1)} \iff 1-\frac{r}{2} \leq \alpha
\]
In other words, it is enough to consider $\alpha = 2- 2r/3$ for every $0<r<3$. This implies that $f(\theta(t)) \to f_{\star}$ with rate of convergence:
\[
o(1/\c{A}(t)) = o(1/t^{2r/3})
\]
as we wanted to prove.
\end{proof}

\bibliography{differentialEquation2}
\bibliographystyle{ieeetr}

\end{document}